\documentclass[11pt,twoside]{article} 

\usepackage{eqnarray,amsmath}
\usepackage[utf8]{inputenc} 
\usepackage[T1]{fontenc}    
\usepackage{booktabs}       
\usepackage{amsfonts}       
\usepackage{nicefrac}       
\usepackage{microtype}      

\usepackage{epsf}
\usepackage{epsfig}
\usepackage{fancyhdr}
\usepackage{graphics}
\usepackage{graphicx}
\usepackage{psfrag}
\usepackage{fullpage}
\usepackage{pdfpages}
\usepackage{subcaption}

\usepackage{url}
\usepackage[colorlinks,linkcolor=magenta,citecolor=blue, pagebackref=true]{hyperref}
\renewcommand*{\backrefalt}[4]{%
    \ifcase #1 \footnotesize{(Not cited.)}%
    \or        \footnotesize{(Cited on page~#2.)}%
    \else      \footnotesize{(Cited on pages~#2.)}%
    \fi}

\usepackage{color}

\usepackage{amsthm}
\usepackage{amsmath}
\usepackage{amssymb,bbm}
\usepackage{caption}
\usepackage{algorithmic}
\usepackage{algorithm}
\usepackage{textcomp}
\usepackage{siunitx}
\usepackage{wrapfig}
\usepackage{algorithmic}
\usepackage{algorithm}
\usepackage{multirow}
\usepackage{multicol}

\usepackage{booktabs}       

\usepackage{makecell}
\usepackage{ragged2e}

\usepackage[textsize=tiny]{todonotes}

\usepackage{pifont}
\usepackage{tabu}
\usepackage{multirow}
\usepackage{xcolor,colortbl}
\usepackage{hhline}
\usepackage{colortbl}
\definecolor{ashgrey}{rgb}{0.7, 0.75, 0.71}

\usepackage{eqnarray,amsmath}

\renewcommand{\arraystretch}{1.4}

\newtheorem{assumption}{Assumption}

\newtheorem{lemma}{Lemma}

\newtheorem{theorem}{Theorem}
\newtheorem{proposition}{Proposition}

\newcommand{\bbE}{\mathbb{E}}

\newcommand{\softmax}{\sigma}

\newcommand{\dint}{\mathrm{d}}

\newcommand{\balpha}{\boldsymbol{\alpha}}

\newcommand{\bgamma}{\boldsymbol{\gamma}}

\newcommand{\brho}{\boldsymbol{\rho}}
\newcommand{\bpsi}{\boldsymbol{\psi}}

\newcommand{\bx}{\boldsymbol{x}}
\newcommand{\bbX}{\boldsymbol{X}}

\newcommand{\dajn}{\Delta \boldsymbol{a}_{j_1}^{n}}

\newcommand{\doijn}{\Delta \boldsymbol{\omega}_{i_2j_2|j_1}^{n}}

\newcommand{\deijn}{\Delta \boldsymbol{\eta}_{j_1i_2j_2}^{n}}
\newcommand{\dvijn}{\Delta \nu_{j_1i_2j_2}^{n}}
\newcommand{\dtijn}{\Delta \tau_{j_1i_2j_2}^{n}}

\newcommand{\ajn}{ \boldsymbol{a}_{j_1}^{n}}
\newcommand{\bjn}{ b_{j_1}^{n}}

\newcommand{\bejn}{ \beta_{j_2|j_1}^{n}}

\newcommand{\vjn}{\nu_{j_1j_2}^{n}}
\newcommand{\tjn}{\tau_{j_1j_2}^{n}}

\newcommand{\oin}{ \boldsymbol{\omega}_{i_2|j_1}^{n}}
\newcommand{\bein}{ \beta_{i_2|j_1}^{n}}
\newcommand{\ein}{ \boldsymbol{\eta}_{j_1i_2}^{n}}
\newcommand{\vin}{\nu_{j_1i_2}^{n}}
\newcommand{\tin}{\tau_{j_1i_2}^{n}}

\newcommand{\aj}{ \boldsymbol{a}_{j_1}^{*}}
\newcommand{\bj}{ b_{j_1}^{*}}
\newcommand{\oj}{ \boldsymbol{\omega}_{j_2|j_1}^{*}}
\newcommand{\bej}{ \beta_{j_2|j_1}^{*}}
\newcommand{\ej}{ \boldsymbol{\eta}_{j_1j_2}^{*}}
\newcommand{\vj}{\nu_{j_1j_2}^{*}}
\newcommand{\tj}{\tau_{j_1j_2}^{*}}

\newcommand{\ai}{ \boldsymbol{a}_{i_1}^{*}}
\newcommand{\bi}{ b_{i_1}^{*}}
\newcommand{\oi}{ \boldsymbol{\omega}_{i_2|i_1}^{*}}
\newcommand{\bei}{ \beta_{i_2|i_1}^{*}}
\newcommand{\ei}{ \boldsymbol{\eta}_{i_1i_2}^{*}}
\newcommand{\vi}{\nu_{i_1i_2}^{*}}
\newcommand{\ti}{\tau_{i_1i_2}^{*}}

\newcommand{\zerod}{\boldsymbol{0}_d}

\newcommand{\brj}{r^{SL}_{j_2|j_1}}
\newcommand{\trj}{r^{LL}_{j_2|j_1}}

\newcommand{\hrj}{r^{SS}_{j_2|j_1}}

\newcommand{\brone}{r^{SL}_{1|1}}
\newcommand{\hrone}{r^{SS}_{1|1}}

\newcommand{\trone}{r^{LL}_{1|1}}

\newcommand{\normf}[1]{\|#1\|_{L^2(\mu)}}

\DeclareMathOperator*{\argmax}{arg\,max}
\DeclareMathOperator*{\argmin}{arg\,min}

\begin{document}

\begin{center}

{\bf{\LARGE{On Expert Estimation in Hierarchical Mixture of
Experts: Beyond Softmax Gating Functions}}}
  
\vspace*{.2in}
{\large{
\begin{tabular}{ccccc}
Huy Nguyen$^{\dagger,\star}$ & Xing Han$^{\diamond,\star}$ & Carl William Harris$^{\diamond}$ & Suchi Saria$^{\diamond,\star\star}$ & Nhat Ho$^{\dagger,\star\star}$
\end{tabular}
}}

\vspace*{.2in}

\begin{tabular}{cc}
$^{\dagger}$The University of Texas at Austin\\
$^{\diamond}$Johns Hopkins University
\end{tabular}

\vspace*{.2in}

\today

\vspace*{.2in}

\begin{abstract}
With the growing prominence of the Mixture of Experts (MoE) architecture in developing large-scale foundation models, we investigate the Hierarchical Mixture of Experts (HMoE), a specialized variant of MoE that excels in handling complex inputs and improving performance on targeted tasks. Our analysis highlights the advantages of using the Laplace gating function over the traditional Softmax gating within the HMoE frameworks. We theoretically demonstrate that applying the Laplace gating function at both levels of the HMoE model helps eliminate undesirable parameter interactions caused by the Softmax gating and, therefore, accelerates the expert convergence as well as enhances the expert specialization. Empirical validation across diverse scenarios supports these theoretical claims. This includes large-scale multimodal tasks, image classification, and latent domain discovery and prediction tasks, where our modified HMoE models show great performance improvements compared to the conventional HMoE models.
\end{abstract}

\end{center}

\let\thefootnote\relax\footnotetext{$\star$ Equal Contribution, $\star\star$ Equal Advising.}

\section{Introduction}\label{sec:intro}
In recent years, the integration of mixture-of-experts (MoE) within large-scale foundation models has markedly advanced the machine learning field \cite{liu2024deepseek, jiang2024mixtral, fedus2022switch, riquelme2021scaling, zhou2022mixture, mustafa2022multimodal}. Going back in time, this statistical model was first introduced by \cite{jacobs1991adaptive} as an adaptive variant of classic mixture models \cite{Lindsay-1995}, combining the power of several experts which are often formulated as feed-forward networks \cite{shazeer2017outrageously,liu2024deepseek}, classifiers \cite{chen2022theory,nguyen2024general}, or regression functions \cite{deveaux1989linear,faria2010regression}. However, instead of assigning those experts constant weights as in mixture models, the MoE employs a gating mechanism to dynamically allocate data-dependent weights to the experts. In other words, the set of weights will vary with the input value, thereby enhancing the model generalization and allowing the MoE to efficiently handle diverse and complex datasets. Furthermore, in order to increase the model capacity, that is, the number of learnable parameters, \cite{shazeer2017outrageously} proposed a so-called Top-$K$ sparse gating which activated only a few relevant experts per input rather than the entire set of experts. They demonstrated that this sparse gating mechanism helps achieve a significant improvement in the model capacity and model performance without a proportional increase in the computational overhead. As a consequence, there is a surge of interest in applying sparse MoE models in various large-scale applications, including natural language processing \cite{puigcerver2024sparse,zhou2023brainformers,Du_Glam_MoE}, computer vision \cite{liang_m3vit_2022,riquelme2021scaling}, multi-task learning \cite{gupta2022sparsely,hazimeh_dselect_k_2021}, speech recognition \cite{You_Speech_MoE,gulati_conformer_2020}, etc.

\vspace{0.5 em}
\noindent
The Hierarchical Mixture of Experts (HMoE) \cite{Jordan-1994,fritsch1996adaptively} is a special type of MoE that is characterized by a layered structure of decision modules and expert networks that operate in tandem to refine decision-making at each level, optimizing the allocation of computational resources and enhancing specialization for complex tasks. Unlike the standard MoE, which typically involves a single gating network directing inputs to various expert networks, HMoE introduces multiple layers of gating mechanisms and experts. This hierarchical design divides the problem space recursively, allowing different experts to specialize in subspaces of the input, leading to enhanced flexibility and model generalization \cite{jiang1999approximation, azran2004data}. 
Figure \ref{fig:demo_fig} compares HMoE and standard MoE in processing multimodal input data. The HMoE's hierarchical arrangement excels at processing intricate inputs, including those that can be categorized into semantically distinct subgroups like text, images, or time series, or involve various sub-domains. This architecture allows experts at lower levels to grasp detailed token-level intricacies while permitting experts at higher levels to concentrate on broader or domain-specific tasks; it also enhances model transparency. Conversely, using a standard MoE with an equivalent number of experts necessitates a single gating network to select from numerous experts each time, potentially causing interference among them.

\begin{figure*}[t]
    \centering
    \includegraphics[width=\textwidth]{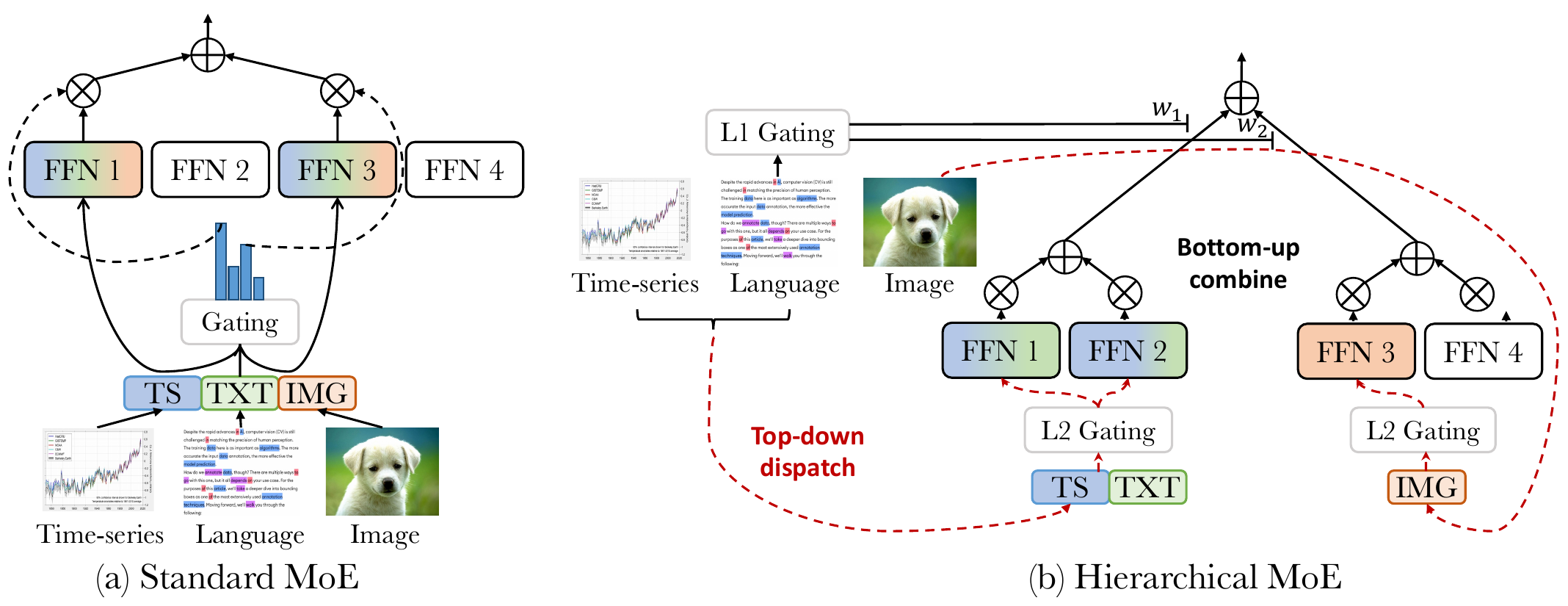}
    \caption{\small Comparison of HMoE and standard MoE in managing multimodal input: MoE excels at processing homogeneous inputs. However, it faces challenges with more intricate structures, such as inputs that can be split into subgroups or those with inherently hierarchical configurations. By contrast, HMoE improves upon this by decomposing tasks into subproblems and directing subsets of data to specialized groups of experts. This approach allows for more granular specialization and enhances the model's capability to handle complex inputs.}
    \label{fig:demo_fig}
\end{figure*}

\vspace{0.5 em}
\noindent
\textbf{Related works.} ~MoE \cite{jacobs1991adaptive, xu1994alternative} has gained significant popularity for managing complex tasks. Unlike traditional models that reuse the same parameters for all inputs, MoE selects distinct parameters for each specific input. This results in a \textit{sparsely} activated layer, enabling a substantial scaling of model capacity without a corresponding increase in computational cost. Recent studies \cite{shazeer2017outrageously, fedus2022switch, mustafa2022multimodal, zhou2023brainformers, shen2023scaling, han2024fusemoe} have demonstrated the effectiveness of integrating MoE with cutting-edge models across a diverse range of tasks. \cite{nie2021evomoe, zhou2022mixture, puigcerver2024sparse} have also tackled key challenges such as accuracy and training instability. As an advanced type of MoE, HMoE has been applied to image classification \cite{irsoy2021dropout}, speech recognition \cite{peng1996bayesian, zhao1994hierarchical}, and complex decision-making tasks \cite{jeremiah2013specifying, moges2016hierarchical}; its hierarchical structures have also been shown to be effective in improving model performance in complex data structures \cite{ng2007extension, peralta2014embedded, zhao2021hierarchical, azran2004data}. Most recently, building upon the spirit of HMoE, \cite{liao2025hmora} proposed a hybrid routing approach combining token-level and task-level routing in a hierarchical manner, and it is more efficient in leveraging the multi-granular information in large language models.

\vspace{0.5 em}
\noindent
While MoE has been widely employed to scale up large models, its theoretical foundations have remained relatively underdeveloped. First of all, \cite{mendes2011convergence} studied the maximum likelihood estimator for parameters of the MoE with each expert being a polynomial regression model. In particular, they investigated the convergence rate of the estimated density to the true density under the Kullback-Leibler (KL) divergence and gave some insights on how many experts should be chosen. Next, \cite{ho2022gaussian} conducted a similar convergence analysis for input-free gating Gaussian MoE but using the Hellinger distance for the density estimation problem instead of the KL divergence. Additionally, they utilized the generalized Wasserstein distance to capture the parameter estimation rates which were negatively affected by the algebraic interactions among parameters.
\cite{nguyen2023demystifying} then generalized these results to a more popular setting known as softmax gating Gaussian MoE. Rather than leveraging the generalized Wasserstein distance for the parameter estimation problem, they proposed novel Voronoi-based loss functions which were shown to characterize the parameter estimation rates more accurately. Recently, \cite{han2024fusemoe} advocated using a new Laplace gating function which induced faster convergence rates than softmax gating due to a reduced number of parameter interactions. However, given that HMoE requires the choice of multiple gating functions, to the best of our knowledge, a comprehensive convergence analysis for HMoE has remained elusive in the literature.

\vspace{0.5 em}
\noindent
\textbf{Contributions.} In this paper, we explore the intricacies of HMoE training by examining the effectiveness of three distinct combinations of two widely used gating functions: the Softmax gating function \cite{Jordan-1994} and the Laplace gating function \cite{han2024fusemoe}, implemented at two hierarchical levels of the HMoE model. Additionally, we provide insights into the practical performance of HMoE when applied to multimodal and multi-domain inputs. We hope this work will serve as a foundation for future research in this relatively underexplored area. Our main contributions can be summarized as follows:

\vspace{0.5 em}
\noindent
\textbf{1. Theoretical convergence analysis of expert estimation.} Expert specialization, as discussed in \cite{dai2024deepseekmoe}, is a critical issue involving the rate at which an expert becomes specialized in specific tasks or aspects of the data. However, to the best of our knowledge, prior research has primarily focused on studying expert specialization in single-level MoE models, leaving the dynamics in HMoE models largely unexplored. To address this gap, we perform a comprehensive convergence analysis of experts within the two-level HMoE model from a statistical perspective. Specifically, we examine the Gaussian HMoE model \cite{Jordan-1994} with three different combinations of Softmax and Laplace gating functions. Our theoretical findings reveal that using Softmax gating at either level induces intrinsic interactions among the model parameters, expressed through partial differential equations (PDEs), which hinder expert convergence. In contrast, employing Laplace gating at both levels helps eliminate these parameter interactions, thereby significantly accelerating expert convergence and enhancing expert specialization.

\vspace{0.5 em}
\noindent
\textbf{2. Application of HMoE in multi-modal and multi-domain learning.} We demonstrate HMoE’s effectiveness over standard MoE, and further validate our theoretical findings on input data with multi-modal or multi-domain structures. By incorporating the three aforementioned combinations of gating functions, our experiments confirm that using the Laplace gating at both levels improves performance across multiple downstream tasks compared to the standard Softmax gating baseline. Additionally, we observe that different combinations of the Laplace and Softmax gating can also noticeably enhance results, leading to better and more robust performance by offering a broader selection of gating function combinations. These findings highlight the practical benefits of selecting appropriate gating functions to enhance HMoE’s capabilities.

\vspace{0.5 em}
\noindent
\textbf{Organization.} The paper proceeds as follows. In Section~\ref{sec:preliminaries}, we exhibit the problem setup following by some fundamental results on the density estimation of the Gaussian HMoE model. Next, we investigate the convergence behavior of parameter estimation and expert estimation in Section~\ref{sec:theory}. Then, in Section~\ref{sec:exp}, we perform comprehensive synthetic and real-world experiments on datasets in different domains to justify our theoretical findings and demonstrate the efficacy of the HMoE model before concluding the paper in Section~\ref{sec:discussion}. Finally, we provide the proof for establishing the parameter and expert estimation rates in Section~\ref{appendix:param_rates}, while other proofs and experimental details are deferred to the Appendices.

\vspace{0.5 em}
\noindent
\textbf{Notations.} We let $[n]$ stand for the set $\{1,2,\ldots,n\}$ for any  $n\in\mathbb{N}$. Next, for any set $S$, we denote $|S|$ as its cardinality. For any vector $v \in \mathbb{R}^{d}$ and $\alpha:=(\alpha_1,\alpha_2,\ldots,\alpha_d)\in\mathbb{N}^d$, we let $v^{\alpha}=v_{1}^{\alpha_{1}}v_{2}^{\alpha_{2}}\ldots v_{d}^{\alpha_{d}}$, $|v|:=v_1+v_2+\ldots+v_d$ and $\alpha!:=\alpha_{1}!\alpha_{2}!\ldots \alpha_{d}!$, while $\|v\|$ stands for its $L^2$-norm value. For any two positive sequences $(a_n)_{n\geq 1}$ and $(b_n)_{n\geq 1}$, we write $a_n = \mathcal{O}(b_n)$ or $a_{n} \lesssim b_{n}$ if there exist $C > 0$ such that $a_n \leq C b_n$ for all $ n\in\mathbb{N}$. Additionally, the notation $a_{n} = \mathcal{O}_{P}(b_{n})$ means that $a_{n}/b_{n}$ is stochastically bounded, while the notation $a_n=\widetilde{\mathcal{O}}(b_n)$ indicates that the previous bound may depend on the logarithmic function of $b_n$. Lastly, for any two probability density functions $p,q$ dominated by the Lebesgue measure $\mu$, we denote $h^2(p,q) = \frac 1 2 \int (\sqrt p - \sqrt q)^2 d\mu$ as their squared Hellinger distance and $V(p,q) = \frac 1 2 \int |p-q| d\mu$ as their Total Variation distance.

\section{Preliminaries}
\label{sec:preliminaries}
In this section, we formulate the Gaussian HMoE model and present some essential assumptions for our theoretical study in Section~\ref{sec:setup}. Then, we explore the convergence behavior of the conditional density estimation of the Gaussian HMoE in Section~\ref{sec:density_estimation}. 
\subsection{Problem Setup}
\label{sec:setup}
To begin with, we assume that an i.i.d. sample of size $n$: $(\bbX_{1}, Y_{1}), (\bbX_{2}, Y_{2}), \ldots, (\bbX_{n}, Y_{n})$ in $\mathbb{R}^d\times\mathbb{R}$, where $\bbX_i$ is a covariate and $Y_i$ is a response variable, is generated from the two-level Gaussian HMoE model whose conditional density function is given by
\begin{align}
    \label{eq:general_density}
    p_{G_{*}}(y|\bx) &:= \sum_{i_1=1}^{k^*_1} \softmax(s_1(\bx,\ai)+\bi)\sum_{i_2=1}^{k^*_2}\softmax(s_2(\bx,\oi)+\bei) \pi(y|(\ei)^{\top}\bx+\ti,\vi).
\end{align}
Throughout this paper, we consider three different types of Gaussian HMoE models corresponding to three different combinations of the Softmax gating and the Laplace gating specified by the similarity score functions $s_1$ and $s_2$. In particular, we refer to the above model as
\begin{itemize}
    \item \emph{the Softmax-Softmax Gating Gaussian HMoE} if $s_1(\bx,\ai)=(\ai)^{\top}\bx$ and $s_2(\bx,\oi)=(\oi)^{\top}\bx$, and customize the conditional density notation~\eqref{eq:general_density} as $p^{SS}_{G_*}(y|\bx)$;
    \item \emph{the Softmax-Laplace Gating Gaussian HMoE} if $s_1(\bx,\ai)=(\ai)^{\top}\bx$ and $s_2(\bx,\oi)=-\|\oi-\bx\|$, and customize the conditional density notation~\eqref{eq:general_density} as $p^{SL}_{G_*}(y|\bx)$;
    \item \emph{the Laplace-Laplace Gating Gaussian HMoE} if $s_1(\bx,\ai)=-\|\ai-\bx\|$ and $s_2(\bx,\oi)=-\|\oi-\bx\|$, and customize the conditional density notation~\eqref{eq:general_density} as $p^{LL}_{G_*}(y|\bx)$;
\end{itemize}
Next, in each type of the Gaussian HMoE, we define $G_*$ as a \emph{mixing measure}, i.e., a weighted sum of Dirac measures $\delta$ given by
\begin{align*}
    G_{*} := \sum_{i_1=1}^{k^*_1}\exp(\bi)\sum_{i_2=1}^{k^*_2}\exp(\bei)\delta_{(\ai,\oi,\ti,\ei,\vi)},
\end{align*}
where $(\bi,\ai,\bei,\oi,\ti,\ei,\vi)$ are true yet unknown parameters in the parameter space $\Theta\subseteq\mathbb{R}\times\mathbb{R}^d\times\mathbb{R}\times\mathbb{R}^d\times\mathbb{R}^q\times\mathbb{R}_+$. Besides, $k^*_1$ denotes the number of mixtures in the two-level Gaussian HMoE, whereas $k^*_2$ is the number of experts in each mixture. 
For any integer $k\in\mathbb{N}$ and real-valued vector $(v_i)_{i=1}^{k}$, we denote by $\softmax(v_i):={\exp(v_i)}/{\sum_{j=1}^{k}\exp(v_j)}$ the softmax function. Meanwhile, $\pi(\cdot|\mu,\nu)$ stands for the univariate Gaussian density function with mean $\mu$ and variance $\nu$. Additionally, it is worth noting that the conditional expectation of the response variable $Y$ given the covariate $\bbX$ is also an HMoE
\begin{align*}
    \mathbb{E}[Y|\bbX]=\sum_{i_1=1}^{k^*_1} \softmax(s_1(\bbX,\ai)+\bi)\sum_{i_2=1}^{k^*_2}\softmax(s_2(\bbX,\oi)+\bei) \cdot[(\ei)^{\top}\bbX+\ti],
\end{align*}
where $(\ei)^{\top}\bx+\ti$ is referred to as an expert. 

\vspace{0.5 em}
\noindent
Recall that expert specialization is an essential problem in the MoE literature where we explore how fast an expert specializes in some tasks or some aspects of the data \cite{dai2024deepseekmoe,oldfield2024multilinear,krishnamurthy2023improvingexpertspecializationmixture}. Therefore, understanding the convergence behavior of expert estimation is of great importance.

\vspace{0.5 em}
\noindent
\textbf{Maximum likelihood estimation (MLE).} We can estimate the experts $(\ei)^{\top}\bx+\ti$ by estimating their parameters. To estimate the unknown parameters, or equivalently the unknown mixing measure $G_*$, we utilize the maximum likelihood method \cite{vandeGeer-00}. For simplicity, we assume that the value of $k^*_1$ is known (since the analysis would become unnecessarily complicated otherwise), while the value of $k^*_2$ remains unknown. Then, we over-specify the true model~\eqref{eq:general_density} by considering an MLE within a class of mixing measures with at most $k^*_1k_2$ components, where $k_2>k^*_2$, as follows:
\begin{align}
    \label{eq:least_squared_estimator}
    \widehat{G}^{type}_n:=\argmax_{G\in\mathcal{G}_{k^*_1,k_2}(\Theta)}\frac{1}{n}\sum_{i=1}^{n}\log(p^{type}_{G}(Y_i|\bbX_i)),
\end{align}
in which
\begin{align*}
    \mathcal{G}_{k^*_1,k_2}(\Theta):=\Big\{G=\sum_{i_1=1}^{k^*_1}\exp(b_{i_1})\sum_{i_2=1}^{k'_2}\exp(\beta_{i_2|i_1})\delta_{(\boldsymbol{a}_{i_1},\boldsymbol{\omega}_{i_2|i_1},\boldsymbol{\eta}_{i_1i_2},\tau_{i_1i_2},\nu_{i_1i_2})}:k'_2\in[k_2],\\
    (b_{i_1},\boldsymbol{a}_{i_1},\beta_{i_2|i_1},\boldsymbol{\omega}_{i_1i_2},\tau_{i_1i_2},\boldsymbol{\eta}_{i_1i_2},\nu_{i_1i_2})\in\Theta\Big\}
\end{align*}
and $type\in\{SS,SL,LL\}$.

\vspace{0.5 em}
\noindent
\textbf{Assumptions.} For the sake of theory, let us introduce some mild assumptions on the model parameters as well as the covariate throughout this paper:

\vspace{0.5 em}
\noindent
\emph{(A.1) We assume that the parameter space $\Theta$ is compact and the covariate space $\mathcal{X}$ is bounded to guarantee the MLE convergence.}

\vspace{0.5 em}
\noindent
\emph{(A.2) In order that the Gaussian HMoE is identifiable, that is, $p^{SS}_{G}(y|\bx)=p^{SS}_{G_*}(y|\bx)$ for almost every $(\bx,y)$ implies $G\equiv G_*$, the softmax gating value must not be invariant to parameter translation. Therefore, we let $\boldsymbol{a}^*_{k^*_1}=\zerod,b^*_{k^*_1}=0$ and $\boldsymbol{\omega}^*_{k^*_2|i_1}=\zerod,\beta^*_{k^*_2|i_1}=0$ for any $i_1\in[k^*_1]$.}

\vspace{0.5 em}
\noindent
\emph{(A.3) For any $i_1\in[k^*_1]$, we let $(\boldsymbol{\eta}^*_{i_11},\tau^*_{i_11},\nu^*_{i_11}),\ldots,(\boldsymbol{\eta}^*_{i_1k^*_2},\tau^*_{i_1k^*_2},\nu^*_{i_1k^*_2})$ be distinct parameters so that the Gaussian distributions within the same mixture are different from each other.}

\vspace{0.5 em}
\noindent
\emph{(A.4) To ensure that the gating depend on the covariate, we assume at least one among gating parameters in the first level $\boldsymbol{a}^*_{1},\ldots,\boldsymbol{a}^*_{k^*_1}$ (resp. those in the second level $\boldsymbol{\omega}^*_{1},\ldots,\boldsymbol{\omega}^*_{k^*_1}$) is different from zero.}

\subsection{Density Estimation}
\label{sec:density_estimation}
Subsequently, we study the consistency of the MLE under the Gaussian HMoE model and determine the convergence rate of the density estimation. 
\begin{proposition}
    \label{prop:identifiability}
    For each $type\in\{SS,SL,LL\}$, suppose that the equation $p^{type}_{G}(y|\bx)=p^{type}_{G_*}(y|\bx)$ holds true for almost surely $(\bx,y)$, then we get that $G\equiv G_*$.
\end{proposition}
\noindent
The proof of Proposition~\ref{prop:identifiability} is deferred to Appendix~\ref{appendix:identifiability}. The above result indicates that the Gaussian HMoE model is identifiable, which ensures that the MLE $\widehat{G}^{type}_n$ converge to the true counterpart $G_*$. Given the identifiable property of the Gaussian HMoE model, we proceed to investigate the convergence behavior of the density estimation $p^{type}_{\widehat{G}_n}$ to the true density $p^{type}_{G_*}$ in Proposition~\ref{prop:density_estimation} whose proof can be found in Appendix~\ref{appendix:density_rate}.
\begin{proposition}    \label{prop:density_estimation}
For each $type\in\{SS,SL,LL\}$ and an MLE $\widehat{G}^{type}_n$ defined in equation~\eqref{eq:least_squared_estimator}, the corresponding density estimation $p^{type}_{\widehat{G}_n}$ converges to the true density $p^{type}_{G_*}$ under the Hellinger distance $h$ at the following rate:
\begin{align*}
    \bbE_{\bbX}[h(p^{type}_{\widehat{G}^{type}_{n}}(\cdot|\bbX), p^{type}_{G_{*}}(\cdot|\bbX))] =\widetilde{\mathcal{O}}_P(n^{-1/2}).
    \end{align*}
\end{proposition}
\noindent
Proposition~\ref{prop:density_estimation} indicates that the conditional density estimation of the Gaussian HMoE $p^{type}_{\widehat{G}_n}$ admits the convergence rate of order $\widetilde{\mathcal{O}}_P(n^{-1/2})$, which is parametric on the sample size $n$. Given this result, we will discuss a strategy to determine the convergence rate of parameter estimation based on the above density estimation rate.

\vspace{0.5 em}
\noindent
\textbf{From density estimation rate to parameter estimation rate.} Consequently, if we are able to construct a loss function among parameters denoted by, for example, $\mathcal{L}(\widehat{G}^{type}_n,G_*)$, satisfying the bound
\begin{align}
    \label{eq:hellinger_bound}
    \mathcal{L}(\widehat{G}^{type}_n,G_*)\lesssim \bbE_{\bbX}[h(p^{type}_{\widehat{G}^{type}_{n}}(\cdot|\bbX), p^{type}_{G_{*}}(\cdot|\bbX))],
\end{align}
then we will obtain the parameter estimation rates $\mathcal{L}(\widehat{G}^{type}_n,G_*)=\widetilde{\mathcal{O}}_P(n^{-1/2})$, which leads to our desired rates for estimating experts. However, while such Hellinger bound has been well studied under the setting of one-level Gaussian MoE \cite{ho2022gaussian,nguyen2023demystifying}, it has remained elusive for the hierarchical setting.

\section{Convergence Rates of Parameter Estimation and Expert Estimation}
\label{sec:theory}
In this section, we conduct a convergence analysis of parameter estimation and expert estimation under three different types of the two-level Gaussian HMoE associated with three distinct combinations of the Softmax gating and the Laplace gating. Our main objective is to find which gating combination would induce the fastest expert estimation rate, and then provide useful insights into the design of Gaussian HMoE.

\subsection{Softmax-Softmax Gating Gaussian HMoE}
\label{sec:softmax_softmax}

We start with the Softmax-Softmax gating Gaussian HMoE model where we use the Softmax gating in both levels, and the corresponding conditional density function is given by
\begin{align}
    \label{eq:softmax_softmax}
    p^{SS}_{G_{*}}(y|\bx) &:= \sum_{i_1=1}^{k^*_1} \softmax((\ai)^{\top}\bx+\bi)\sum_{i_2=1}^{k^*_2}\softmax((\oi)^{\top}\bx+\bei) \pi(y|(\ei)^{\top}\bx+\ti,\vi),
\end{align}
where the abbreviation $SS$ stands for ``Softmax-Softmax''. As mentioned in Section~\ref{sec:density_estimation}, in order to determine the parameter and expert estimation rates given the density estimation rate in Proposition~\ref{prop:density_estimation}, it suffices to build a loss function among parameters $\mathcal{L}(\widehat{G}^{SS}_n,G_*)$ such that the Hellinger lower bound in equation~\eqref{eq:hellinger_bound} holds true. In the following paragraph, we will highlight some fundamental challenges for deriving that bound, which indicates how to design the loss function among parameters in order to capture the convergence rates of parameter estimation and expert estimation accurately.

\vspace{0.5 em}
\noindent
\textbf{Challenges.} 
Our main technique for establishing the Hellinger lower bound~\eqref{eq:hellinger_bound} is to decompose the density estimation and the true density, i.e., $p^{SS}_{\widehat{G}^{SS}_n}(y|\bx)-p^{SS}_{G_*}(y|\bx)$, into a combination of linearly independent terms by applying the Taylor expansion to the function $u(\bx;\boldsymbol{a},\boldsymbol{\omega},\boldsymbol{\eta},\tau,\nu):=\exp(\boldsymbol{a}^{\top}\bx)\exp(\boldsymbol{\omega}^{\top}\bx)\pi(y|\boldsymbol{\eta}^{\top}\bx+\tau,\nu)$ with respect to its parameters. In previous works \cite{ho2022gaussian,nguyen2023demystifying}, it is well-known that there is an interaction between the mean parameter $\tau$ and the variance parameter $\nu$ of the Gaussian density via the partial differential equation (PDE) $\frac{\partial u}{\partial\nu}=\frac{1}{2}\cdot\frac{\partial^2u}{\partial\tau^2}$. Such PDE induces several linearly dependent terms in the aforementioned decomposition, thereby leading to significantly slow rates for estimating those parameters. In this paper, we discover that the first-level gating parameter $\boldsymbol{a}$ also interacts with the second-level parameters $\boldsymbol{\eta}, \tau, \boldsymbol{\omega}$, that is,
\begin{align}
    \label{eq:PDE}
    \text{(I) }~ \frac{\partial u}{\partial\boldsymbol{\eta}}=\frac{\partial^2u}{\partial\boldsymbol{a}\partial\tau}; \qquad  \text{(II) }~ \frac{\partial u}{\partial\boldsymbol{a}}=\frac{\partial u}{\partial \boldsymbol{\omega}}.
\end{align}
To the best of our knowledge, these intrinsic interactions have not been noted before in the literature. Therefore, we have to take the solvability of the unforeseen system of polynomial equations~\eqref{eq:system_SS} into account to capture that interaction.

\vspace{0.5 em}
\noindent
\textbf{System of polynomial equations.} For each $m\geq 2$, we define $r^{SS}(m)$ as the smallest natural number $r$ such that the following system does not have any non-trivial solutions for the unknown variables $(p_{i_2},\boldsymbol{q}_{1i_2},\boldsymbol{q}_{2},\boldsymbol{q}_{3i_2},q_{4i_2},q_{5i_2})_{i_2=1}^{m}$
\begin{align}
    \label{eq:system_SS}
    \sum_{i_2=1}^{m}\sum_{(\balpha_1,\balpha_2,\balpha_3,\alpha_4,\alpha_5)\in\mathcal{I}^{SS}_{\brho_1,\rho_2}}\frac{1}{\balpha!}\cdot p^2_{i_2}\boldsymbol{q}_{1i_2}^{\balpha_1}\boldsymbol{q}_{2}^{\balpha_2}\boldsymbol{q}_{3i_2}^{\balpha_3}q_{4i_2}^{\alpha_4}q_{5i_2}^{\alpha_5}=0, \quad 1\leq|\brho_1|+\rho_2\leq r, 
\end{align}
where $\mathcal{I}^{SS}_{\brho_1,\rho_2}:=\{(\balpha_1,\balpha_2,\balpha_3,\alpha_4,\alpha_5)\in\mathbb{R}^d\times\mathbb{R}^d\times\mathbb{R}^d\times\mathbb{R}\times\mathbb{R}_+:\balpha_1+\balpha_2+\balpha_3=\brho_1, |\balpha_3|+\alpha_4+2\alpha_5=\rho_2\}$. Here, a solution is categorized as non-trivial if all the values of $p_{i_2}$ are different from zero and at least one among $q_{4i_2}$ is non-zero. Note that $r^{SS}(m)$ is a monotonically increasing function. However, finding the exact value of $r^{SS}(m)$ is a demanding problem in the field of algebraic geometry \cite{Sturmfels_System}. Thus, we provide in Lemma~\ref{lemma:rss_values} (whose proof is in Appendix~\ref{appendix:rss_values}) some specific values of $r^{SS}(m)$ when $m$ is small, while those for larger $m$ are left for future development. 
\begin{lemma}
    \label{lemma:rss_values}
    For any $d\geq 1$, we have that $r^{SS}(2)=4$ and $r^{SS}(3)=6$, while we conjecture that $r^{SS}(m)\geq 7$ for $m\geq 4$.
\end{lemma}
\noindent
Subsequently, we need to design a loss function $\mathcal{L}(\cdot,\cdot)$ among parameters that satisfies the lower bound in equation~\eqref{eq:hellinger_bound}. In the literature, \cite{nguyen2016latentmixing} utilized the generalized Wasserstein to capture the convergence behavior of MLE in mixture models. Then, \cite{ho2022gaussian} reused the generalized Wasserstein for establishing the convergence rate of parameter estimation in input-independent gating Gaussian MoE. An advantage of using this divergence is that we can deduce the convergence rates of individual parameters from
the convergence rate of the MLE $\widehat{G}_n$ as indicated in Theorem 1 in \cite{ho2022gaussian}. On the other hand, the generalized Wasserstein divergence is incapable of accurately capturing those rates. More concretely, the generalized Wasserstein implies the same estimation rates for all the individual parameters although those rates should change with  
the number of fitted experts. To close this gap, \cite{nguyen2023demystifying} proposed using a loss function constructed based on the concept of Voronoi cells \cite{manole22refined} for analyzing the convergence of parameter estimation in one-level Softmax gating Gaussian MoE. In order to leverage this Voronoi loss function for our work, we need to generalize it to the hierarchical setting.

\begin{figure*}[t]
    \centering
    \includegraphics[scale=0.3]{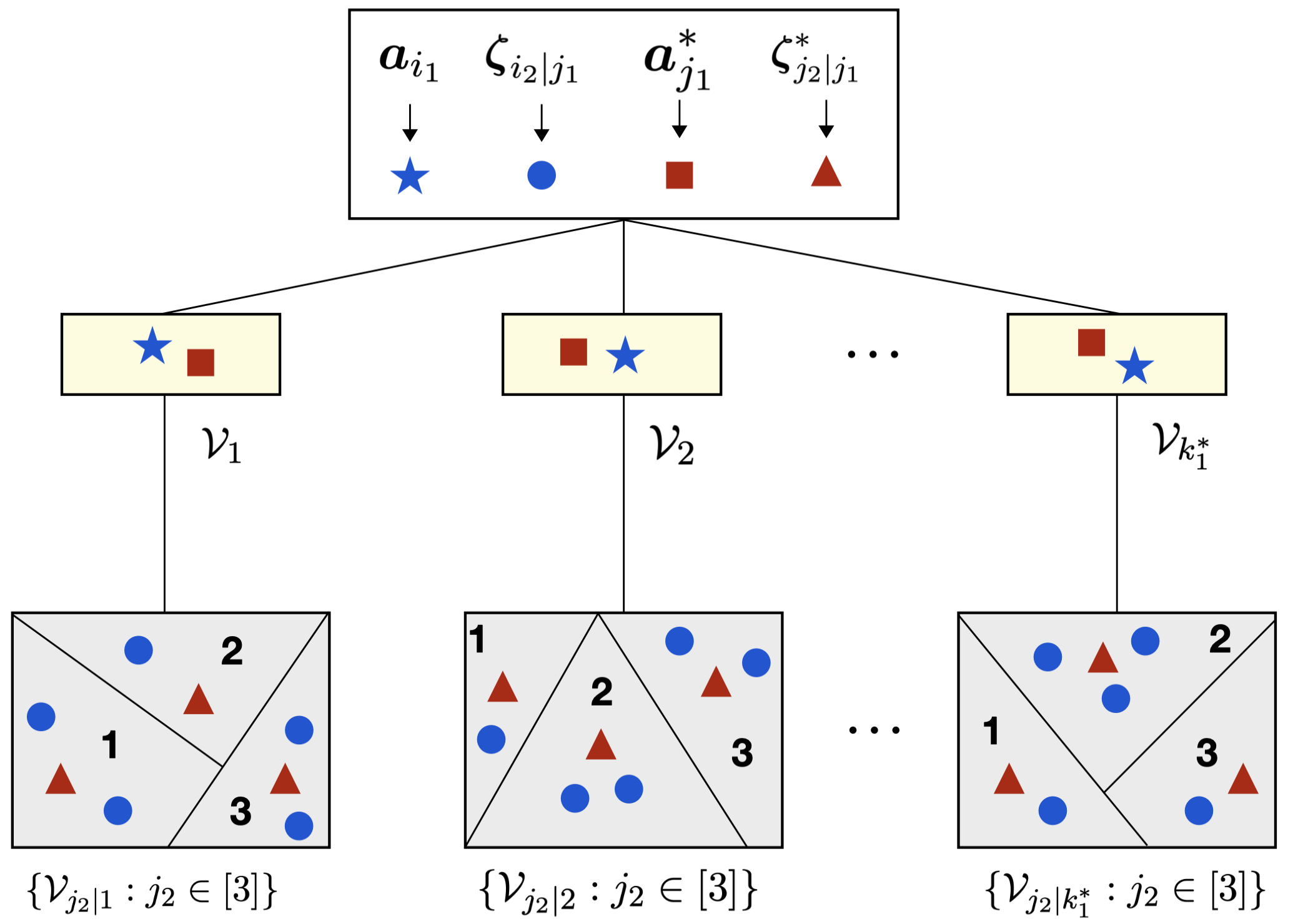}
    \caption{Illustration of Voronoi cells defined in equations~\eqref{eq:Voronoi_cells_level_1} and \eqref{eq:Voronoi_cells_level_2}. In the first level, Voronoi cells $\mathcal{V}_{j_1}$, for $j_1\in[k_1^*]$, are generated by ground-truth first-level parameters $\boldsymbol{a}^*_{j_1}$ (red squares) and contain first-level fitted parameters $\boldsymbol{a}_{i_1}$ (blue stars). Since the value of $k_1^*$ is known, the red squares are exactly fitted, implying that each Voronoi cell $\mathcal{V}_{j_1}$ has only one blue star. In the second level, each gray rectangle depicts a set of $k_2^*=3$ Voronoi cells $\{\mathcal{V}_{j_2|j_1}:j_2\in[k_2^*]\}$ generated by ground-truth second-level parameters $\boldsymbol{\zeta}^*_{j_2|j_1}$ (red triangles), for $j_1\in[k_1^*]$. These three Voronoi cells $\mathcal{V}_{j_2|j_1}$ contain a total of $k_2=5$ second-level fitted parameters $\boldsymbol{\zeta}_{i_2|j_1}$ (blue rounds). Since $k_2>k_2^*$, there exist some Voronoi cells $\mathcal{V}_{j_2|j_1}$ having more than one blue round.
    } 
    \label{fig:voronoi}
\end{figure*}

\vspace{0.5 em}
\noindent
\textbf{Voronoi loss.} To precisely characterize the convergence rate of parameter estimation, it is necessary to capture the number of fitted parameters approaching each individual true parameter in both levels of Gaussian HMoE. For that purpose, let us introduce the concept of Voronoi cells \cite{manole22refined}. In particular, given an arbitrary mixing measure $G\in\mathcal{G}_{k^*_1k_2}(\Theta)$, we distribute its atoms across the Voronoi cells $\{\mathcal{V}_{j_1}(G),j_1 \in [k^*_1]\}$ and $\{\mathcal{V}_{j_2|j_1}(G),j_1 \in [k^*_1],j_2\in[k^*_2]\}$ generated by the atoms of $G_*$ (see also Figure~\ref{fig:voronoi}), where
\begin{align}
    \label{eq:Voronoi_cells_level_1}
    \mathcal{V}_{j_1}\equiv\mathcal{V}_{j_1}(G)&:=\{i_1\in[k^*_1]:\|\boldsymbol{a}_{i_1}-\boldsymbol{a}^*_{j_1}\|\leq\|\boldsymbol{a}_{i_1}-\boldsymbol{a}^*_{\ell_1}\|,\forall \ell_1\neq j_1\},\\
    \label{eq:Voronoi_cells_level_2}
    \mathcal{V}_{j_2|j_1}\equiv\mathcal{V}_{j_2|j_1}(G)&:=\{i_2\in[k_2]:\|\boldsymbol{\zeta}_{i_2|j_1}-\boldsymbol{\zeta}^*_{j_2|j_1}\|\leq\|\boldsymbol{\zeta}_{i_2|j_1}-\boldsymbol{\zeta}^*_{\ell_2|j_1}\|,\forall \ell_2\neq j_2\},
\end{align}
with $\boldsymbol{\zeta}_{i_2|j_1}:=(\boldsymbol{\omega}_{i_2|j_1},\boldsymbol{\eta}_{j_1i_2},\tau_{j_1i_2},\nu_{j_1i_2})$ and $\boldsymbol{\zeta}^*_{j_2|j_1}:=(\boldsymbol{\omega}^*_{j_2|j_1},\boldsymbol{\eta}^*_{j_2|j_1},\tau^*_{j_1j_2},\nu^*_{j_1j_2})$. Note that when the MLE $\widehat{G}_n$ is sufficiently close to its true counterpart $G_*$, since the value of $k^*_1$ is known, we have $|\mathcal{V}_{j_1}(\widehat{G}_n)|=1$ for any $j_1\in[k^*_1]$, meaning that each parameter $a^*_{j_1}$ is fitted by exactly one parameter. On the other hand, as $k^*_2$ is unknown and we over-specify it by a larger value $k_2$, a Voronoi cell $\mathcal{V}_{j_2|j_1}$ could have more than one element. Furthermore, the cardinality of $\mathcal{V}_{j_2|j_1}$ is exactly the number of fitted parameters converging to $\zeta^*_{j_2|j_1}$. For instance, $|\mathcal{V}_{j_2|j_1}|=2$ indicates that $\zeta^*_{j_2|j_1}$ is fitted by two parameters.
Now, we define a Voronoi loss function based on the Voronoi cells as follows:
\begin{align}
    &\mathcal{L}_{(r_1,r_2,r_3)}(G,G_*):=\sum_{j_1=1}^{k^*_1}\Big|\sum_{i_1\in\mathcal{V}_{j_1}}\exp(b_{i_1})-\exp(b^*_{j_1})\Big|+\sum_{j_1=1}^{k^*_1}\sum_{i_1\in\mathcal{V}_{j_1}}\exp(b_{i_1})\|\Delta\boldsymbol{a}_{i_1j_1}\|\nonumber\\
    +&\sum_{j_1=1}^{k^*_1}\sum_{i_1\in\mathcal{V}_{j_1}}\exp(b_{i_1})\Bigg[\sum_{j_2:|\mathcal{V}_{j_2|j_1}|=1}\sum_{i_2\in\mathcal{V}_{j_2|j_1}}\exp(\beta_{i_2|j_1})\Big(\|\Delta\boldsymbol{\omega}_{i_2j_2|j_1}\|+\|\Delta\boldsymbol{\eta}_{j_1i_2j_2}\|+|\Delta\tau_{j_1i_2j_2}|+|\Delta\nu_{j_1i_2j_2}|\Big)\nonumber\\
    +&\sum_{j_2:|\mathcal{V}_{j_2|j_1}|>1}\sum_{i_2\in\mathcal{V}_{j_2|j_1}}\exp(\beta_{i_2|j_1})\Big(\|\Delta\boldsymbol{\omega}_{i_2j_2|j_1}\|^{2}+\|\Delta\boldsymbol{\eta}_{j_1i_2j_2}\|^{r_1(|\mathcal{V}_{j_2|j_1}|)}+|\Delta\tau_{j_1i_2j_2}|^{r_2(|\mathcal{V}_{j_2|j_1}|)}\nonumber\\
     \label{eq:loss_l1}
    +&|\Delta\nu_{j_1i_2j_2}|^{r_3(|\mathcal{V}_{j_2|j_1}|)}\Big)\Bigg]+\sum_{j_1=1}^{k^*_1}\sum_{i_1\in\mathcal{V}_{j_1}}\exp(b_{i_1})\sum_{j_2=1}^{k^*_2}\Big|\sum_{i_2\in\mathcal{V}_{j_2|j_1}}\exp(\beta_{i_2|j_1})-\exp(\beta^*_{j_2|j_1})\Big|,
\end{align}
where $r_1,r_2,r_3:\mathbb{N}\to\mathbb{N}$ are some integer-valued functions and we denote $\Delta\boldsymbol{a}_{i_1j_1}:=\boldsymbol{a}_{i_1}-\boldsymbol{a}^*_{j_1}$, $\Delta\boldsymbol{\omega}_{i_2j_2|j_1}:=\boldsymbol{\omega}_{i_2|j_1}-\boldsymbol{\omega}_{j_2|j_1}$, $\Delta\boldsymbol{\eta}_{j_1i_2j_2}:=\boldsymbol{\eta}_{j_1i_2}-\boldsymbol{\eta}^*_{j_1j_2}$, $\Delta\tau_{j_1i_2j_2}:=\tau_{j_1i_2}-\tau^*_{j_1j_2}$ and $\Delta\nu_{j_1i_2j_2}:=\nu_{j_1i_2}-\nu^*_{j_1j_2}$. Given the above loss function, we are ready to characterize the convergence behavior of expert estimation in the following theorem.
\begin{theorem}
    \label{theorem:param_rates_SS}
    The following Hellinger lower bounds hold true for any $G\in\mathcal{G}_{k^*_1,k_2}(\Theta)$:
    \begin{align*}
        \bbE_{\bbX}[h(p^{SS}_{G}(\cdot|\bbX),p^{SS}_{G_*}(\cdot|\bbX))]&\gtrsim\mathcal{L}_{(\frac{1}{2}r^{SS},r^{SS},\frac{1}{2}r^{SS})}(G,G_*).
    \end{align*}
   As a result, we obtain that $\mathcal{L}_{(\frac{1}{2}r^{SS},r^{SS},\frac{1}{2}r^{SS})}(\widehat{G}^{SS}_n,G_*)=\widetilde{\mathcal{O}}_P(n^{-1/2})$.
\end{theorem}
\noindent
Proof of Theorem~\ref{theorem:param_rates_SS} is in Section~\ref{appendix:softmax_softmax}. The above results together with the formulation of the Voronoi loss $\mathcal{L}_{(\frac{1}{2}r^{SS},r^{SS},\frac{1}{2}r^{SS})}$ in equation~\eqref{eq:loss_l1} implies that

\vspace{0.5 em}
\noindent
\textbf{(i) Exact-specified parameters:} The rates for estimating exact-specified parameters $\boldsymbol{a}^*_{j_1}$, $\boldsymbol{\omega}^*_{j_2|j_1}$, $\boldsymbol{\eta}^*_{j_1j_2}$, $\tau^*_{j_1j_2}$, $\nu^*_{j_1j_2}$ which are approached by exactly one fitted parameter, i.e. their Voronoi cells have only one element $|\mathcal{V}_{j_1}|=|\mathcal{V}_{j_2|j_1}|=1$, are parametric on the sample size $n$, standing at the order $\widetilde{\mathcal{O}}_P(n^{-1/2})$. Additionally, the gating bias parameters $\exp(b^*_{j_1})$ and $\exp(\beta^*_{j_2|j_1})$ also share the same parametric estimation rates.

\vspace{0.5 em}
\noindent
\textbf{(ii) Over-specified parameters:} For over-specified parameters $\boldsymbol{\omega}^*_{j_2|j_1},\boldsymbol{\eta}^*_{j_1j_2},\tau^*_{j_1j_2},\nu^*_{j_1j_2}$ which are fitted by more than one parameter, i.e. $|\mathcal{V}_{j_2|j_1}|>1$, their estimation rates are not homogeneous. In particular, the rates for estimating $\boldsymbol{\omega}^*_{j_2|j_1}$ are of order $\widetilde{\mathcal{O}}_P(n^{-1/4})$. At the same time, those for $\boldsymbol{\eta}^*_{j_1j_2},\tau^*_{j_1j_2},\nu^*_{j_1j_2}$ depend on their number of fitted parameters $|\mathcal{V}_{j_2|j_1}|$ and the solvability of the polynomial equation system in equation~\eqref{eq:system_SS}, standing at the orders of $\widetilde{\mathcal{O}}_P(n^{-1/r^{SS}(|\mathcal{V}_{j_2|j_1}|)})$, $\widetilde{\mathcal{O}}_P(n^{-1/2r^{SS}(|\mathcal{V}_{j_2|j_1}|)})$, $\widetilde{\mathcal{O}}_P(n^{-1/r^{SS}(|\mathcal{V}_{j_2|j_1}|)})$, respectively. For instance, when $|\mathcal{V}_{j_2|j_1}|=3$, these rates become $\widetilde{\mathcal{O}}_P(n^{-1/6}), \widetilde{\mathcal{O}}_P(n^{-1/12}),\widetilde{\mathcal{O}}_P(n^{-1/6})$, which are significantly slower than those for exact-specified parameters. These slow rates occur due to the interactions mentioned in the ``Challenges'' paragraph.

\vspace{0.5 em}
\noindent
\textbf{(iii) Expert estimation:} Recall that expert specialization is an essential problem where we learn how fast an expert specializes in some tasks or some aspects of the data. Therefore, it is important to understand the convergence behavior of the expert estimation, particularly its data-dependent term $(\boldsymbol{\eta}^*_{j_1j_2})^{\top}\bx$. According to the Cauchy-Schwarz inequality, we have
\begin{align}
    \label{eq:expert_bound}
    \Big|(\hat{\boldsymbol{\eta}}^{SS,n}_{i_1i_2})^{\top}\bx-(\boldsymbol{\eta}^*_{j_1j_2})^{\top}\bx\Big|\leq\|\hat{\boldsymbol{\eta}}^{SS,n}_{i_1i_2}-\boldsymbol{\eta}^*_{j_1j_2}\|\cdot\|\bx\|,
\end{align}
where $\hat{\boldsymbol{\eta}}^{SS,n}_{i_1i_2}$ is an MLE of $\boldsymbol{\eta}^*_{j_1j_2}$.
Since the input space is bounded and
from the estimation rate of $\boldsymbol{\eta}^*_{j_1j_2}$ in the above two remarks, we deduce that $(\boldsymbol{\eta}^*_{j_1j_2})^{\top}\bx$ admits an estimation rate of order $\widetilde{\mathcal{O}}_P(n^{-1/2})$ when $|\mathcal{V}_{j_2|j_1}|=1$ or $\widetilde{\mathcal{O}}_P(n^{-1/r^{SS}(|\mathcal{V}_{j_2|j_1}|)})$ when $|\mathcal{V}_{j_2|j_1}|>1$. Note that the latter rate is significantly slow since the term $r^{SS}(|\mathcal{V}_{j_2|j_1}|)$ grows as the number of fitted experts $|\mathcal{V}_{j_2|j_1}|$ increases.

\subsection{Softmax-Laplace Gating Gaussian HMoE} 
\label{sec:softmax_laplace}
Moving to this section, we study the convergence behavior of parameter and expert estimation under the Softmax-Laplace gating Gaussian HMoE model where we replace the Softmax gating in the second level with the Laplace gating. In particular, the conditional density function in equation~\eqref{eq:softmax_softmax} becomes
\begin{align}
    \label{eq:softmax_laplace}
    p^{SL}_{G_{*}}(y|\bx) &:= \sum_{i_1=1}^{k^*_1} \softmax((\ai)^{\top}\bx+\bi)\sum_{i_2=1}^{k^*_2}\softmax(-\|\oi-\bx\|+\bei) \pi(y|(\ei)^{\top}\bx+\ti,\vi),
\end{align}
where the abbreviation $SL$ stands for ``Softmax-Laplace''. 
The main difference between the density $p^{SL}_{G_{*}}(y|\bx)$ and its counterpart $p^{SS}_{G_{*}}(y|\bx)$ is the Laplace gating function $\softmax(-\|\oi-\bx\|+\bei)$ in the second level. 

\vspace{0.5 em}
\noindent
\textbf{Disappearance of the gating parameter interaction.} Due to the gating change in the second level, the interaction between parameters $\boldsymbol{a}$ and $\boldsymbol{\omega}$ via the PDE $\frac{\partial u}{\partial\boldsymbol{a}}=\frac{\partial u}{\partial\boldsymbol{\omega}}$ in equation~\eqref{eq:PDE} no longer holds true, while others still exist. As a consequence, we only need to consider a simpler (fewer variables) system of polynomial equations than that in equation~\eqref{eq:system_SS}. More specifically, for each $m\geq 2$, we define $r^{SL}(m)$ as the smallest natural number $r$ such that the following system does not have any non-trivial solutions for the unknown variables $(p_{i_2},\boldsymbol{q}_{2},\boldsymbol{q}_{3i_2},q_{4i_2},q_{5i_2})_{i_2=1}^{m}$:
\begin{align}
    \label{eq:system_SL}
    \sum_{i_2=1}^{m}\sum_{(\balpha_2,\balpha_3,\alpha_4,\alpha_5)\in\mathcal{I}^{SL}_{\brho_1,\rho_2}}\frac{1}{\balpha!}\cdot p^2_{i_2}\boldsymbol{q}_{2}^{\balpha_2}\boldsymbol{q}_{3i_2}^{\balpha_3}q_{4i_2}^{\alpha_4}q_{5i_2}^{\alpha_5}=0, \quad 1\leq|\brho_1|+\rho_2\leq r, 
\end{align}
where $\mathcal{I}^{SL}_{\brho_1,\rho_2}:=\{(\balpha_2,\balpha_3,\alpha_4,\alpha_5)\in\mathbb{R}^d\times\mathbb{R}^d\times\mathbb{R}\times\mathbb{R}_+:\balpha_2+\balpha_3=\brho_1, |\balpha_3|+\alpha_4+2\alpha_5=\rho_2\}$. Here, a solution is called non-trivial if all the values of $p_{i_2}$ are different from zero and at least one among $q_{4i_2}$ is non-zero. This system has been considered in \cite{nguyen2023demystifying} where they show that $r^{SL}(2)=4$ and $r^{SL}(3)=6$.
We observe that the function $r^{SL}$ shares the same values with $r^{SS}$ in Lemma~\ref{lemma:rss_values} at some particular points. Nevertheless, it is challenging to make an explicit comparison between these two functions, which requires further technical tools in algebraic geometry \cite{Sturmfels_System} to be developed. 

\vspace{0.5 em}
\noindent
Next, given the density estimation rate $\bbE_{\bbX}[h(p^{SL}_{\widehat{G}^{SL}_{n}}(\cdot|\bbX), p^{SL}_{G_{*}}(\cdot|\bbX))]=\widetilde{\mathcal{O}}_P(n^{-1/2})$ in Proposition~\ref{prop:density_estimation} and the Voronoi loss function $\mathcal{L}_{(\frac{1}{2}r^{SL},r^{SL},\frac{1}{2}r^{SL})}(G,G_*)$ defined in equation~\eqref{eq:loss_l1}, we will establish the convergence of parameter and expert estimation under the Softmax-Laplace gating Gaussian HMoE in Theorem~\ref{theorem:param_rates_SL}.
\begin{theorem}
    \label{theorem:param_rates_SL}
    The following Hellinger lower bounds hold true for any $G\in\mathcal{G}_{k^*_1,k_2}(\Theta)$:
    \begin{align*}
        \bbE_{\bbX}[h(p^{SL}_{G}(\cdot|\bbX),p^{SL}_{G_*}(\cdot|\bbX))]&\gtrsim\mathcal{L}_{(\frac{1}{2}r^{SL},r^{SL},\frac{1}{2}r^{SL})}(G,G_*).
    \end{align*}
   As a result, we obtain that $\mathcal{L}_{(\frac{1}{2}r^{SL},r^{SL},\frac{1}{2}r^{SL})}(\widehat{G}^{SL}_n,G_*)=\widetilde{\mathcal{O}}_P(n^{-1/2})$.
\end{theorem}
\noindent
Proof of Theorem~\ref{theorem:param_rates_SL} is in Section~\ref{appendix:softmax_laplace}. From the above results, it can be observed that the parameter and expert estimation when using the Softmax gating and Laplace gating in the first and second levels of the Gaussian HMoE admit similar convergence behavior as when using the Softmax gating in both levels in Theorem~\ref{theorem:param_rates_SS}. 

\vspace{0.5em}
\noindent
\textbf{(i) Parameter estimation rates:} Exact-specified parameters $\boldsymbol{a}^*_{j_1},\boldsymbol{\omega}^*_{j_2|j_1},\boldsymbol{\eta}^*_{j_1j_2},\tau^*_{j_1j_2},\nu^*_{j_1j_2}$ share the same estimation rate of order $\widetilde{\mathcal{O}}_P(n^{-1/2})$. On the other hand, the convergence rates of estimating over-specified parameters are diverse. More concretely, parameters $\boldsymbol{\omega}^*_{j_2|j_1}$ admit the estimation rate of the order $\widetilde{\mathcal{O}}_P(n^{-1/4})$, while those for  $\boldsymbol{\eta}^*_{j_1j_2},\tau^*_{j_1j_2},\nu^*_{j_1j_2}$ are of the orders $\widetilde{\mathcal{O}}_P(n^{-1/r^{SL}(|\mathcal{V}_{j_2|j_1}|)})$, $\widetilde{\mathcal{O}}_P(n^{-1/2r^{SL}(|\mathcal{V}_{j_2|j_1}|)})$, $\widetilde{\mathcal{O}}_P(n^{-1/r^{SL}(|\mathcal{V}_{j_2|j_1}|)})$, respectively. Note that since the last three rates hinge upon the solvability of the system~\eqref{eq:system_SL} and the cardinalities of Voronoi cells $\mathcal{V}_{j_2|j_1}$, they will become increasingly slow when the value of $|\mathcal{V}_{j_2|j_1}|$ increases, e.g., $\widetilde{\mathcal{O}}_P(n^{-1/6})$, $\widetilde{\mathcal{O}}_P(n^{-1/12})$, $\widetilde{\mathcal{O}}_P(n^{-1/6})$ when $|\mathcal{V}_{j_2|j_1}|=3$.

\vspace{0.5em}
\noindent
\textbf{(ii) Expert estimation rates:} By arguing analogously to equation~\eqref{eq:expert_bound}, it follows that the data-dependent term of expert $(\boldsymbol{\eta}^*_{j_1j_2})^{\top}\bx$ has an estimation rate of order $\widetilde{\mathcal{O}}_P(n^{-1/2})$ when $|\mathcal{V}_{j_2|j_1}|=1$ or $\widetilde{\mathcal{O}}_P(n^{-1/r^{SL}(|\mathcal{V}_{j_2|j_1}|)})$ when $|\mathcal{V}_{j_2|j_1}|>1$. Thus, we can see that substituting the Softmax gating with the Laplace gating in the second level is insufficient to accelerate the expert estimation rate (see Table~\ref{table:expert_rates}). This is because the interaction $\frac{\partial u}{\partial\boldsymbol{\eta}}=\frac{\partial^2u}{\partial\boldsymbol{a}\partial\tau}$ between $\boldsymbol{\eta}$ and other parameters mentioned in equation~\eqref{eq:PDE} still holds under the setting of Softmax-Laplace gating Gaussian HMoE.

\subsection{Laplace-Laplace Gating Gaussian HMoE}
\label{sec:laplace_laplace}
In this section, we consider the Laplace-Laplace gating Gaussian HMoE where we employ the Laplace gating in both levels of the model. More specifically, the conditional density function in equation~\eqref{eq:softmax_laplace} turns into
\begin{align}
    \label{eq:laplace_laplace}
    p^{LL}_{G_{*}}(y|\bx) &:= \sum_{i_1=1}^{k^*_1} \softmax(-\|\ai-\bx\|+\bi)\sum_{i_2=1}^{k^*_2}\softmax(-\|\oi-\bx\|+\bei)\pi(y|(\ei)^{\top}\bx+\ti,\vi),
\end{align}
where the abbreviation $LL$ stands for ``Laplace-Laplace''. 

\vspace{0.5 em}
\noindent
\textbf{Benefits of the Laplace gating over the Softmax gating.} Under this setting, the first-level Softmax gating $\softmax((\ai)^{\top}\bx+\bi)$ used in previous sections is replaced with the Laplace gating $\softmax(-\|\ai-\bx\|+\bi)$, leading to the disappearance of the interaction $\frac{\partial u}{\partial\boldsymbol{\eta}}=\frac{\partial^2u}{\partial\boldsymbol{a}\partial\tau}$ between $\boldsymbol{\eta}$ and other parameters mentioned in equation~\eqref{eq:PDE}. Therefore, we only need to cope with the parameter interaction $\frac{\partial u}{\partial\nu}=\frac{1}{2}\cdot\frac{\partial^2u}{\partial\tau^2}$ as in \cite{ho2022gaussian}. Consequently, it is sufficient to take account of the following system of polynomial equations with substantially fewer variables than those in equations~\eqref{eq:system_SS} and \eqref{eq:system_SL}. In particular, for each $m\geq 2$, we define $r^{LL}(m)$ as the smallest natural number $r$ such that the following system does not have any non-trivial solutions for the unknown variables $(p_{i_2},q_{4i_2},q_{5i_2})_{i_2=1}^{m}$:
\begin{align}
    \label{eq:system_LL}
    \sum_{i_2=1}^{m}\sum_{(\alpha_4,\alpha_5)\in\mathcal{I}^{LL}_{\rho}}\frac{1}{\balpha!}\cdot p^2_{i_2}q_{4i_2}^{\alpha_4}q_{5i_2}^{\alpha_5}=0, \quad 1\leq\rho\leq r, 
\end{align}
where $\mathcal{I}^{LL}_{\rho}:=\{(\alpha_4,\alpha_5)\in\mathbb{R}\times\mathbb{R}_+:\alpha_4+2\alpha_5=\rho\}$. Here, a solution is called non-trivial if all the values of $p_{i_2}$ are different from zero and at least one among $q_{4i_2}$ is non-zero. The above system has been studied in \cite{Ho-Nguyen-Ann-16} which show that $r^{LL}(2)=4$ and $r^{LL}(3)=6$. These values are similar to those of the aforementioned functions $r^{SS}$ and $r^{SL}$. 

\vspace{0.5 em}
\noindent
As demonstrated in Appendix~\ref{appendix:density_rate}, we also obtain the convergence rate of density estimation $\bbE_{\bbX}[h(p^{LL}_{\widehat{G}^{LL}_{n}}(\cdot|\bbX), p^{LL}_{G_{*}}(\cdot|\bbX))]=\widetilde{\mathcal{O}}_P(n^{-1/2})$ under this setting. Given that result and the Voronoi loss function $\mathcal{L}_{(2,r^{LL},\frac{1}{2}r^{LL})}(G,G_*)$ defined in equation~\eqref{eq:loss_l1}, we are ready to investigate the impacts of using the Laplace gating in both levels on the convergence behavior of parameter and expert estimation in the below theorem. 
\begin{theorem}
    \label{theorem:param_rates_LL}
    The following Hellinger lower bounds hold true for any $G\in\mathcal{G}_{k^*_1,k_2}(\Theta)$:
    \begin{align*}
        \bbE_{\bbX}[h(p^{LL}_{G}(\cdot|\bbX),p^{LL}_{G_*}(\cdot|\bbX))]&\gtrsim\mathcal{L}_{(2,r^{LL},\frac{1}{2}r^{LL})}(G,G_*).
    \end{align*}
   As a result, we obtain that $\mathcal{L}_{(2,r^{LL},\frac{1}{2}r^{LL})}(\widehat{G}^{LL}_n,G_*)=\widetilde{\mathcal{O}}_P(n^{-1/2})$.
\end{theorem}
\noindent
The proof of Theorem~\ref{theorem:param_rates_LL} can be found in Section~\ref{appendix:laplace_laplace}. From the formulation of the loss function $\mathcal{L}_{(2,r^{LL},\frac{1}{2}r^{LL})}$ in equation~\eqref{eq:loss_l1}, we have two following critical observations:

\vspace{0.5em}
\noindent
\textbf{(i) Parameter estimation rates:} All parameter estimations share the same convergence behavior as those under the previous two settings, except for the estimations of parameters $\boldsymbol{\eta}^*_{j_1j_2}$ which enjoy a convergence rate of order $\widetilde{\mathcal{O}}_P(n^{-1/2})$ when $|\mathcal{V}_{j_2|j_1}|=1$ and $\widetilde{\mathcal{O}}_P(n^{-1/4})$ when $|\mathcal{V}_{j_2|j_1}|>1$. It is worth noting that these rates are faster than their counterparts in Sections~\ref{sec:softmax_softmax} and \ref{sec:softmax_laplace} as they no longer depend on the solvability of any equation system. 

\begin{table*}[t!]
\caption{Summary of estimation rates for the data-dependent term $(\boldsymbol{\eta}^*_{j_1j_2})^{\top}\bx$ in experts. Experts are called exact-specified when $|\mathcal{V}_{j_2|j_1}|=1$ and over-specified when $|\mathcal{V}_{j_2|j_1}|>1$.}
\renewcommand\arraystretch{1.5}
\centering
\scalebox{1}{
\begin{tabular}{ | c | c |c|c|c|c|} 
\hline
{\textbf{}} & {\bf Softmax-Softmax} & {\bf Softmax-Laplace}& {\bf Laplace-Laplace}  \\
\hline 
Exact-specified experts &$\widetilde{\mathcal{O}}_P(n^{-1/2})$ &$\widetilde{\mathcal{O}}_P(n^{-1/2})$  & $\widetilde{\mathcal{O}}_P(n^{-1/2})$ \\
\hline
Over-specified experts&$\widetilde{\mathcal{O}}_P(n^{-1/r^{SS}(|\mathcal{V}_{j_2|j_1}|)})$ &$\widetilde{\mathcal{O}}_P(n^{-1/r^{SL}(|\mathcal{V}_{j_2|j_1}|)})$  & $\boldsymbol{\widetilde{\mathcal{O}}_P(n^{-1/4})}$ \\
\hline
\end{tabular}}
\label{table:expert_rates}
\end{table*}

\begin{table*}[t!]
\caption{Summary of estimation rates for over-specified parameters $\boldsymbol{\omega}^*_{j_2|j_1}$, $\boldsymbol{\eta}^*_{j_1j_2}$, $\tau^*_{j_1j_2}$, and $\nu^*_{j_1j_2}$. Meanwhile, exact-specified parameters $\boldsymbol{a}^*_{j_1}$, $\boldsymbol{\omega}^*_{j_2|j_1}$, $\boldsymbol{\eta}^*_{j_1j_2}$, $\tau^*_{j_1j_2}$, and $\nu^*_{j_1j_2}$ share the same estimation rate of order $\widetilde{\mathcal{O}}_P(n^{-1/2})$.}
\centering
\scalebox{0.95}{
\begin{tabular}{ | c | c |c|c|} 
\hline
{\textbf{}} & {\bf Softmax-Softmax} & {\bf Softmax-Laplace}& {\bf Laplace-Laplace}  \\
\hline
$\boldsymbol{\omega}^*_{j_2|j_1}$& $\widetilde{\mathcal{O}}_P(n^{-1/4})$ & $\widetilde{\mathcal{O}}_P(n^{-1/4})$  & $\widetilde{\mathcal{O}}_P(n^{-1/4})$ \\
\hline
$\boldsymbol{\eta}^*_{j_1j_2}$&$\widetilde{\mathcal{O}}_P(n^{-1/r^{SS}(|\mathcal{V}_{j_2|j_1}|)})$ &$\widetilde{\mathcal{O}}_P(n^{-1/r^{SL}(|\mathcal{V}_{j_2|j_1}|)})$  & $\widetilde{\mathcal{O}}_P(n^{-1/4})$ \\
\hline
$\tau^*_{j_1j_2}$&$\widetilde{\mathcal{O}}_P(n^{-1/2r^{SS}(|\mathcal{V}_{j_2|j_1}|)})$ &$\widetilde{\mathcal{O}}_P(n^{-1/2r^{SL}(|\mathcal{V}_{j_2|j_1}|)})$  & $\widetilde{\mathcal{O}}_P(n^{-1/2r^{LL}(|\mathcal{V}_{j_2|j_1}|)})$ \\
\hline
$\nu^*_{j_1j_2}$&$\widetilde{\mathcal{O}}_P(n^{-1/r^{SS}(|\mathcal{V}_{j_2|j_1}|)})$ &$\widetilde{\mathcal{O}}_P(n^{-1/r^{SL}(|\mathcal{V}_{j_2|j_1}|)})$  & $\widetilde{\mathcal{O}}_P(n^{-1/r^{LL}(|\mathcal{V}_{j_2|j_1}|)})$ \\
\hline
\end{tabular}}
\label{table:parameter_rates}
\end{table*}

\vspace{0.5em}
\noindent
\textbf{(ii) Expert estimation rates:} By employing the same arguments as in equation~\eqref{eq:expert_bound}, we deduce that the data-dependent terms of experts $(\boldsymbol{\eta}^*_{j_1j_2})^{\top}\bx$ also admit the same estimation rates as $\boldsymbol{\eta}^*_{j_1j_2}$, that is, $\widetilde{\mathcal{O}}_P(n^{-1/2})$ when $|\mathcal{V}_{j_2|j_1}|=1$ and $\widetilde{\mathcal{O}}_P(n^{-1/4})$ when $|\mathcal{V}_{j_2|j_1}|>1$. Compared to those when using the Softmax gating in either level or both levels of the Gaussian HMoE, the expert estimation rates when using the Laplace gating in both levels are improved significantly, as they no longer depend on the term $r^{LL}(|\mathcal{V}_{j_2|j_1}|)$ (see Table~\ref{table:expert_rates}). This acceleration occurs since the interaction $\frac{\partial u}{\partial\boldsymbol{\eta}}=\frac{\partial^2u}{\partial\boldsymbol{a}\partial\tau}$ between $\boldsymbol{\eta}$ and other parameters mentioned in equation~\eqref{eq:PDE} disappear under this setting. As a result, we claim that the convergence of expert estimation under the two-level Gaussian HMoE is benefited the most when equipped with the Laplace gating in both levels.


\subsection{Summary of Main Theoretical Findings}
\label{sec:practical_implications}
In this section, we summarize the key findings from our convergence analysis of parameter estimation and expert estimation under three types of the Gaussian HMoE model in Sections~\ref{sec:softmax_softmax}, \ref{sec:softmax_laplace} and \ref{sec:laplace_laplace}:

\vspace{0.5 em}
\noindent
\textbf{1. Softmax-Softmax Gating Gaussian HMoE:} Using the Softmax gating in both levels of the Gaussian HMoE model induces  parameter interactions between the first-level gating parameter $\boldsymbol{a}$ with not only the second-level expert parameters $\boldsymbol{\eta},\tau$ but also the second-level gating parameters $\boldsymbol{\omega}$ through the PDEs in equation~\eqref{eq:PDE}. As a result, the convergence rates of estimating the over-specified parameters and experts hinge upon the solvability of a complex system of polynomial equations, which are significantly slow.

\vspace{0.5 em}
\noindent
\textbf{2. Softmax-Laplace Gating Gaussian HMoE:} When replacing the Softmax gating with the Laplace gating in the second level of the Gaussian HMoE model, the gating parameter in the first level $\boldsymbol{a}$ does not interact with the second-level gating parameter $\boldsymbol{\omega}$. However, since the interaction between $\boldsymbol{a}$ and the second-level expert parameters $\boldsymbol{\eta},\tau$ still holds true, our theory indicates that the disappearance of the gating parameter interaction only helps slightly reduce the complexity of the polynomial equation system but not improve the convergence rates of parameter estimation and expert estimation substantially. 

\vspace{0.5 em}
\noindent
\textbf{3. Laplace-Laplace Gating Gaussian HMoE:} By employing the Laplace gating in both levels of the Gaussian HMoE model, we observe that the interactions of the first-level gating parameter $\boldsymbol{a}$ with both the second-level gating parameters $\boldsymbol{\omega}$ and expert parameters $\boldsymbol{\eta},\tau$ no longer exist. Consequently, the convergence rate of expert estimation is considerably accelerated and becomes independent of the previous systems of polynomial equations. Hence, our theory suggests that the combination of Laplace gating in both levels of the Gaussian HMoE model is optimal for the expert convergence.

\section{Experiments}\label{sec:exp}

\definecolor{grn}{RGB}{27, 129, 62}
\begin{algorithm}[t]
\small
\caption{Computation Procedure for the 2-Level Hierarchical MoE Module}
\label{alg:hmoe} 
\begin{algorithmic}[1]
\STATE \textbf{Input}: $\mathbf{x} \in \mathbb{R}^{B\times N \times D}$; batch size $B$, sequence length $N$, embedding dimension $D$, number of outer/inner experts $E_o/E_i$, capacity per outer/inner expert $\mathcal{C}_o, \mathcal{C}_i$, dispatch tensor $\mathbf{D}$, combine tensor $\mathbf{C}$ 
\STATE $\mathbf{D}_{o}, \mathbf{C}_{o}, \mathbf{L}_{o} = \mathsf{Gate_{outer}}(\mathbf{x})$ ~\textcolor{grn}{$\rhd$ compute outer dispatch, outer combine tensors, and outer gating loss}
\STATE $\mathbf{x}_{\mathsf{outer}}^{(e,b,c,d)} = \sum_{n} \mathbf{D}_o^{(b,n,e,c)} \cdot \mathbf{x}^{(b,n,d)}$ ~\textcolor{grn}{$\rhd$ dispatch inputs to outer experts using dispatch tensor}
\STATE $\mathbf{D}_{i}, \mathbf{C}_{i}, \mathbf{L}_{i} = \mathsf{Gate_{inner}}(\mathbf{x}_{\mathsf{outer}})$ ~\textcolor{grn}{$\rhd$ compute inner dispatch, inner combine tensors, and inner gating loss}
\STATE $\mathbf{x}_{\mathsf{experts}}^{(e_o, e_i, b, c_i, d)} = \sum_{c_o} \mathbf{D}_i^{(e_o, b, c_o, e_i, c_i)} \cdot \mathbf{x}_{\mathsf{outer}}^{(e_o,b,c_o,d)}$ ~\textcolor{grn}{$\rhd$ dispatch inputs to the inner experts}
\STATE $\mathbf{y}_{\mathsf{experts}} = \mathsf{Experts}(\mathbf{x}_{\mathsf{experts}})$ ~\textcolor{grn}{$\rhd$ expert processing}
\STATE $\mathbf{y}_{\mathsf{outer}}^{(e_o,b,n,d)} = \sum_{e_i, c_i} \mathbf{C}_i^{(e_o, b, c_o, e_i, c_i)} \cdot \mathbf{y}_{\mathsf{experts}}^{(e_o, e_i, b, c_i, d)}$~\textcolor{grn}{$\rhd$ combine inner expert outputs}
\STATE $\mathbf{y}^{(b,n,d)} = \sum_{e,c}\mathbf{C}_o^{(b,n,e,c)} \cdot \mathbf{y}_{\mathsf{outer}}^{(e,b,c,d)}$ ~\textcolor{grn}{$\rhd$ combine outer expert outputs}
\STATE $\mathcal{L} = \lambda (\mathcal{L}_o + \mathcal{L}_i)$ ~\textcolor{grn}{$\rhd$ compute total loss}
\STATE \textbf{Return}: $\mathbf{y}, \mathcal{L}$
\end{algorithmic}
\end{algorithm}
\noindent
In this section, we empirically demonstrate the effects of employing various combinations of gating functions in HMoE to validate our theoretical findings and discuss empirical insights. 
We conduct a comprehensive empirical analysis of hierarchical gating mechanisms and perform case studies across various applications. Besides, we show that HMoE outperforms standard MoE and other alternatives, particularly in cases with inherent subgroups or multilevel structures, where HMoE excels.
Beyond performance improvements, these experiments provide valuable insights into how different gating function combinations influence the distribution of input modules, offering explanations for the performance variations observed with different gating configurations.
\begin{figure*}[t!]
    \begin{minipage}{\textwidth}
    \centering
    \begin{tabular}{@{\hspace{-2.8ex}} c @{\hspace{-1.4ex}} c @{\hspace{-1.5ex}} c @{\hspace{-1.5ex}}}
        \begin{tabular}{c}
        \includegraphics[width=.7\textwidth]{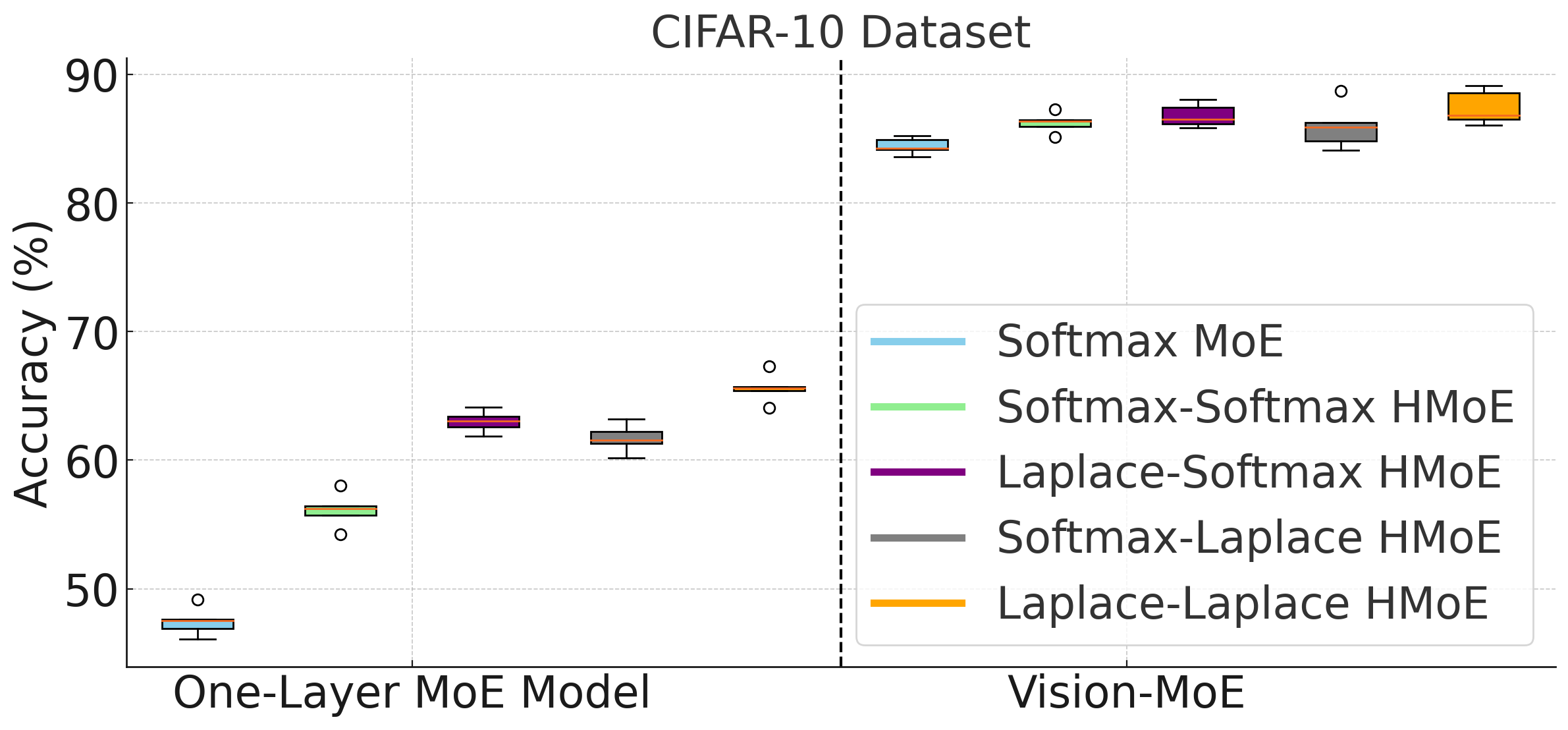}
        \\
        {\small{(a)}}
        \end{tabular} & \\
        \begin{tabular}{c}
        \includegraphics[width=.7\textwidth]{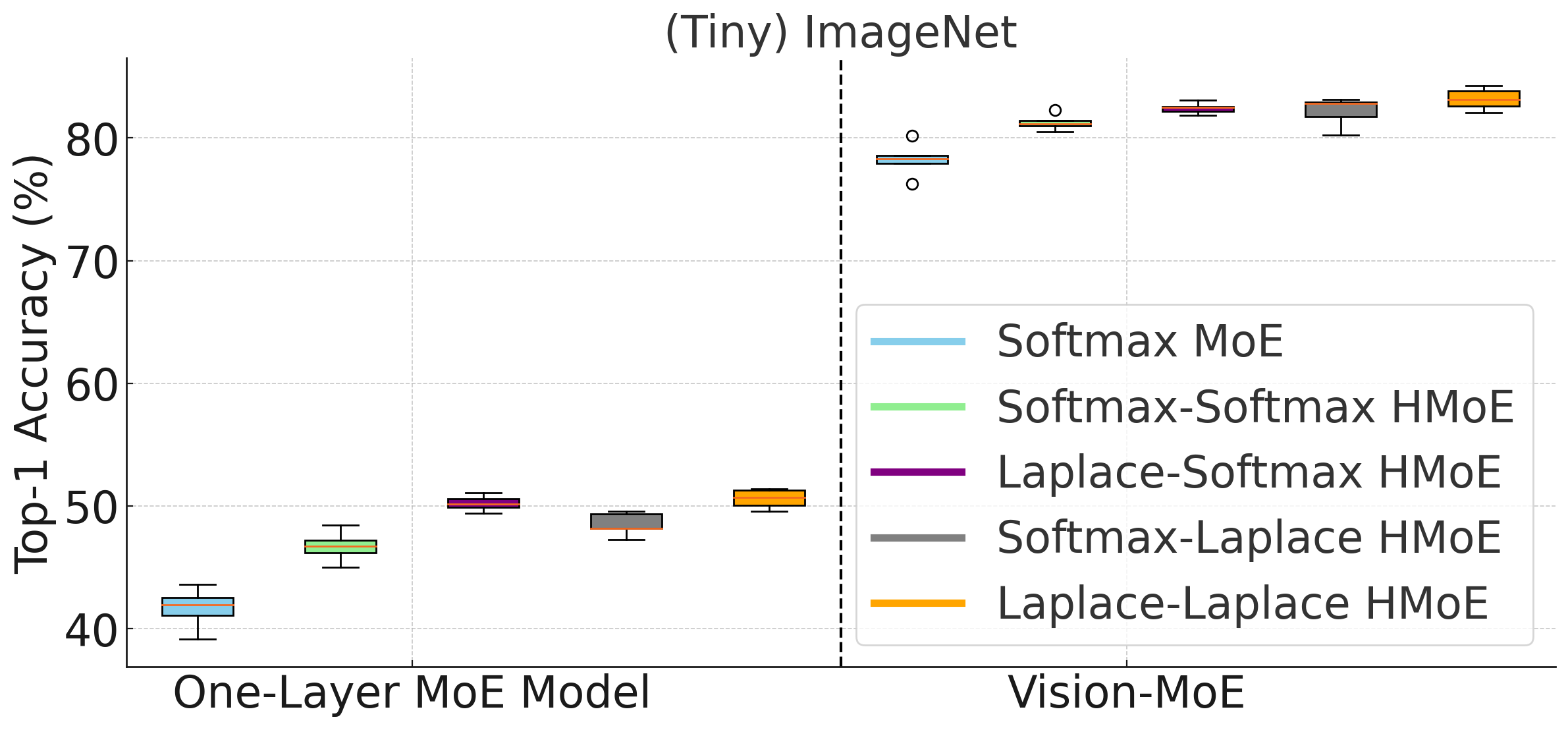} 
        \\
        {\small{(b)}}
        \end{tabular} \\
        \begin{tabular}{c}
        \includegraphics[width=.7\textwidth]{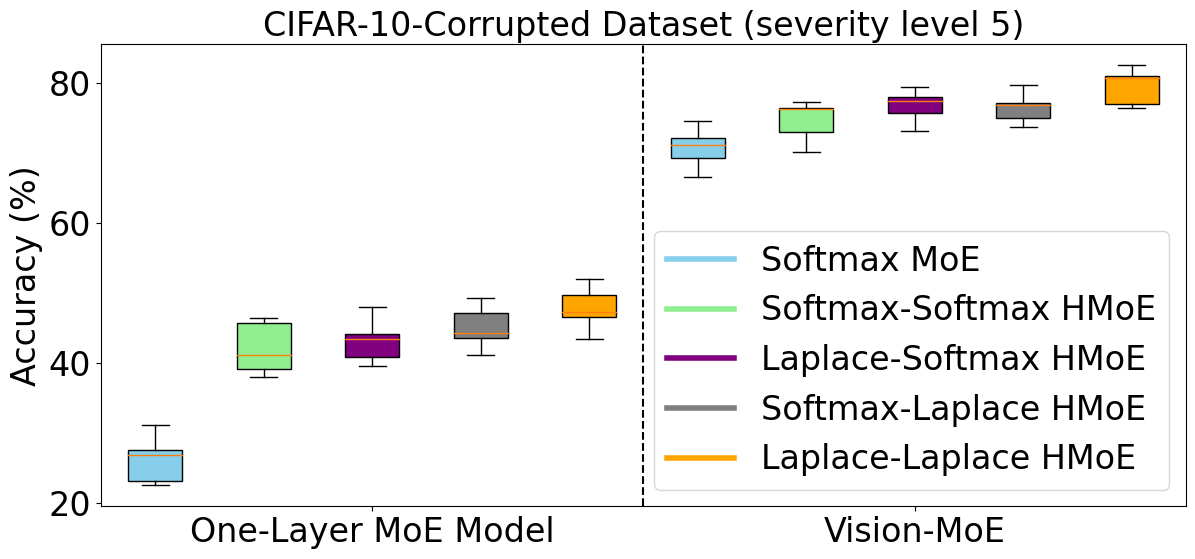} 
        \\
        {\small{(c)}}
        \end{tabular} \\
        \end{tabular}
    \end{minipage}
    \caption{\small We evaluate the impact of using different gating function combinations in HMoE and compare it with standard MoE on (a) CIFAR-10, (b) ImageNet, and (c) CIFAR-10-Corrupted. First, we present the results of one-layer MoE models (left side of each figure), where the model contains only the module of that specific setting. For the one-layer results, we use Tiny-ImageNet as a substitute for the full ImageNet. Next, we integrate these MoE modules into the state-of-the-art Vision MoE model (right) \cite{riquelme2021scaling} and compare the performance on the full datasets.}
    \label{fig:image_results}
\end{figure*}

\vspace{0.5 em}
\noindent
\textbf{HMoE Implementation.}
We implement the two-level HMoE module, drawing on the work of \cite{lepikhin2020gshard}. Algorithm \ref{alg:hmoe} outlines the procedure, which uses a recursive computation strategy to process inputs from coarse to fine. First, the inputs are partitioned by the outer dispatcher (Step 2), and then further subdivided by the inner dispatcher (Step 4). These subgroups are directed to specialized groups and experts for independent processing, based on the Top-$k$ routing mechanism with a specified gating function. In particular, each level's choice of gating functions can strongly influence how the inputs are partitioned. The outputs from the experts are then recursively combined using inner and outer combination tensors to form the final output. Gating losses from both levels are integrated and scaled to regularize training, ensuring balanced expert utilization.

\subsection{Comparison of Different Hierarchical Gating Mechanisms}
Figure \ref{fig:image_results} compares the performance of different gating function combinations on the CIFAR-10 \cite{krizhevsky2009learning} and ImageNet \cite{deng2009imagenet} datasets. We first evaluate a single module (i.e., a one-layer MoE model) on CIFAR-10 and Tiny-ImageNet, followed by integrating these modules into the Vision-MoE framework \cite{riquelme2021scaling}: in the Vision Transformer (ViT) models, we selectively replace an even number of FFN layers with targeted MoE layers and test the models on the full datasets. The performance gap between different gating functions is more pronounced in the one-layer MoE models due to the amplified effect of the module differences, while the difference becomes smaller after incorporating them into Vision MoE. The results show that (1) HMoE can noticeably improve the performance of standard MoE; (2) the Laplace-Laplace gating combination achieves the best performance, while the combination of Laplace and Softmax gating also improves the results over pure Softmax-gating HMoE. 

\vspace{0.5 em}
\noindent
\textbf{Generalization to Out-of-Distribution Data.}
We further evaluate HMoE’s robustness to out-of-distribution (OOD) data by applying the same pipeline on the CIFAR-10-corrupted dataset \cite{hendrycks2019benchmarking}. The models are trained on the original clean data and then tested on corrupted variants. To better control the level of distribution shift, we combine clean and corrupted samples in the test set using self-defined mixture ratios. Figure \ref{fig:image_results} (c) presents the results, averaged over five random seeds and 20 corruption types. Specifically, we mix $50\%$ of brightness-type corruptions at severity level 5 with clean samples in the test set. Under this setting, HMoE shows a greater performance advantage over standard MoE. We also observe a trend consistent with our clean-data experiments regarding the impact of different gating-function combinations. This advantage stems from HMoE’s hierarchical structure, which partitions the input space more finely, promoting better expert specialization and thus improved OOD robustness.
For both experiments, the standard Softmax MoE uses 8 experts, while HMoE employs 2 groups with 4 experts each, ensuring both methods have the same overall capacity.

\begin{table}[ht]
\centering
\caption{\small Comparison of HMoE-based fusion methods (gray) and baselines, utilizing vital signs, clinical notes, and CXR from the MIMIC ecosystem. The best results are highlighted in \textbf{bold font}, and the second-best results are \underline{underlined}. All results are averaged across 5 random experiments.}
\label{tab:mimic_4_mods}
\renewcommand\arraystretch{1.8}
\scalebox{0.72}{
\begin{tabular}{c|c|ccccccc}
\toprule
\textbf{Task} & \textbf{Metric} & \textbf{HAIM} & \textbf{MISTS} & \textbf{MoE} & \cellcolor{ashgrey}\textbf{HMoE-SS} & \cellcolor{ashgrey}\textbf{HMoE-SL} & \cellcolor{ashgrey}\textbf{HMoE-LS} & \cellcolor{ashgrey}\textbf{HMoE-LL} \\
\midrule
\multirow{2}{*}{48-IHM} & AUROC 
& $78.87 \pm 0.00$ & $77.23 \pm 0.82$ & $83.13 \pm 0.36$ & $85.59 \pm 0.44$ & $86.41 \pm 0.38$ & \underline{$86.52 \pm 0.42$} & $\mathbf{87.49 \pm 0.27}$ \\
& F1 
& $39.78 \pm 0.00$ & $45.98 \pm 0.49$ & $46.82 \pm 0.28$ & $47.57 \pm 0.32$ & $47.65 \pm 0.23$ & \underline{$47.73 \pm 0.28$} & $\mathbf{47.91 \pm 0.34}$ \\
\midrule
\multirow{2}{*}{LOS} & AUROC 
& $82.46 \pm 0.00$ & $80.34 \pm 0.61$ & $83.76 \pm 0.59$ & $86.26 \pm 0.61$ & \underline{$86.37 \pm 0.55$} & $86.22 \pm 0.74$ & $\mathbf{86.45 \pm 0.48}$ \\
& F1 
& $72.75 \pm 0.00$ & $73.22 \pm 0.43$ & $74.32 \pm 0.44$ & $76.07 \pm 0.29$ & \underline{$76.23 \pm 0.32$} & $75.79 \pm 0.28$ & $\mathbf{77.31 \pm 0.37}$ \\
\midrule
\multirow{2}{*}{25-PHE} & AUROC 
& $63.57 \pm 0.00$ & $71.49 \pm 0.59$ & $73.87 \pm 0.71$ & $73.81 \pm 0.51$ & $\mathbf{74.59 \pm 0.47}$ & $74.31 \pm 0.62$ & \underline{$74.54 \pm 0.53$} \\
& F1 
& $\mathbf{42.80 \pm 0.00}$ & $33.29 \pm 0.23$ & \underline{$35.96 \pm 0.23$} & $35.64 \pm 0.18$ & $35.88 \pm 0.31$ & $35.72 \pm 0.24$ & $35.92 \pm 0.19$ \\
\bottomrule
\end{tabular}}
\end{table}

\subsection{Laplace Gating Mechanism Improves Multimodal Fusion}
\paragraph{The MIMIC Ecosystem} We evaluate the combination of Laplace gating and HMoE using the MIMIC ecosystem—a comprehensive database that includes records from nearly $300$k patients admitted to a medical center between 2008 and 2019—focusing on a subset of 73,181 ICU stays. We integrated multiple patient modalities, including vital signs (time series) and clinical notes from MIMIC-IV \cite{johnson2020mimic}, and chest X-ray images from MIMIC-CXR \cite{johnson2019mimicjpg}. These modalities are linked via corresponding patient IDs, creating a multimodal input for each patient sample. Our tasks of interest include 48-hour in-hospital mortality prediction (48-IHM), 25-type phenotype classification (25-PHE), and length-of-stay (LOS) prediction. The baselines include: (1) the HAIM data pipeline \cite{soenksen2022integrated}, specifically designed for integrating multimodal data from MIMIC-IV; (2) MISTS, a cross-attention fusion approach combined with irregular sequence modeling for multimodal EHR \cite{zhang2023improving}; and (3) multimodal fusion using MoE \cite{han2024fusemoe}. We implement the HMoE-based fusion approach following \cite{han2024fusemoe}. First, the data is processed by modality-specific encoders. The resulting modality embeddings are then fed into 12 stacked HMoE modules with residual connections to generate the final outcome. Detailed descriptions of these building blocks are provided in the appendix.
Table \ref{tab:mimic_4_mods} summarizes the performance of integrating time series, clinical notes, and CXR data across multiple prediction tasks. HMoE-LL (Laplace-Laplace) outperforms most baselines by a substantial margin. Note that the HAIM approach \cite{soenksen2022integrated} uses simple feature extractors as modality encoders and straightforwardly concatenates modality embeddings for prediction, resulting in no randomness. While the MoE-based fusion method \cite{han2024fusemoe} has demonstrated effectiveness for multimodal fusion, the hierarchical nature of the HMoE module further enhances its ability to handle multimodal inputs, enabling more specialized expert assignments and improved performance.

\begin{table}[ht]
\centering
\caption{Comparison of HMoE-based fusion methods (shown in gray) and baselines on the CMU-MOSI dataset, a multimodal sentiment analysis task leveraging text, video, and audio. Results are averaged across 5 random experiments.}
\label{tab:cmu_mosi}
\begin{tabular}{c|c|c|c|c}
\toprule
\textbf{Method / Metric} & \textbf{MAE}$\downarrow$ & \textbf{Acc-2}$\uparrow$ & \textbf{Corr}$\uparrow$ & \textbf{F1}$\uparrow$ \\
\midrule
TFN & $0.90 \pm 0.02$ & $80.81 \pm 0.34$ & $0.70 \pm 0.04$ & $80.70 \pm 0.18$ \\
MulT & $0.86 \pm 0.01$ & $84.10 \pm 0.21$ & $0.71 \pm 0.02$ & $83.90 \pm 0.27$ \\
MAG & $0.71 \pm 0.04$ & $86.10 \pm 0.44$ & $0.80 \pm 0.03$ & $86.00 \pm 0.09$ \\
Softmax-MoE & $0.67 \pm 0.01$ & $87.28 \pm 0.18$ & $0.82 \pm 0.02$ & $87.29 \pm 0.22$ \\
\cellcolor{ashgrey}Softmax-Softmax HMoE & $0.61 \pm 0.02$ & $89.31 \pm 0.13$ & $0.82 \pm 0.03$ & $87.83 \pm 0.14$ \\
\cellcolor{ashgrey}Softmax-Laplace HMoE & \underline{$0.58 \pm 0.01$} & \underline{$89.75 \pm 0.22$} & \underline{$0.83 \pm 0.05$} & \underline{$88.02 \pm 0.10$} \\
\cellcolor{ashgrey}Laplace-Softmax HMoE & $0.61 \pm 0.01$ & $89.34 \pm 0.24$ & $0.82 \pm 0.02$ & $87.74 \pm 0.07$ \\
\cellcolor{ashgrey}Laplace-Laplace HMoE & $\mathbf{0.56 \pm 0.01}$ & $\mathbf{90.27 \pm 0.17}$ & $\mathbf{0.84 \pm 0.03}$ & $\mathbf{88.36 \pm 0.15}$ \\
\bottomrule
\end{tabular}
\end{table}

\begin{figure*}[t!]
    \begin{minipage}{\textwidth}
    \centering
    \begin{tabular}{@{\hspace{-3.8ex}} c @{\hspace{-4.9ex}} c @{\hspace{-4.9ex}} c @{\hspace{-1.5ex}}}
        \begin{tabular}{c}
        \includegraphics[width=.37\textwidth]{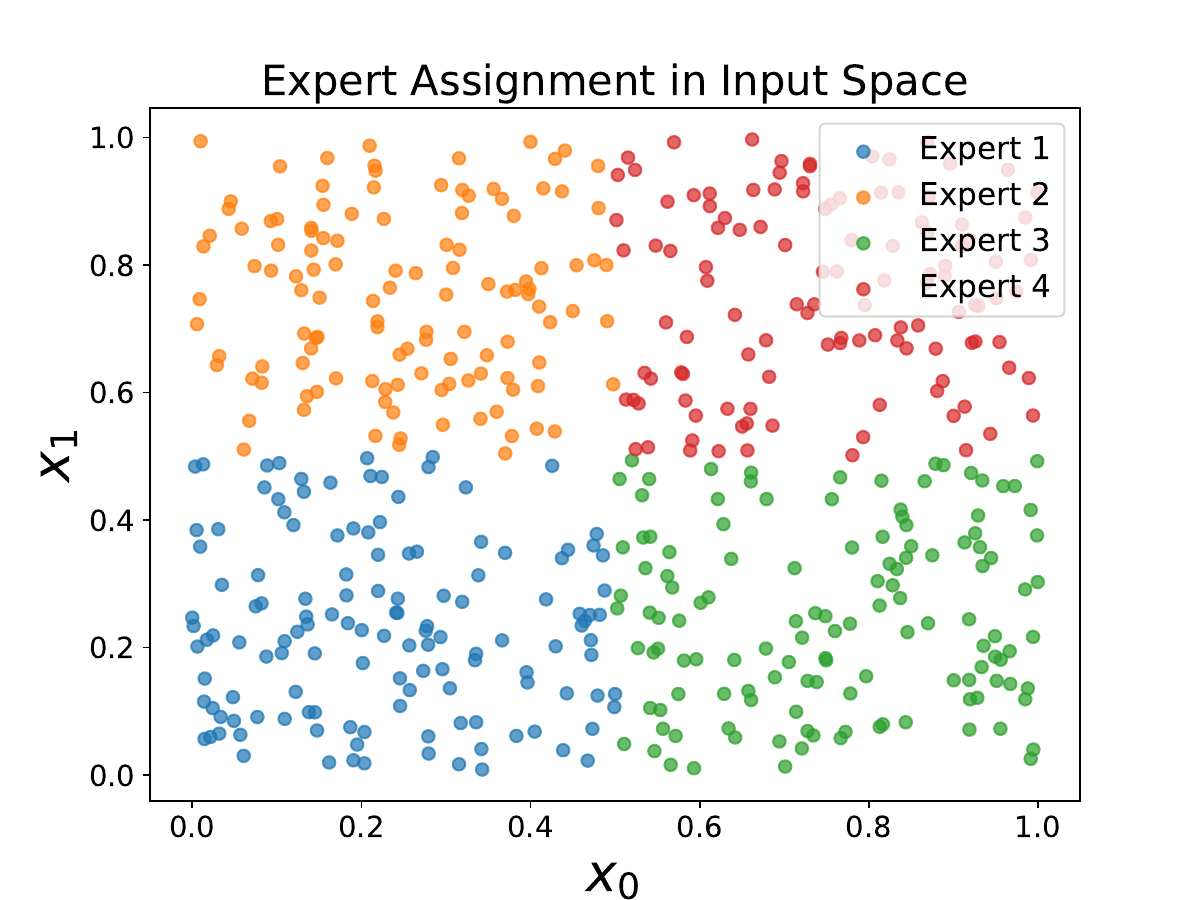}
        \\
        {\small{(a)}}
        \end{tabular} & 
        \begin{tabular}{c}
        \includegraphics[width=.37\textwidth]{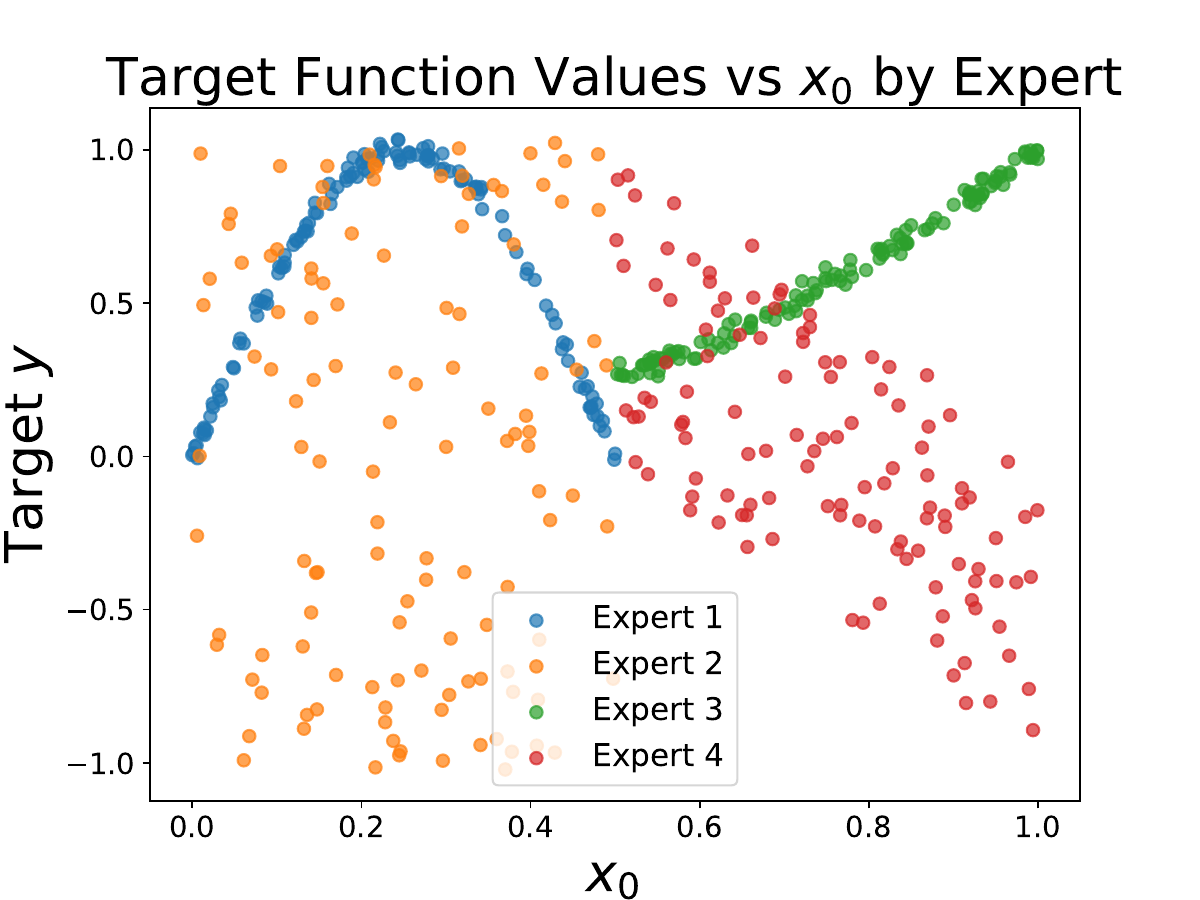} 
        \\
        {\small{(b)}}
        \end{tabular} &
        \begin{tabular}{c}
        \includegraphics[width=.37\textwidth]{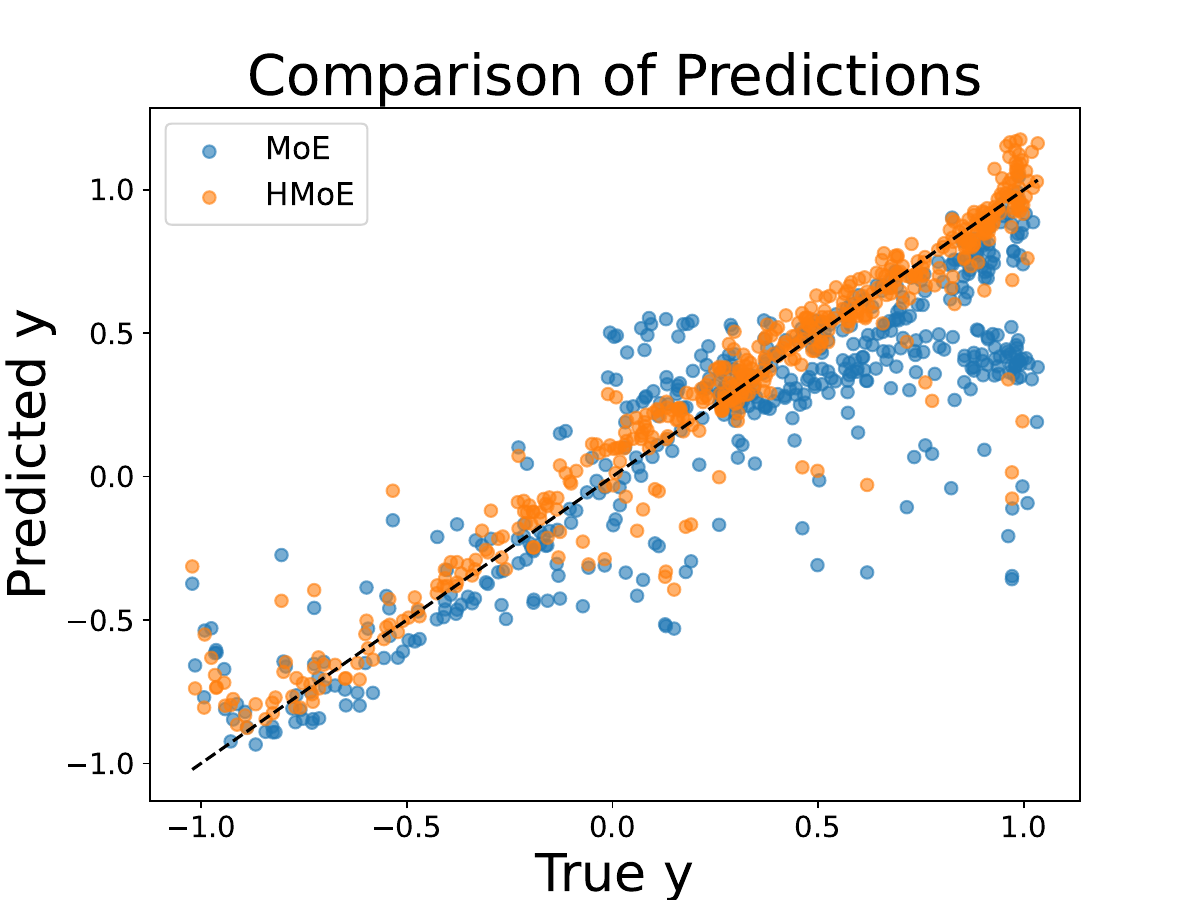} 
        \\
        {\small{(c)}}
        \end{tabular} \\
        \end{tabular}
    \end{minipage}
    \caption{Synthetic experiment illustrating how HMoE more effectively handles data with multi-level structures. Figures (a) and (b) depict the hierarchical target generation process, and (c) shows HMoE’s predictive advantage over MoE.}
    \label{fig:synthetic_result_1}
\end{figure*}
\begin{figure*}[t!]
    \begin{minipage}{\textwidth}
    \centering
    \vspace{-3ex}
    \begin{tabular}{@{\hspace{-3.8ex}} c @{\hspace{-4.9ex}} c @{\hspace{-4.9ex}} c @{\hspace{-1.5ex}}}
        \begin{tabular}{c}
        \includegraphics[width=.37\textwidth]{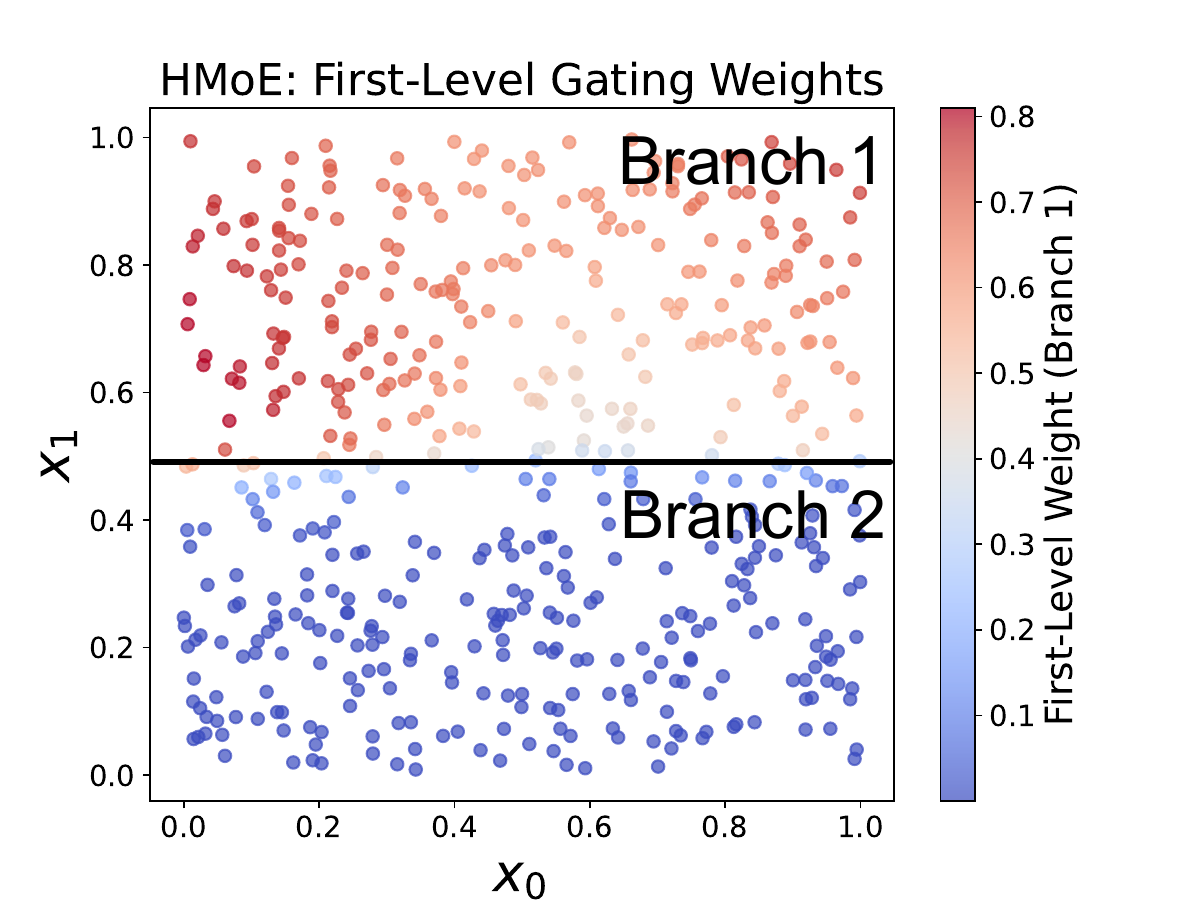}
        \\
        {\small{(d)}}
        \end{tabular} & 
        \begin{tabular}{c}
        \includegraphics[width=.37\textwidth]{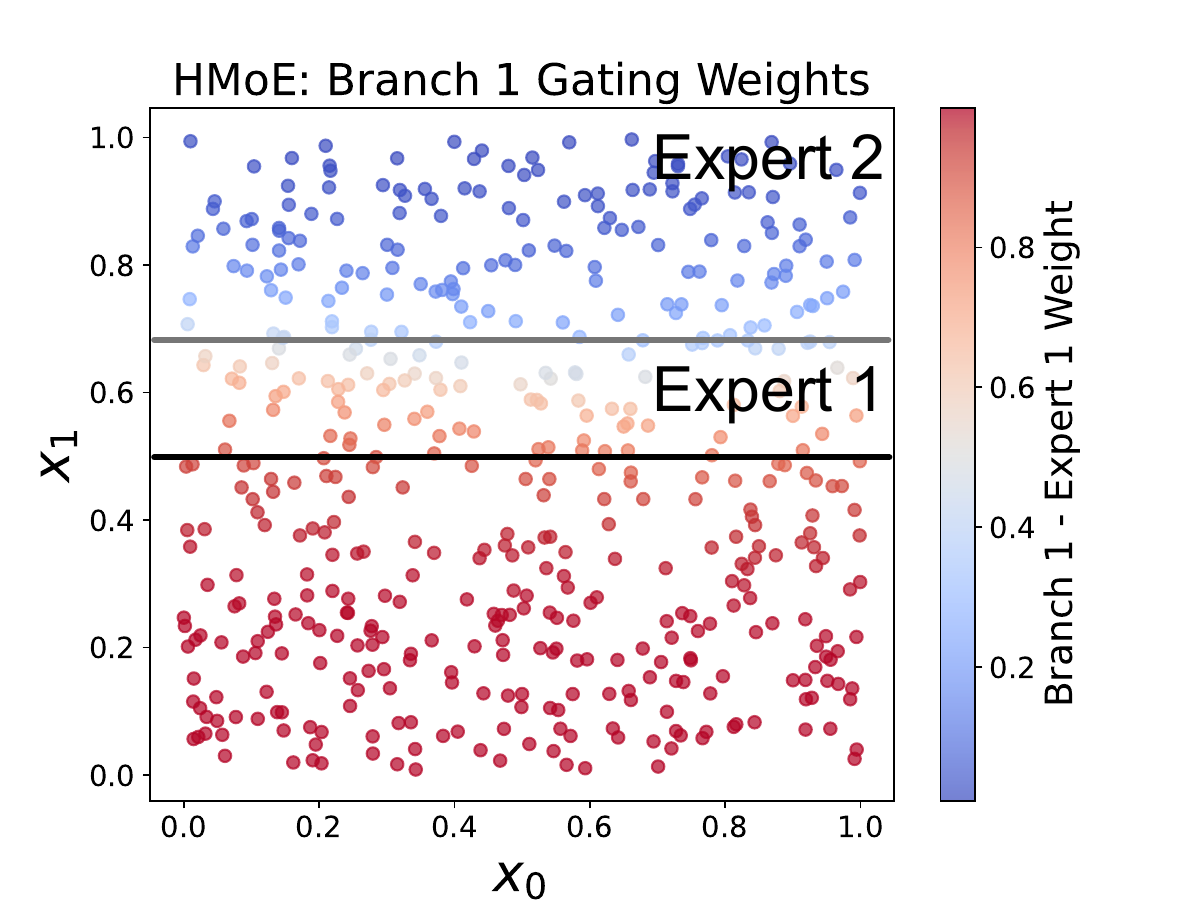} 
        \\
        {\small{(e)}}
        \end{tabular} &
        \begin{tabular}{c}
        \includegraphics[width=.37\textwidth]{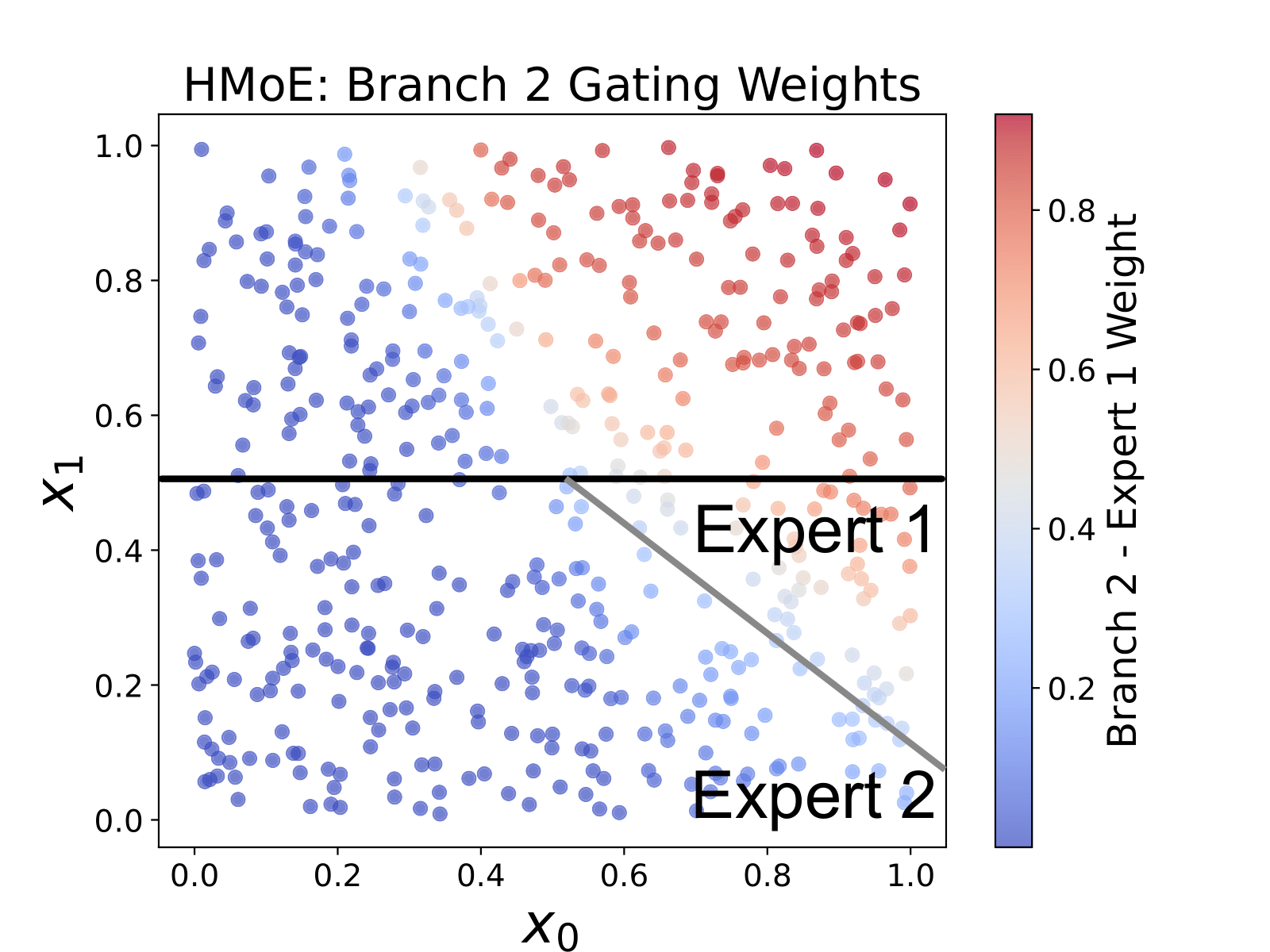} 
        \\
        {\small{(f)}}
        \end{tabular} \\
        \end{tabular}
    \end{minipage}
    \caption{Synthetic experiment illustrating how HMoE more effectively handles data with multi-level structures. Figures (d)–(f) highlight how HMoE’s coarse-to-fine partitioning of the input space results in stronger expert specialization.}
    \label{fig:synthetic_result_2}
\end{figure*}

\paragraph{CMU-MOSI Dataset}
We also tested HMoE as a fusion method on the CMU-MOSI dataset \cite{zadeh2018multi}, which utilizes visual, acoustic, and textual data for a sentiment analysis task. Following the preprocessing steps outlined by \cite{hu2022unimse}, we employed a pre-trained T5 \cite{2020t5} for text encoding, librosa \cite{mcfee2015librosa} for audio feature extraction, and EfficientNet \cite{tan2019efficientnet} for video feature encoding. The baselines include (1) the early fusion method, Tensor Fusion Network (TFN) \cite{zadeh2017tensor}; (2) the Multimodal Transformer (MulT), which fuses modalities by modeling their interactions \cite{tsai2019multimodal}; and (3) the Multimodal Adaptation Gate (MAG), which focuses on the consistency and differences across modalities \cite{rahman2020integrating}. As shown in Table \ref{tab:cmu_mosi}, among all fusion methods, employing Laplace gating at both levels of HMoE yields the best results, while the Softmax-Laplace combination ranks a close second.


\subsection{HMoE Naturally Capture Hierarchical Structures in the Data}
\vspace{0.5 em}
\noindent
\textbf{Synthetic Experiment.} 
We begin by demonstrating HMoE’s advantage in handling data with multi-level structures compared to standard MoE. As illustrated in Figure \ref{fig:synthetic_result_2}(a), we designed a target generation process where two input features, $x_0$ and $x_1$, are each sampled uniformly from the interval $[0, 1]$. The feature $x_0$ provides a coarse partition of the data into two groups, and within each group, $x_1$ further divides the data into distinct regions. Each region is governed by a different target function—specifically, sine, cosine, quadratic, or linear (see Figure \ref{fig:synthetic_result_2}(b)).
In our setup, the standard MoE model utilizes a single Softmax gating mechanism to assign data among four experts, whereas HMoE employs two branches, each containing two experts. Both models were trained on 2,000 samples and evaluated on 500 samples under the same configuration. Figure \ref{fig:synthetic_result_2}(c) presents a comparison of prediction accuracy, showing that HMoE significantly outperforms standard MoE, particularly in the positive $y$ region. We further examine the outputs of the gating networks at both levels: Figure \ref{fig:synthetic_result_2}(d) shows the first-level, coarse partition, while Figures \ref{fig:synthetic_result_2}(e) and \ref{fig:synthetic_result_2}(f) illustrate how experts specialize in each branch’s corresponding region. The resulting specialization boundaries closely align with the target function shapes, demonstrating that HMoE enhances expert specialization and interpretability, and highlighting its advantage in capturing multi-level structures in the data.

\vspace{0.5 em}
\noindent
\textbf{Laplace HMoE Enhances Latent Domain Generalization.}
Many real-world datasets can be grouped into different latent domains. For example, in clinical prediction tasks, patients might be categorized by factors such as age, medical history, treatments, or symptoms. Training a single, generic model on heterogeneous patient data often proves less effective than using a domain-specific model, as suggested by SLDG \cite{wu2024iterative}. However, SLDG assigns a fixed classifier to each domain without accounting for potential interactions among domains. Moreover, it relies heavily on hierarchical clustering, making the approach vulnerable to variations in clustering quality.
We evaluated HMoE on this task by replacing domain-specific classifiers with the HMoE module. Through its hierarchical routing mechanism, HMoE recursively partitions inputs, allowing tokens from each patient to interact with multiple inner and outer experts. For a fair comparison with baselines, we excluded clinical notes from MIMIC-IV and used only lab values to test different methods; we also evaluated HMoE on the eICU dataset \cite{pollard2018eicu}, which includes over $139$k ICU stays from 2014 to 2015. Following \cite{wu2024iterative}, we evaluated HMoE on two predictive tasks—readmission prediction and mortality prediction—and compared against the following baselines: (1) Oracle: Trained directly on the target test data. (2) Base: Trained only on the source training data. (3) DANN \cite{ganin2016domain} and (4) MLDG \cite{li2018learning}, which require domain IDs. (5) IRM \cite{arjovsky2019invariant}, which does not require domain IDs. Tables \ref{tab:MIMICIV_results} and \ref{tab:eICU_results} show the performance on both datasets. By leveraging hierarchical routing mechanisms, HMoE effectively partitions the input and identifies potential latent subgroups, assigning specialized experts to handle them. This leads to better overall generalization. Among the HMoE models, while performance differences are small, the Laplace-Laplace gating variant achieves the strongest results.

\begin{table}[t]
\caption{\small On the eICU dataset, domain generalization results show that HMoE achieves a balance between personalization and interactions across domains, while applying Laplace gating on both levels achieves the best performance. The best outcome is highlighted in \textbf{bold font}, the second-best is \underline{underlined}, and Oracle’s results are in \textit{italics}. Results are averaged across 5 random experiments.}
\centering
\label{tab:eICU_results}
\begin{tabular}{c|cc|cc}
\toprule
\multirow{2}{*}{\textbf{Task}} & \multicolumn{2}{c}{\textbf{Readmission}} & \multicolumn{2}{c}{\textbf{Mortality}} \\
\cmidrule(lr){2-3} \cmidrule(lr){4-5}
& \textbf{AUPRC} & \textbf{AUROC} & \textbf{AUPRC} & \textbf{AUROC} \\
\midrule
Oracle  & $\mathit{21.92 \pm 0.15}$ & $\mathit{67.72 \pm 0.42}$ & $\mathit{27.14 \pm 0.06}$ & $\mathit{83.87 \pm 0.57}$ \\
Base    & $10.41 \pm 0.12$ & $51.01 \pm 0.31$ & $23.02 \pm 0.24$ & $80.31 \pm 0.43$ \\
DANN    & $13.50 \pm 0.09$ & $53.79 \pm 0.19$ & $24.47 \pm 0.08$ & $80.82 \pm 0.27$ \\
MLDG    & $10.41 \pm 0.07$ & $52.54 \pm 0.43$ & $22.41 \pm 0.12$ & $79.73 \pm 0.39$ \\
IRM     & $13.62 \pm 0.13$ & $53.78 \pm 0.22$ & $25.18 \pm 0.09$ & $80.09 \pm 0.47$ \\
SLDG    & $18.57 \pm 0.10$ & $62.30 \pm 0.46$ & $\mathbf{26.79 \pm 0.16}$ & $\mathbf{82.44 \pm 0.19}$ \\
\cellcolor{ashgrey}HMoE-SS & $19.39 \pm 0.05$ & $63.61 \pm 0.23$ & $26.60 \pm 0.08$ & $81.92 \pm 0.28$ \\
\cellcolor{ashgrey}HMoE-SL & $19.35 \pm 0.09$ & $65.33 \pm 0.15$ & $26.57 \pm 0.04$ & $81.97 \pm 0.33$ \\
\cellcolor{ashgrey}HMoE-LS & \underline{$19.46 \pm 0.06$} & \underline{$65.54 \pm 0.21$} & $26.63 \pm 0.13$ & $81.93 \pm 0.41$ \\
\cellcolor{ashgrey}HMoE-LL & $\mathbf{19.74 \pm 0.11}$ & $\mathbf{65.67 \pm 0.17}$ & \underline{$26.71 \pm 0.11$} & \underline{$82.06 \pm 0.29$} \\
\bottomrule
\end{tabular}
\end{table}


\begin{table}[ht]
\centering
\caption{\small For domain generalization on the MIMIC-IV dataset (excluding clinical notes), HMoE with Laplace gating outperforms most baselines. The results are averaged over 5 random experiments.}
\label{tab:MIMICIV_results}
\begin{tabular}{c|cc|cc}
\toprule
\multirow{2}{*}{\textbf{Task}} & \multicolumn{2}{c}{\textbf{Readmission}} & \multicolumn{2}{c}{\textbf{Mortality}} \\
\cmidrule(lr){2-3} \cmidrule(lr){4-5}
 & \textbf{AUPRC} & \textbf{AUROC} & \textbf{AUPRC} & \textbf{AUROC} \\
\midrule
Oracle   & $\mathit{28.21 \pm 0.34}$ & $\mathit{69.31 \pm 0.53}$ & $\mathit{42.83 \pm 0.48}$ & $\mathit{89.82 \pm 0.75}$ \\
Base     & $23.70 \pm 0.23$ & $66.54 \pm 0.41$ & $37.40 \pm 0.20$ & $86.10 \pm 0.64$ \\
DANN     & $24.68 \pm 0.09$ & $67.31 \pm 0.33$ & $38.01 \pm 0.17$ & $87.34 \pm 0.39$ \\
MLDG     & $20.50 \pm 0.14$ & $63.72 \pm 0.29$ & $35.98 \pm 0.31$ & $85.72 \pm 0.68$ \\
IRM      & $24.23 \pm 0.21$ & $66.80 \pm 0.22$ & $38.72 \pm 0.19$ & $87.59 \pm 0.43$ \\
SLDG     & $27.41 \pm 0.10$ & $69.02 \pm 0.40$ & $41.56 \pm 0.12$ & $\mathbf{89.85 \pm 0.59}$ \\
\cellcolor{ashgrey}HMoE-SS  & \underline{$27.82 \pm 0.24$} & $69.13 \pm 0.21$ & $42.23 \pm 0.32$ & $89.47 \pm 0.18$ \\
\cellcolor{ashgrey}HMoE-SL  & $\mathbf{27.96 \pm 0.18}$ & \underline{$69.17 \pm 0.25$} & \underline{$42.44 \pm 0.35$} & $89.62 \pm 0.13$ \\
\cellcolor{ashgrey}HMoE-LS  & $27.63 \pm 0.13$ & $69.08 \pm 0.36$ & $42.41 \pm 0.19$ & \underline{$89.69 \pm 0.25$} \\
\cellcolor{ashgrey}HMoE-LL  & $\mathbf{27.96 \pm 0.22}$ & $\mathbf{69.19 \pm 0.31}$ & $\mathbf{42.46 \pm 0.27}$ & $89.67 \pm 0.23$ \\
\bottomrule
\end{tabular}
\end{table}

\subsection{Quantatitive Analysis}



\paragraph{Multimodal Routing Distributions.}
We then analyze how modality tokens are distributed across different experts and groups. Figure \ref{fig:mod_dist} displays the distribution of three modality tokens in the best-performing HMoE block for corresponding tasks from MIMIC-IV. The HMoE module consists of two expert groups, each containing four experts. The results are taken from the final HMoE block of the trained model, using the first batch of data. Most vital signs and clinical notes tokens are routed to expert group 1, while CXR tokens are predominantly routed to expert group 2. For tasks (a) and (b), vital signs and clinical notes contribute more heavily to the overall HMoE prediction, particularly in task (b). However, for task (c), CXR tokens play a more significant role, contributing almost as much as vital signs, despite being present in smaller quantities. Additionally, due to the load-balancing loss applied during training, the total token count is nearly uniformly distributed among experts, with minimal token dropping because of exceeding capacity limits.

\paragraph{Distribution of Clinical Events.}
Given that the number of clinical event categories is much larger than the number of modalities, it is more intuitive to visualize the impact of different gating function combinations on the distribution of clinical events. Figure \ref{fig:event_dist} (a) illustrates the routing distribution for the most commonly observed clinical events using the best-performing Laplace-Laplace gating function combination of HMoE in latent domain discovery, compared to the Softmax gating function. The results indicate that the Laplace-Laplace combination promotes greater diversification in routing clinical event samples to experts while encouraging expert sharing across different categories. We further conduct ablation studies by varying the number of inner and outer experts in the best-performing HMoE across four tasks, as shown in Figure \ref{fig:event_dist} (b) and (c), where their number of outer and inner experts is fixed at 2 and 4, respectively. The results demonstrate that increasing the number of experts has a positive impact on performance, particularly for inner experts, though this improvement comes with an increase in computational demands.

\vspace{0.5 em}
\noindent
\textbf{Why Laplace Gating Performs Better.}
In the standard Softmax gating \cite{nguyen2023demystifying}, the similarity score is computed as the inner product of a token’s hidden representation and an expert embedding. However, this approach can lead to representation collapse \cite{chi2022on, pham2024competesmoe}, where a small number of experts dominate the decision-making process, rendering other experts redundant and slowing parameter estimation. By contrast, Laplace gating partially addresses this issue by computing similarity as the $L_2$-distance between token representations and expert embeddings. This approach is less biased towards experts with large norms, giving all experts a more balanced chance of selection based on proximity to the token representation. Consequently, Laplace gating is especially effective for heterogeneous or multimodal/multi-domain inputs, since it is less sensitive to the scale and variance of feature distributions. Empirically, using Laplace gating at both gating layers further enhances these benefits: it often yields lower validation errors across tasks, indicating that each gating layer more effectively supports expert specialization. 

\vspace{0.5 em}
\noindent
\textbf{Limitations.}
The enhanced ability to process complex, multi-domain inputs comes with an increased computational cost, which is a key limitation of HMoE. From our large-scale experiments, we observed that standard MoE requires approximately $80\%$ of the computation time for ImageNet and $76\%$ for MIMIC-IV multimodal tasks compared to HMoE, assuming the same total number of experts. While the gating function itself does not introduce additional parameters, the increase in computation primarily arises from extra dispatch and combination steps (e.g., steps 2 and 8 in Algorithm \ref{alg:hmoe}).

\begin{figure*}[t!]
    \begin{minipage}{\textwidth}
    \centering
    \begin{tabular}{@{\hspace{-3.8ex}} c @{\hspace{-2.4ex}} c @{\hspace{-1.5ex}} c @{\hspace{-1.5ex}}}
        \begin{tabular}{c}
        \includegraphics[width=.34\textwidth]{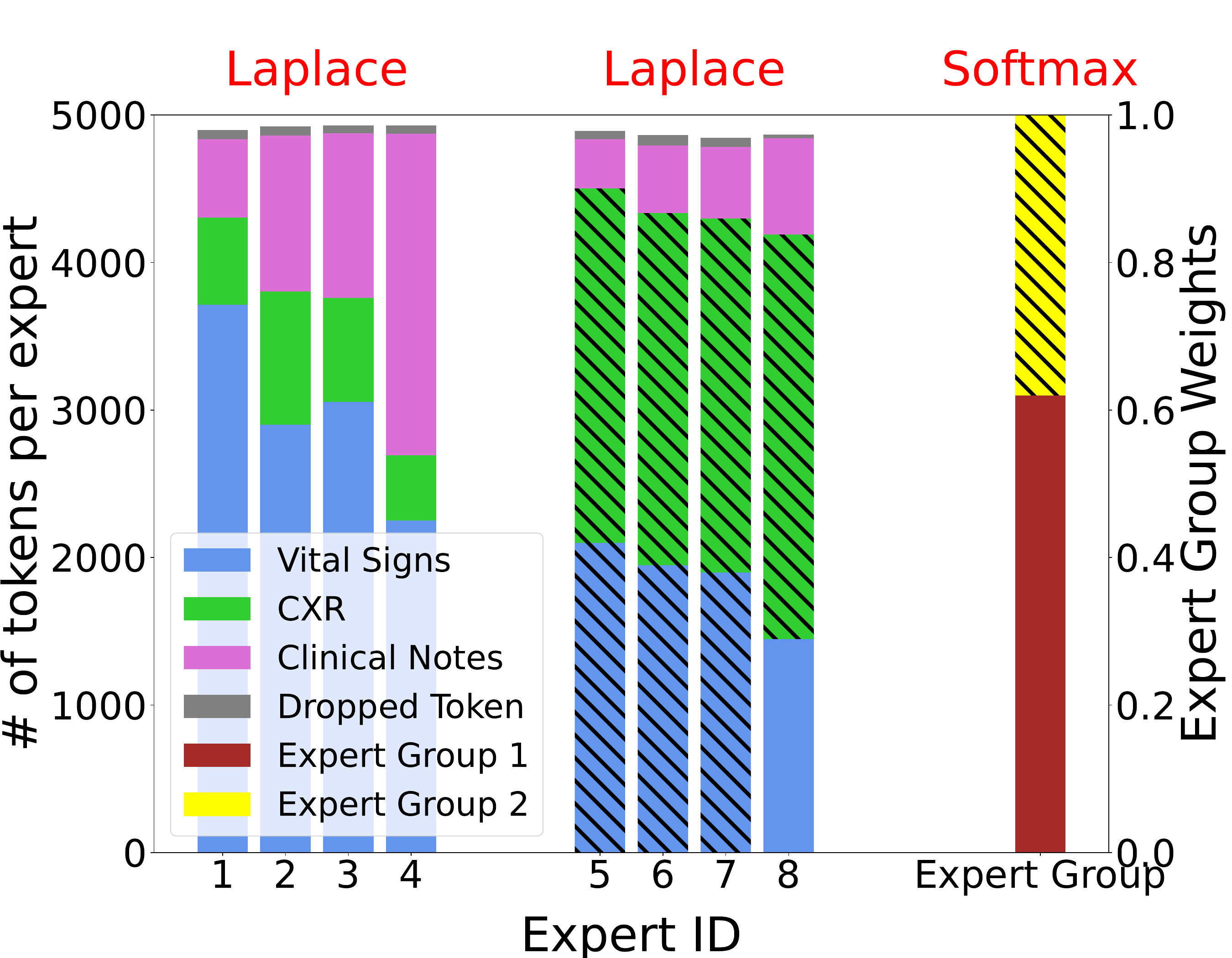}
        \\
        {\small{(a) In-Hospital Mortality}}
        \end{tabular} & 
        \begin{tabular}{c}
        \includegraphics[width=.34\textwidth]{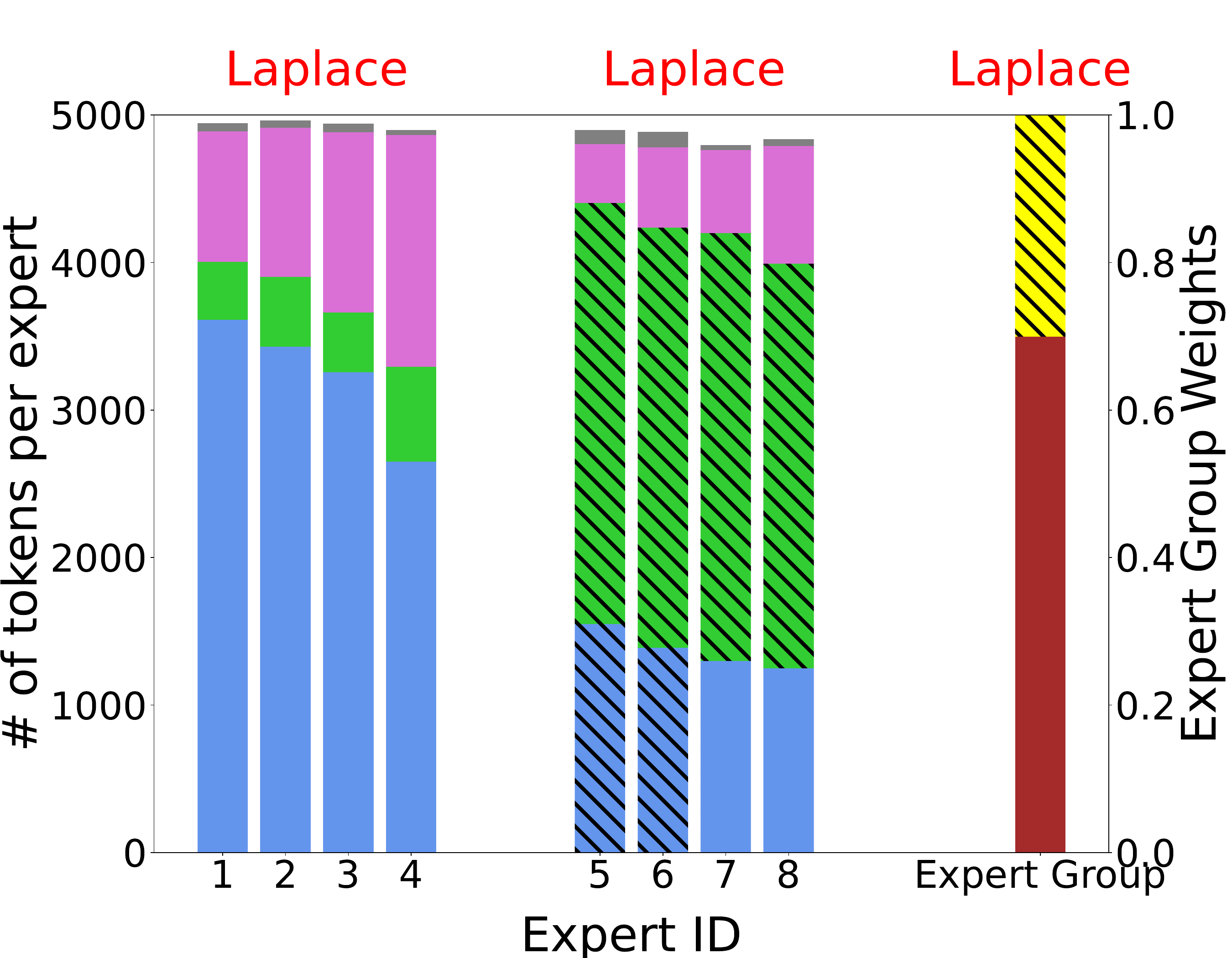} 
        \\
        {\small{(b) Length-of-Stay}}
        \end{tabular} &
        \begin{tabular}{c}
        \includegraphics[width=.34\textwidth]{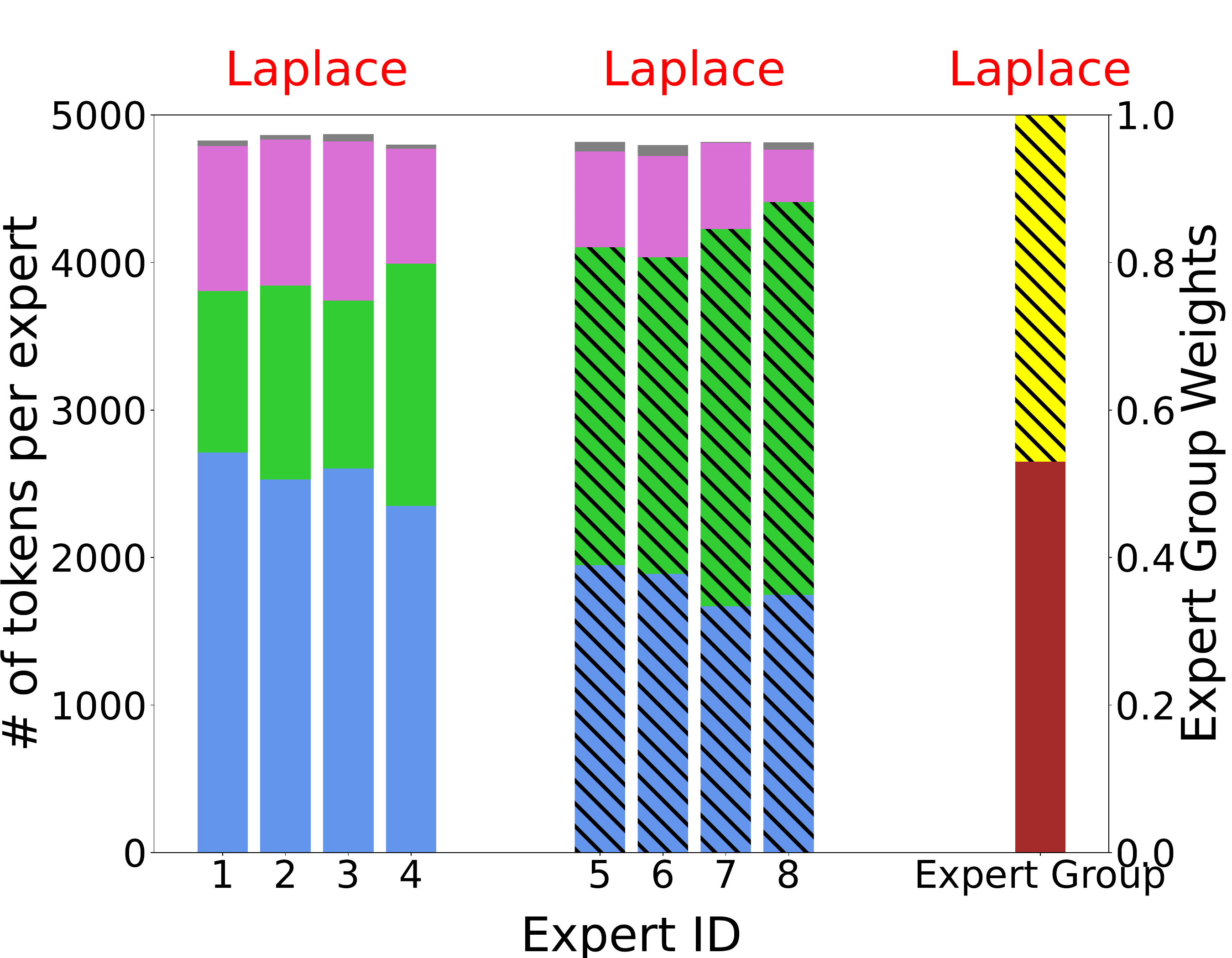} 
        \\
        {\small{(c) Phenotyping}}
        \end{tabular} \\
        \end{tabular}
    \end{minipage}
    \caption{\small Token distribution (time series, CXR, clinical notes) of HMoE blocks of a multimodal transformer. We present the best-performing gating combinations for three tasks evaluated on MIMIC-IV, where the HMoE block comprises 2 outer expert groups, each containing 4 inner experts. Expert IDs 1 to 4 (left section of each figure) represent token distributions from expert group 1, and expert IDs 5 to 8 (middle section) represent token distributions from expert group 2. The right section shows the relative weights assigned to each expert group.}
    \label{fig:mod_dist}
\end{figure*}

\begin{figure*}[t!]
    \begin{minipage}{\textwidth}
    \centering
    \begin{tabular}{@{\hspace{-3.8ex}} c @{\hspace{-2.4ex}} c @{\hspace{-1.5ex}} c @{\hspace{-1.5ex}}}
        \begin{tabular}{c}
        \includegraphics[width=.34\textwidth]{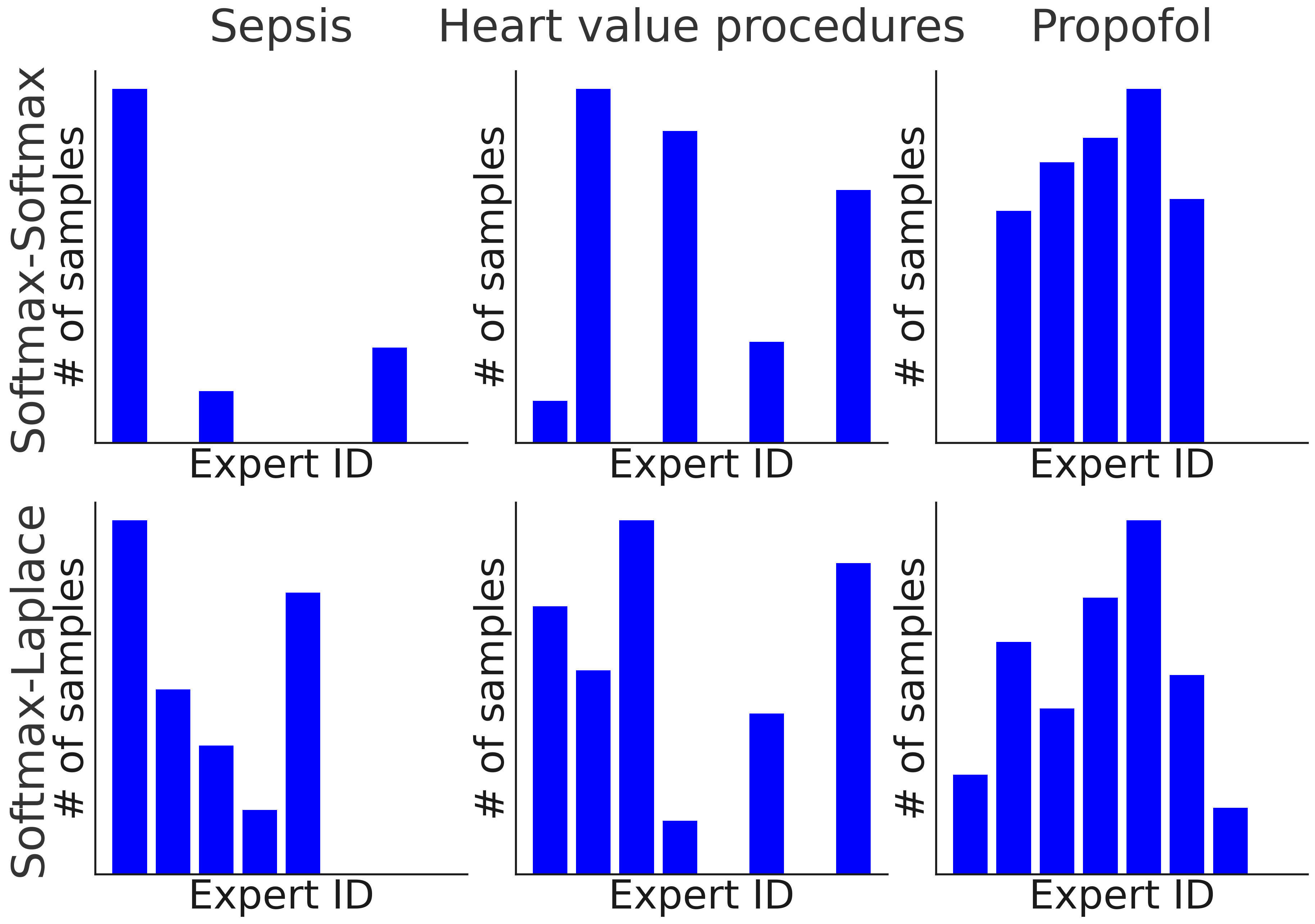}
        \\
        {\small{(a)}}
        \end{tabular} & 
        \begin{tabular}{c}
        \includegraphics[width=.34\textwidth]{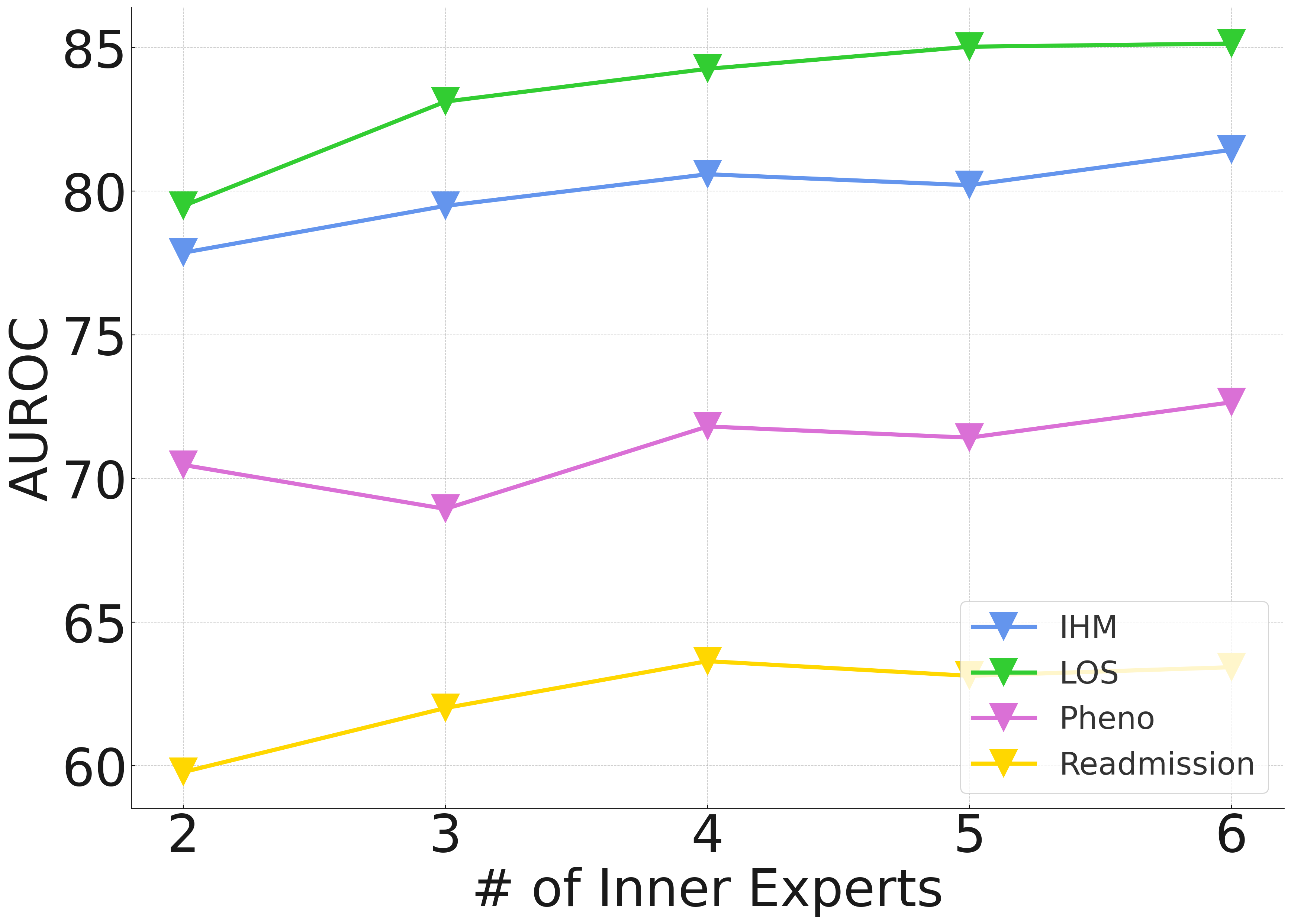} 
        \\
        {\small{(b)}}
        \end{tabular} &
        \begin{tabular}{c}
        \includegraphics[width=.34\textwidth]{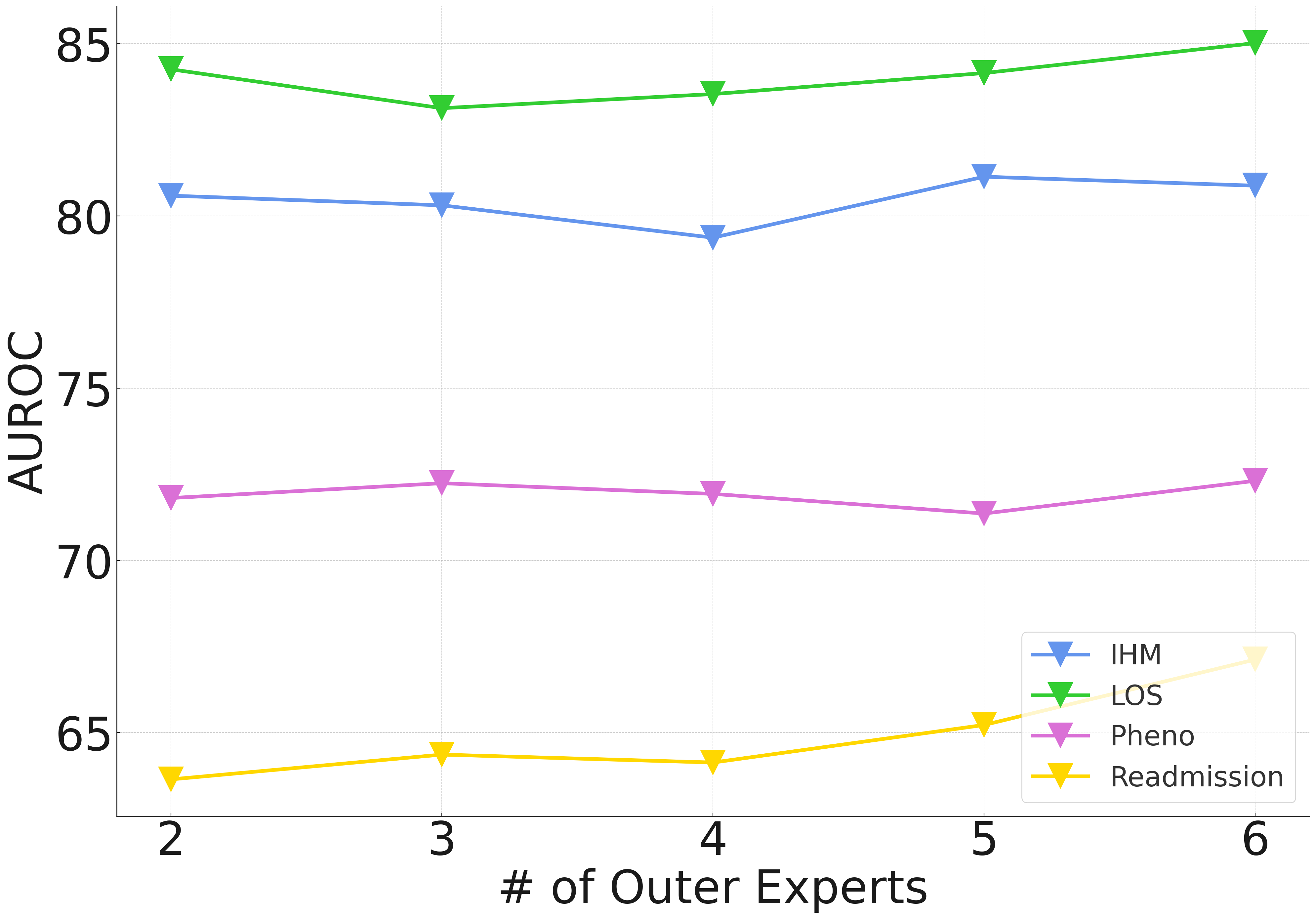} 
        \\
        {\small{(c)}}
        \end{tabular} \\
        \end{tabular}
    \end{minipage}
    \caption{\small (a) Distribution of top clinical events across HMoE expert IDs under the Laplace-Laplace gating combination (top row) compared to the Softmax-Softmax gating combination (bottom row).
(b)/(c) Performance variations as the number of inner/outer experts increases.}
    \label{fig:event_dist}
\end{figure*}

\section{Discussion}
\label{sec:discussion}
In this paper, we explore three different types of two-level hierarchical mixture of experts (HMoE) equipped with three combinations of the vanilla Softmax gating and the Laplace gating. Our theoretical analysis illustrates that using the Softmax gating at either level of the HMoE model would induce some intrinsic parameter interactions expressed in the language of partial differential equations, which decelerates the convergence rates of parameter estimation and expert estimation. Meanwhile, we demonstrate that employing the Laplace gating at both levels allows the model parameters to avoid the interactions caused by the Softmax gating. Therefore, the parameter and expert convergence is substantially accelerated, thereby leading to the improvement of the expert specialization.

\vspace{0.5 em}
\noindent
We conducted a series of experiments to compare different gating combinations across multiple tasks and datasets. The results consistently showed that replacing one or both Softmax gating layers with Laplace gating improved model performance. We also found that Laplace gating provides more robust expert assignments under multi-domain or multimodal inputs, which supports the theoretical premise. Therefore, we conclude that Laplace-based gating strategies, and in particular Laplace-Laplace gating, are highly effective for hierarchical mixture-of-experts models, reinforcing the broader argument for exploring alternative gating functions beyond the standard Softmax.

\vspace{0.5 em}
\noindent
\textbf{Future directions.} There are a few potential research directions based on our paper: 

\vspace{0.5 em}
\noindent
Firstly, the problem of estimating the true number of experts $k^*_2$ has remained open in the literature. It is worth noting from Table~\ref{table:parameter_rates} that the convergence rates of parameter estimation fall proportionately to the cardinality of the Voronoi cells, that is, the corresponding number of fitted experts. Thus, a solution to estimate $k^*_2$ is to reduce the number of fitted experts $k_2$, which leads to the decrease of the Voronoi cell cardinality, until the convergence of all the parameter estimations reach the optimal rate of order $\widetilde{\mathcal{O}}_P(n^{-1/2})$. This can be done by regularizing the log-likelihood function of the Gaussian HMoE model using the parameter discrepancies as suggested by \cite{manole_selection_2021}.  

\vspace{0.5 em}
\noindent
Secondly, we can conduct the convergence analysis of parameter and expert estimation under a more practical scenario called a misspecified setting where the data are generated from an arbitrary distribution $Q(Y|X)$ rather than the Gaussian HMoE model. The MLE then converges to a mixing measure $\overline{G} \in \argmin_{G \in \mathcal{G}_{k^*_1k_2}(\Theta)} \text{KL}(Q(Y|X) || p_{G}(Y|X))$ where $\text{KL}$ denotes the Kullback-Leibler divergence. However, since the current MLE convergence analysis under the misspecified setting has only been conducted when the function space is convex~\cite{vandeGeer-00} while the space $\mathcal{G}_{k^*_1k_2}(\Theta)$ is non-convex, we believe that further technical tools need to be developed to tackle that issue.

\vspace{0.5 em}
\noindent
On the practical side, we plan to explore techniques like pruning or expert-sharing to reduce computational costs in large-scale or multimodal tasks. We also intend to investigate more diverse hybrid gating mechanisms, by introducing additional gatings such as Cosine gating \cite{li2023sparse,nguyen2025cosine} and Sigmoid gating \cite{csordas2023approximating,nguyen2024sigmoid}, to identify the best configurations for specific tasks. Finally, we aim to discover novel applications where HMoE’s hierarchical structure and robust gating functions can provide significant improvements.

\section{Proofs for Convergence of Expert Estimation}
\label{appendix:param_rates}
In this section, we provide proofs for Theorems~\ref{theorem:param_rates_SS}-~\ref{theorem:param_rates_LL}. We first proceed with an overall of the proof strategy.

\vspace{0.5 em}
\noindent
\textbf{Overview.} We will focus on establishing the following inequality:
    \begin{align*}
    \inf_{G\in\mathcal{G}_{k^*_1,k_2}(\Theta)}\bbE_{\bbX}[h(p^{type}_{G}(\cdot|\bbX),p^{type}_{G_*}(\cdot|\bbX))]/\mathcal{L}_{(r_1,r_2,r_3)}(G,G_*)>0,
\end{align*}
where the value of $(r_1,r_2,r_3)$ varies with the variable $type\in\{SS,SL,LL\}$. Note that the Hellinger distance $h$ is lower bounded by the Total Variation distance $V$, that is, $h\geq V$, it suffices to demonstrate that
\begin{align}
    \label{eq:general_universal_inequality_over_ss}
    \inf_{G\in\mathcal{G}_{k^*_1,k_2}(\Theta)}\bbE_{\bbX}[V(p^{type}_{G}(\cdot|\bbX),p^{type}_{G_*}(\cdot|\bbX))]/\mathcal{L}_{(r_1,r_2,r_3)}(G,G_*)>0.
\end{align}
To this end, we first show that
\begin{align}
    \label{eq:general_local_inequality_over_ss}
    \lim_{\varepsilon\to0}\inf_{G\in\mathcal{G}_{k^*_1,k_2}(\Theta):\mathcal{L}_{(r_1,r_2,r_3)}(G,G_*)\leq\varepsilon}\bbE_{\bbX}[V(p^{type}_{G}(\cdot|\bbX),p^{type}_{G_*}(\cdot|\bbX))]/\mathcal{L}_{(r_1,r_2,r_3)}(G,G_*)>0.
\end{align}
The proof of this result will be presented later. Now, suppose that it holds true, then there exists a positive constant $\varepsilon'$ that satisfies
\begin{align*}
    \inf_{G\in\mathcal{G}_{k^*_1,k_2}(\Theta):\mathcal{L}_1(G,G_*)\leq\varepsilon'}\bbE_{\bbX}[V(p^{type}_{G}(\cdot|\bbX),p^{type}_{G_*}(\cdot|\bbX))]/\mathcal{L}_{(r_1,r_2,r_3)}(G,G_*)>0.
\end{align*}
Thus, it suffices to establish the following inequality:
\begin{align}
    \label{eq:general_global_inequality_ss}
    \inf_{G\in\mathcal{G}_{k^*_1,k_2}(\Theta):\mathcal{L}_1(G,G_*)>\varepsilon'}\bbE_{\bbX}[V(p^{type}_{G}(\cdot|\bbX),p^{type}_{G_*}(\cdot|\bbX))]/\mathcal{L}_{(r_1,r_2,r_3)}(G,G_*)>0.
\end{align}
Assume by contrary that the inequality~\eqref{eq:general_global_inequality_ss} does not hold true, then we can seek a sequence of mixing measures $G'_n\in\mathcal{G}_{k^*_1,k_2}(\Theta)$ that satisfy $\mathcal{L}_1(G'_n,G_*)>\varepsilon'$ and
\begin{align*}
    \lim_{n\to\infty}\bbE_{\bbX}[V(p^{type}_{G'_n}(\cdot|\bbX),p^{type}_{G_*}(\cdot|\bbX))]/\mathcal{L}_{(r_1,r_2,r_3)}(G'_n,G_*)=0.
\end{align*}
Thus, we deduce that $\bbE_{\bbX}[V(p^{type}_{G'_n}(\cdot|\bbX),p^{type}_{G_*}(\cdot|\bbX))]\to0$ as $n\to\infty$. Since $\Theta$ is a compact set, we can substitute the sequence $(G'_n)$ by one of its subsequences that converges to a mixing measure $G'\in\mathcal{G}_{k^*_1,k_2}(\Theta)$. Recall that $\mathcal{L}_{(r_1,r_2,r_3)}(G'_n,G_*)>\varepsilon'$, then we deduce that $\mathcal{L}_{(r_1,r_2,r_3)}(G',G_*)>\varepsilon'$.
By employing the Fatou's lemma, it follows that
\begin{align*}
    0&=\lim_{n\to\infty}\bbE_{\bbX}[V(p^{type}_{G'_n}(\cdot|\bbX),p^{type}_{G_*}(\cdot|\bbX))]/\mathcal{L}_{(r_1,r_2,r_3)}(G'_n,G_*)\\
    &\geq\frac{1}{2} \int\liminf_{n\to\infty}\Big|p^{type}_{G'_n}(y|\bx)-p^{type}_{G_*}(y|\bx)\Big|^2~\dint(\bx,y).
\end{align*}
Thus, we obtain that $p^{type}_{G'}(y|\bx)=p^{type}_{G_*}(y|\bx)$ for almost surely $(\bx,y)$. According to Proposition~\ref{prop:identifiability}, we get that $G'\equiv G_*$, which yields that $\mathcal{L}_{(r_1,r_2,r_3)}(G',G_*)=0$. This result contradicts the fact that $\mathcal{L}_{(r_1,r_2,r_3)}(G',G_*)>\varepsilon'>0$. Hence, we obtain the result in equation~\eqref{eq:general_global_inequality_ss}, which together with the inequality~\eqref{eq:general_local_inequality_over_ss} leads to the conclusion in equation~\eqref{eq:general_universal_inequality_over_ss}.

\vspace{0.5 em}
\noindent
Now, we are going back to the proof of the inequality~\eqref{eq:general_local_inequality_over_ss}.

\vspace{0.5 em}
\noindent
\textbf{Proof of the inequality~\eqref{eq:general_local_inequality_over_ss}:} Suppose that the inequality~\eqref{eq:general_local_inequality_over_ss} does not hold, then we can find a sequence of mixing measures $(G_n)$ in $\mathcal{G}_{k^*_1,k_2}(\Theta)$ that satisfies $\mathcal{L}_{(r_1,r_2,r_3)}(G_n,G_*)\to0$ and
\begin{align}
    \label{eq:general_ratio_limit_over_ss}
    \bbE_{\bbX}[V(p^{type}_{G_n}(\cdot|\bbX),p^{type}_{G_*}(\cdot|\bbX))]/\mathcal{L}_{(r_1,r_2,r_3)}(G_n,G_*)\to0,
\end{align}
as $n\to\infty$. For each $j_1\in[k^*_1]$, let $\mathcal{V}^n_{j_1}:=\mathcal{V}_{j_1}(G_n)$ be a Voronoi cell of $G_n$ generated by the $j_1$-th components of $G_*$. As the Voronoi loss $\mathcal{V}^n_{j_1}$ has only one element and our arguments are asymptotic, we may assume WLOG that $\mathcal{V}^n_{j_1}=\mathcal{V}_{j_1}=\{j_1\}$ for any $j_1\in[k^*_1]$. Then, the Voronoi loss becomes
\begin{align}
    &\mathcal{L}_{(r_1,r_2,r_3)}(G_n,G_*)=\sum_{j_1=1}^{k^*_1}\Big|\exp(b^n_{j_1})-\exp(b^*_{j_1})\Big|+\sum_{j_1=1}^{k^*_1}\exp(b^n_{j_1})\|\dajn\|+\sum_{j_1=1}^{k^*_1}\exp(b^n_{j_1})\nonumber\\
    &\times\Bigg[\sum_{j_2:|\mathcal{V}_{j_2|j_1}|=1}\sum_{i_2\in\mathcal{V}_{j_2|j_1}}\exp(\beta^n_{i_2|j_1})\Big(\|\doijn\|+\|\deijn\|+|\dtijn|+|\dvijn|\Big)\nonumber\\
    &+\sum_{j_2:|\mathcal{V}_{j_2|j_1}|>1}\sum_{i_2\in\mathcal{V}_{j_2|j_1}}\exp(\beta^n_{i_2|j_1})\Big(\|\doijn\|^{2}+\|\deijn\|^{r_1}+|\dtijn|^{r_2}\nonumber\\
     \label{eq:loss}
    &+|\dvijn|^{r_3}\Big)\Bigg]+\sum_{j_1=1}^{k^*_1}\exp(b^n_{j_1})\sum_{j_2=1}^{k^*_2}\Big|\sum_{i_2\in\mathcal{V}_{j_2|j_1}}\exp(\beta^n_{i_2|j_1})-\exp(\beta^*_{j_2|j_1})\Big|.
\end{align}
Since $\mathcal{L}_{(r_1,r_2,r_3)}(G_n,G_*)\to0$ as $n\to\infty$, it follows that $\exp(\bjn)\to\exp(\bj)$, $\ajn\to\aj$, $\exp(\bein)\to\exp(\bej)$, $\oin\to\oj$, $\ein\to\ej$, $\tin\to\tj$ and $\vin\to\vj$ for all $j_1\in[k^*_1]$, $j_2\in[k^*_2]$ and $i_2\in\mathcal{V}_{j_2|j_1}$.

\noindent
Subsequently, we consider three different settings where the variable $type$ takes the value in the set $\{SS,SL,LL\}$ in Appendices~\ref{appendix:softmax_softmax}, \ref{appendix:softmax_laplace} and \ref{appendix:laplace_laplace}, respectively. In each appendix, the proof will be divided into three main stages.

\subsection{Proof of Theorem~\ref{theorem:param_rates_SS}: When $type=SS$}
\label{appendix:softmax_softmax}
When $type=SS$, the corresponding Voronoi loss function is $\mathcal{L}_{(\frac{1}{2}r^{SS},r^{SS},\frac{1}{2}r^{SS})}(G_n,G_*)=\mathcal{L}_{1n}$ where we define
\begin{align}
    &\mathcal{L}_{1n}:=\sum_{j_1=1}^{k^*_1}\Big|\exp(b^n_{j_1})-\exp(b^*_{j_1})\Big|+\sum_{j_1=1}^{k^*_1}\exp(b^n_{j_1})\|\dajn\|+\sum_{j_1=1}^{k^*_1}\exp(b^n_{j_1})\nonumber\\
    &\times\Bigg[\sum_{j_2:|\mathcal{V}_{j_2|j_1}|=1}\sum_{i_2\in\mathcal{V}_{j_2|j_1}}\exp(\beta^n_{i_2|j_1})\Big(\|\doijn\|+\|\deijn\|+|\dtijn|+|\dvijn|\Big)\nonumber\\
    &+\sum_{j_2:|\mathcal{V}_{j_2|j_1}|>1}\sum_{i_2\in\mathcal{V}_{j_2|j_1}}\exp(\beta^n_{i_2|j_1})\Big(\|\doijn\|^{2}+\|\deijn\|^{\frac{\hrj}{2}}+|\dtijn|^{\hrj}\nonumber\\
     \label{eq:loss_l1_ss}
    &+|\dvijn|^{\frac{\hrj}{2}}\Big)\Bigg]+\sum_{j_1=1}^{k^*_1}\exp(b^n_{j_1})\sum_{j_2=1}^{k^*_2}\Big|\sum_{i_2\in\mathcal{V}_{j_2|j_1}}\exp(\beta^n_{i_2|j_1})-\exp(\beta^*_{j_2|j_1})\Big|.
\end{align}
\textbf{Step 1 - Taylor expansion:} In this stage, we aim to decompose the term
\begin{align*}
    Q_n:=\left[\sum_{j_1=1}^{k^*_1}\exp((\aj)^{\top}\bx+\bj)\right][p^{SS}_{G_n}(y|\bx)-p^{SS}_{G_*}(y|\bx)]
\end{align*}
into a combination of linearly independent terms using the Taylor expansion. For that purpose, let us denote
\begin{align*}
    p^{SS,n}_{j_1}(y|\bx)&:=\sum_{j_2=1}^{k^*_2}\sum_{i_2\in\mathcal{V}_{j_2|j_1}}\softmax((\oin)^{\top}\bx+\beta^n_{i_2|j_1}) \pi(y|(\ein)^{\top}\bx+\tin,\vin),\\
    p^{SS,*}_{j_1}(y|\bx)&:=\sum_{j_2=1}^{k^*_2}\softmax((\oj)^{\top}\bx+\bej) \pi(y|(\ej)^{\top}\bx+\tj,\vj).
\end{align*}
Then, it can be checked that the quantity $Q_n$ is divided as
\begin{align}
    Q_n&=\sum_{j_1=1}^{k^*_1}\exp(\bjn)\left[\exp((\ajn)^{\top}\bx)p^{SS,n}_{j_1}(y|\bx)-\exp((\aj)^{\top}\bx)p^{SS,*}_{j_1}(y|\bx)\right]\nonumber\\
    &-\sum_{j_1=1}^{k^*_1}\exp(\bjn)\left[\exp((\ajn)^{\top}\bx)-\exp((\aj)^{\top}\bx)\right]p^{SS}_{G_n}(y|\bx)\nonumber\\
    &+\sum_{j_1=1}^{k^*_1}\left(\exp(\bjn)-\exp(\bj)\right)\exp((\aj)^{\top}\bx)\left[p^{SS,n}_{j_1}(y|\bx)-p^{SS}_{G_n}(y|\bx)\right]\nonumber\\
     \label{eq:decompose_Qn_ss}
    :&=A_n-B_n+C_n.
\end{align}
\textbf{Step 1A - Decompose $A_n$:} Using the same techniques for decomposing $Q_n$, we can decompose $A_n$ as follows:
\begin{align*}
    A_n&:=\sum_{j_1=1}^{k^*_1}\frac{\exp(\bjn)}{\sum_{j'_2=1}^{k^*_2}\exp((\boldsymbol{\omega}^*_{j'_2|j_1})^{\top}\bx+\beta^*_{j'_2|j_1})}[A_{n,j_1,1}-A_{n,j_1,2}+A_{n,j_1,3}],
\end{align*}
where
\begin{align*}
    A_{n,j_1,1}&:=\sum_{j_2=1}^{k^*_2}\sum_{i_2\in\mathcal{V}_{j_2|j_1}}\exp(\beta^n_{i_2|j_1})\Big[\exp((\oin)^{\top}\bx)\exp((\ajn)^{\top}\bx)\pi(y|(\ein){\top}\bx+\tin,\vin)\nonumber\\
    &\hspace{3cm}-\exp((\oj)^{\top}\bx)\exp((\aj)^{\top}\bx)\pi(y|(\ej)^{\top}\bx+\tj,\vj)\Big],\\
    A_{n,j_1,2}&:=\sum_{j_2=1}^{k^*_2}\sum_{i_2\in\mathcal{V}_{j_2|j_1}}\exp(\beta^n_{i_2|j_1})\Big[\exp((\oin)^{\top}\bx)-\exp((\oj)^{\top}\bx)\Big]\nonumber\\
    &\hspace{7cm}\times\exp((\ajn)^{\top}\bx)p^{SS,n}_{j_1}(y|\bx),\nonumber\\
    A_{n,j_1,3}&:=\sum_{j_2=1}^{k^*_2}\Big(\sum_{i_2\in\mathcal{V}_{j_2|j_1}}\exp(\bein)-\exp(\bej)\Big)\exp((\oj)^{\top}\bx)\\
    &\hspace{1cm}\times[\exp((\aj)^{\top}\bx)\pi(y|(\ej)^{\top}\bx+\tj,\vj)-\exp((\ajn)^{\top}\bx)p^{SS,n}_{j_1}(y|\bx)].
\end{align*}
Based on the cardinality of the Voronoi cells $\mathcal{V}_{j_2|j_1}$, we continue to divide the term $A_{n,j_1,1}$ into two parts as
\begin{align*}
    A_{n,j_1,1}&=\sum_{j_2:|\mathcal{V}_{j_2|j_1}|=1}\sum_{i_2\in\mathcal{V}_{j_2|j_1}}\exp(\bein)\Big[\exp((\oin)^{\top}\bx)\exp((\ajn)^{\top}\bx)\pi(y|(\ein){\top}\bx+\tin,\vin)\\
    &\hspace{3cm}-\exp((\oj)^{\top}\bx)\exp((\aj)^{\top}\bx)\pi(y|(\ej)^{\top}\bx+\tj,\vj)\Big],\\
    &+\sum_{j_2:|\mathcal{V}_{j_2|j_1}|>1}\sum_{i_2\in\mathcal{V}_{j_2|j_1}}\exp(\bein)\Big[\exp((\oin)^{\top}\bx)\exp((\ajn)^{\top}\bx)\pi(y|(\ein){\top}\bx+\tin,\vin)\\
    &\hspace{3cm}-\exp((\oj)^{\top}\bx)\exp((\aj)^{\top}\bx)\pi(y|(\ej)^{\top}\bx+\tj,\vj)\Big]\\
    :&=A_{n,j_1,1,1}+A_{n,j_1,1,2}.
\end{align*}
Let $\xi(\boldsymbol{\eta},\tau)=\boldsymbol{\eta}^{\top}\bx+\tau$. By applying the first-order Taylor expansion, the term $A_{n,j_1,1,1}$ can be rewritten as
\begin{align*}
    &A_{n,j_1,1,1}=\sum_{j_2:|\mathcal{V}_{j_2|j_1}|=1}\sum_{i_2\in\mathcal{V}_{j_2|j_1}}\sum_{|\balpha|=1}\frac{\exp(\bein)}{2^{\alpha_5}\balpha!}(\doijn)^{\balpha_1}(\dajn)^{\balpha_2}(\deijn)^{\balpha_3}(\dtijn)^{\alpha_4}\\
    &\times (\dvijn)^{\alpha_5}\bx^{\balpha_1+\balpha_2+\balpha_3}\exp((\oj)^{\top}\bx)\exp((\aj)^{\top}\bx)\frac{\partial^{|\balpha_3|+\alpha_4+2\alpha_5}\pi}{\partial\xi^{|\balpha_3|+\alpha_4+2\alpha_5}}(y|(\ej)^{\top}\bx+\tj,\vj)
    \\
    &\hspace{12cm}+R_{n,1,1}(\bx)\\
    &=\sum_{j_2:|\mathcal{V}_{j_2|j_1}|=1}\sum_{|\brho_1|+\brho_2=1}^{2}S_{n,j_2|j_1,\brho_1,\rho_2}\cdot \bx^{\brho_1}\cdot\exp((\oj)^{\top}\bx)\exp((\aj)^{\top}\bx)\\
    &\hspace{6cm}\times\frac{\partial^{\rho_2}\pi}{\partial\xi^{\rho_2}}(y|(\ej)^{\top}\bx+\tj,\vj)+R_{n,1,1}(\bx),
\end{align*}
where $R_{n,1,1}(\bx)$ is a Taylor remainder satisfying $R_{n,1,1}(\bx)/\mathcal{L}_{1n}\to0$ as $n\to\infty$, and
\begin{align*}
    S_{n,j_2|j_1,\brho_1,\rho_2}&:=\sum_{i_2\in\mathcal{V}_{j_2|j_1}}\sum_{(\balpha_1,\balpha_2,\balpha_3,\alpha_4,\alpha_5)\in\mathcal{I}^{SS}_{\brho_1,\rho_2}}\frac{\exp(\bein)}{2^{\alpha_5}\balpha!}(\doijn)^{\balpha_1}(\dajn)^{\balpha_2}(\deijn)^{\balpha_3}\\
    &\hspace{7cm}\times(\dtijn)^{\alpha_4}(\dvijn)^{\alpha_5},
\end{align*}
for any $(\brho_1,\rho_2)\neq(\zerod,0)$ and $j_1\in[k^*_1],j_2\in[k^*_2]$ in which 
\begin{align*}
    \mathcal{I}^{SS}_{\brho_1,\rho_2}:=\{(\balpha_1,\balpha_2,\balpha_3,\alpha_4,\alpha_5)\in\mathbb{R}^d\times\mathbb{R}^d\times\mathbb{R}^d\times\mathbb{R}:\balpha_1+\balpha_2+\balpha_3=\brho_1, |\balpha_3|+\alpha_4+2\alpha_5=\rho_2\}.
\end{align*}
\noindent
For each $(j_1,j_2)\in[k^*_1]\times[k^*_2]$, by applying the Taylor expansion of order $r^{SS}(|\mathcal{V}_{j_2|j_1}|):=\hrj$, we can represent the term $A_{n,j_1,1,2}$ as
\begin{align*}
    &A_{n,j_1,1,2}=\sum_{j_2:|\mathcal{V}_{j_2|j_1}|>1}\sum_{|\brho_1|+\brho_2=1}^{2\hrj}S_{n,j_2|j_1,\brho_1,\rho_2}\cdot \bx^{\brho_1}\cdot\exp((\oj)^{\top}\bx)\exp((\aj)^{\top}\bx)\\
    &\hspace{7cm}\times\frac{\partial^{\rho_2}\pi}{\partial\xi^{\rho_2}}(y|(\ej)^{\top}\bx+\tj,\vj)+R_{n,1,2}(\bx),
\end{align*}
where $R_{n,1,2}(\bx)$ is a Taylor remainder such that $R_{n,1,2}(\bx)/\mathcal{L}_{1n}\to0$ as $n\to\infty$.

\vspace{0.5 em}
\noindent
Subsequently, we rewrite the term $A_{n,j_1,2}$ as follows:
\begin{align*}
   &\sum_{j_2:|\mathcal{V}_{j_2|j_1}|=1}\sum_{i_2\in\mathcal{V}_{j_2|j_1}}\exp(\bein)\Big[\exp((\oin)^{\top}\bx)\nonumber-\exp((\oj)^{\top}\bx)\Big]\exp((\ajn)^{\top}\bx)p^{SS,n}_{j_1}(y|\bx)\\
    &+\sum_{j_2:|\mathcal{V}_{j_2|j_1}|>1}\sum_{i_2\in\mathcal{V}_{j_2|j_1}}\exp(\bein)\Big[\exp((\oin)^{\top}\bx)\nonumber-\exp((\oj)^{\top}\bx)\Big]\exp((\ajn)^{\top}\bx)p^{SS,n}_{j_1}(y|\bx)\\
    :&=A_{n,j_1,2,1}+A_{n,j_1,2,2}.
\end{align*}
By means of the first-order Taylor expansion, we have
\begin{align*}
    &A_{n,j_1,2,1}=\sum_{j_2:|\mathcal{V}_{j_2|j_1}|=1}\sum_{i_2\in\mathcal{V}_{j_2|j_1}}\sum_{|\bpsi|=1}\frac{\exp(\bein)}{\bpsi!}(\doijn)^{\bpsi}\\
    &\hspace{5cm}\times \bx^{\bpsi}\exp((\oj)^{\top}\bx)\exp((\ajn)^{\top}\bx)p^{SS,n}_{j_1}(y|\bx)+R_{n,2,1}(\bx),\\
    &=\sum_{j_2:|\mathcal{V}_{j_2|j_1}|=1}\sum_{|\bpsi|=1}T_{n,j_2|j_1,\bpsi}\cdot\bx^{\bpsi}\exp((\oj)^{\top}\bx)\exp((\ajn)^{\top}\bx)p^{SS,n}_{j_1}(y|\bx)+R_{n,2,1}(\bx),
\end{align*}
where $R_{n,2,1}(\bx)$ is a Taylor remainder such that $R_{n,2,1}(\bx)/\mathcal{L}_{1n}\to0$ as $n\to\infty$, and
\begin{align*}
    T_{n,j_2|j_1,\bpsi}:=\sum_{i_2\in\mathcal{V}_{j_2|j_1}}\frac{\exp(\bein)}{\bpsi!}(\doijn)^{\bpsi},
\end{align*}
for any $j_2\in[k^*_2]$ and $\bpsi\neq\zerod$.

\vspace{0.5 em}
\noindent
At the same time, we apply the second-order Taylor expansion to $A_{n,j_1,2,2}$:
\begin{align*}
    A_{n,j_1,2,2}&=\sum_{j_2:|\mathcal{V}_{j_2|j_1}|>1}\sum_{|\bpsi|=1}^{2}T_{n,j_2|j_1,\bpsi}\cdot\bx^{\bpsi}\exp((\oj)^{\top}\bx)\exp((\ajn)^{\top}\bx)p^{SS,n}_{j_1}(y|\bx)+R_{n,2,2}(\bx),
\end{align*}
where $R_{n,2,2}(\bx)$ is a Taylor remainder such that $R_{n,2,2}(\bx)/\mathcal{L}_{1n}\to0$ as $n\to\infty$.

\vspace{0.5 em}
\noindent
As a result, the term $A_{n}$ can be rewritten as
\begin{align}
    &A_{n}=\sum_{j_1=1}^{k^*_1}\sum_{j_2=1}^{k^*_2}\frac{\exp(\bjn)}{\sum_{j'_2=1}^{k^*_2}\exp((\boldsymbol{\omega}^*_{j'_2|j_1})^{\top}\bx+\beta^*_{j'_2|j_1})}\Bigg[\sum_{|\brho_1|+\brho_2=0}^{2\hrj}S_{n,j_2|j_1,\brho_1,\rho_2}\cdot\bx^{\brho_1}\cdot\exp((\oj)^{\top}\bx)\nonumber\\
    &\times \exp((\aj)^{\top}\bx)\frac{\partial^{\rho_2}\pi}{\partial\xi^{\rho_2}}(y|(\ej)^{\top}\bx+\tj,\vj)+R_{n,1,1}(\bx)+R_{n,1,2}(\bx)\nonumber\\
    \label{eq:decompose_An1_ss}
    &-\sum_{|\bpsi|=0}^{2}T_{n,j_2|j_1,\bpsi}\cdot\bx^{\bpsi}\exp((\oj)^{\top}\bx)\exp((\ajn)^{\top}\bx)p^{SS,n}_{j_1}(y|\bx)-R_{n,2,1}(\bx)-R_{n,2,2}(\bx)\Bigg],
\end{align}
where $S_{n,j_2|j_1,\brho_1,\rho_2}=T_{n,j_2|j_1,\bpsi}=\sum_{i_2\in\mathcal{V}_{j_2|j_1}}\exp(\bein)-\exp(\bej)$ for any $j_2\in[k^*_2]$ where $(\balpha_1,\brho_1,\rho_2)=(\zerod,\zerod,0)$ and $\bpsi=\zerod$.

\vspace{0.5 em}
\noindent
\textbf{Step 1B - Decompose $B_{n}$:} By invoking the first-order Taylor expansion, the term $B_{n}$ defined in equation~\eqref{eq:decompose_Qn_ss} can be rewritten as
\begin{align}
     \label{eq:decompose_Bn_ss}
    B_{n}&=\sum_{j_1=1}^{k^*_1}\exp(\bjn)\sum_{|\bgamma|=1}(\dajn)^{\bgamma}\cdot\bx^{\bgamma}\exp((\aj)^{\top}\bx)p^{SS}_{G_n}(y|\bx) +R_{n,3}(\bx),
\end{align}
where $R_{n,3}(\bx)$ is a Taylor remainder such that $R_{n,3}(\bx)/\mathcal{L}_{1n}\to0$ as $n\to\infty$.

\vspace{0.5 em}
\noindent
From the decomposition in equations~\eqref{eq:decompose_Qn_ss}, \eqref{eq:decompose_An1_ss} and \eqref{eq:decompose_Bn_ss}, we realize that $A_n$, $B_n$ and $C_n$ can be viewed as a combination of elements from the following set union:
\begin{align*}
   &\Bigg\{\bx^{\brho_1}\cdot\exp((\oj)^{\top}\bx)\exp((\aj)^{\top}\bx)\frac{\partial^{\rho_2}\pi}{\partial\xi^{\rho_2}}(y|(\ej)^{\top}\bx+\tj,\vj):j_1\in[k^*_1], \  j_2\in[k^*_2],\\
   &\hspace{9cm} 0\leq|\brho_1|+\rho_2\leq2\hrj\Bigg\}\\
   \cup~&\Bigg\{\frac{\bx^{\bpsi}\exp((\oj)^{\top}\bx)\exp((\ajn)^{\top}\bx)p^{SS,n}_{j_1}(y|\bx)}{\sum_{j'_2=1}^{k^*_2}\exp((\boldsymbol{\omega}^*_{j'_2|j_1})^{\top}\bx+\beta^*_{j'_2|j_1})}:j_1\in[k^*_1], \ j_2\in[k^*_2], \ 0\leq|\bpsi|\leq2\Bigg\}\\
   \cup~&\left\{\bx^{\bgamma}\exp((\aj)^{\top}\bx)p^{SS,n}_{j_1}(y|\bx), \ \bx^{\bgamma}\exp((\aj)^{\top}\bx)p^{SS}_{G_n}(y|\bx): j_1\in[k^*_1], \ 0\leq|\bgamma|\leq1\right\}.
\end{align*}
\noindent
\textbf{Step 2 - Non-vanishing coefficients:} In this stage, we show that not all the coefficients in the representation of $A_n/\mathcal{L}_{1n}$, $B_n/\mathcal{L}_{1n}$ and $C_n/\mathcal{L}_{1n}$ go to zero as $n\to\infty$. Assume that all of them approach zero, then by looking into the coefficients associated with the term 
\begin{itemize}
    \item $\exp((\aj)^{\top}\bx)p^{SS,n}_{j_1}(y|\bx)$ in $C_n/\mathcal{L}_{1n}$, we have
    \begin{align}
        \label{eq:limit_bias_1_ss}
        \frac{1}{\mathcal{L}_{1n}}\cdot\sum_{j_1=1}^{k^*_1}\Big|\exp(\bjn)-\exp(\bj)\Big|\to0.
    \end{align}
    \item $\dfrac{\exp((\oj)^{\top}\bx)\exp((\aj)^{\top}\bx)\pi(y|(\ej)^{\top}\bx+\tj,\vj)}{\sum_{j'_2=1}^{k^*_2}\exp((\boldsymbol{\omega}^*_{j'_2|j_1})^{\top}\bx+\beta^*_{j'_2|j_1})}$ in $A_{n}/\mathcal{L}_{1n}$, we get that
    \begin{align}
        \label{eq:limit_bias_2_ss}
        \frac{1}{\mathcal{L}_{1n}}\cdot\sum_{j_1=1}^{k^*_1}\exp(\bjn)\sum_{j_2=1}^{k^*_2}\Big|\sum_{i_2\in\mathcal{V}_{j_2|j_1}}\exp(\bein)-\exp(\bej)\Big|\to0.
    \end{align}
    \item $\dfrac{\bx^{\bpsi}\exp((\oj)^{\top}\bx)\exp((\ajn)^{\top}\bx)p^{SS,n}_{j_1}(y|\bx)}{\sum_{j'_2=1}^{k^*_2}\exp((\boldsymbol{\omega}^*_{j'_2|j_1})^{\top}\bx+\beta^*_{j'_2|j_1})}$ in $A_{n}/\mathcal{L}_{1n}$ for $j_1\in[k^*_1],j_2\in[k^*_2]:|\mathcal{V}_{j_2|j_1}|=1$ and $\bpsi=e_{d,u}$ where $e_{d,u}:=(0,\ldots,0,\underbrace{1}_{\textit{u-th}},0,\ldots,0)\in\mathbb{N}^{d}$, we receive
    \begin{align*}
         \frac{1}{\mathcal{L}_{1n}}\cdot\sum_{j_1=1}^{k^*_1}\exp(\bjn)\sum_{j_2\in[k^*_2]:|\mathcal{V}_{j_2|j_1}|=1}\sum_{i_2\in\mathcal{V}_{j_2|j_1}}\exp(\bein)\|\oin-\oj\|_1\to0.
    \end{align*}
    Note that since the norm-1 is equivalent to the norm-2, then we can replace the norm-1 with the norm-2, that is,
    \begin{align}
        \label{eq:limit_exact_1_ss}
        \frac{1}{\mathcal{L}_{1n}}\cdot\sum_{j_1=1}^{k^*_1}\exp(\bjn)\sum_{j_2\in[k^*_2]:|\mathcal{V}_{j_2|j_1}|=1}\sum_{i_2\in\mathcal{V}_{j_2|j_1}}\exp(\bein)\|\oin-\oj\|\to0.
    \end{align}
    \item $\dfrac{\exp((\oj)^{\top}\bx)\exp((\aj)^{\top}\bx)\frac{\partial^{\rho_2}\pi}{\partial\xi^{\rho_2}}(y|(\ej)^{\top}\bx+\tj,\vj)}{\sum_{j'_2=1}^{k^*_2}\exp((\boldsymbol{\omega}^*_{j'_2|j_1})^{\top}\bx+\beta^*_{j'_2|j_1})}$ in $A_{n}/\mathcal{L}_{1n}$ for $j_1\in[k^*_1],j_2\in[k^*_2]:|\mathcal{V}_{j_2|j_1}|=1$ and $\rho_2=1$, we have that
    \begin{align}
        \label{eq:limit_exact_3_ss}
         \frac{1}{\mathcal{L}_{1n}}\cdot\sum_{j_1=1}^{k^*_1}\exp(\bjn)\sum_{j_2\in[k^*_2]:|\mathcal{V}_{j_2|j_1}|=1}\exp(\bejn)|\tjn-\tj|\to0.
    \end{align}
    \item $\dfrac{\bx^{\brho_1}\exp((\oj)^{\top}\bx)\exp((\aj)^{\top}\bx)\frac{\partial^{\rho_2}\pi}{\partial\xi^{\rho_2}}(y|(\ej)^{\top}\bx+\tj,\vj)}{\sum_{j'_2=1}^{k^*_2}\exp((\boldsymbol{\omega}^*_{j'_2|j_1})^{\top}\bx+\beta^*_{j'_2|j_1})}$ in $A_{n}/\mathcal{L}_{1n}$ for $j_1\in[k^*_1],j_2\in[k^*_2]:|\mathcal{V}_{j_2|j_1}|=1$, $\brho_1=e_{d,u}$ and $\rho_2=1$, we have that
    \begin{align}
        \label{eq:limit_exact_6_ss}
         \frac{1}{\mathcal{L}_{1n}}\cdot\sum_{j_1=1}^{k^*_1}\exp(\bjn)\sum_{j_2\in[k^*_2]:|\mathcal{V}_{j_2|j_1}|=1}\sum_{i_2\in\mathcal{V}_{j_2|j_1}}\exp(\bejn)\|\ein-\ej\|\to0.
    \end{align}
    \item $\dfrac{\exp((\oj)^{\top}\bx)\exp((\aj)^{\top}\bx)\frac{\partial^{\rho_2}\pi}{\partial\xi^{\rho_2}}(y|(\ej)^{\top}\bx+\tj,\vj)}{\sum_{j'_2=1}^{k^*_2}\exp((\boldsymbol{\omega}^*_{j'_2|j_1})^{\top}\bx+\beta^*_{j'_2|j_1})}$ in $A_{n}/\mathcal{L}_{1n}$ for $j_1\in[k^*_1],j_2\in[k^*_2]:|\mathcal{V}_{j_2|j_1}|=1$ and $\rho_2=2$, we have that
    \begin{align}
        \label{eq:limit_exact_7_ss}
         \frac{1}{\mathcal{L}_{1n}}\cdot\sum_{j_1=1}^{k^*_1}\exp(\bjn)\sum_{j_2\in[k^*_2]:|\mathcal{V}_{j_2|j_1}|=1}\exp(\bejn)|\vjn-\vj|\to0.
    \end{align}
    \item $\bx^{\bgamma}\exp((\aj)^{\top}\bx)p^{SS}_{G_n}(y|\bx)$ in $B_n/\mathcal{L}_{1n}$ for $j_1\in[k^*_1]$ and $\bgamma=e_{d,u}$, we obtain
    \begin{align}
        \label{eq:limit_exact_4_ss}
        \frac{1}{\mathcal{L}_{1n}}\cdot\sum_{j_1=1}^{k^*_1}\exp(\bjn)\|\ajn-\aj\|\to0.
    \end{align}
     \item $\dfrac{\bx^{\bpsi}\exp((\oj)^{\top}\bx)\exp((\ajn)^{\top}\bx)p^{SS,n}_{j_1}(y|\bx)}{\sum_{j'_2=1}^{k^*_2}\exp((\boldsymbol{\omega}^*_{j'_2|j_1})^{\top}\bx+\beta^*_{j'_2|j_1})}$ in $A_{n}/\mathcal{L}_{1n}$ for $j_1\in[k^*_1],j_2\in[k^*_2]:|\mathcal{V}_{j_2|j_1}|>1$ and $\bpsi=2e_{d,u}$, we receive that
    \begin{align}
         \frac{1}{\mathcal{L}_{1n}}\cdot\sum_{j_1=1}^{k^*_1}\exp(\bjn)\sum_{j_2\in[k^*_2]:|\mathcal{V}_{j_2|j_1}|>1}\sum_{i_2\in\mathcal{V}_{j_2|j_1}}\exp(\bein)\|\oin-\oj\|^2\to0.
    \end{align}
\end{itemize}
Combine the above limits together with the loss $\mathcal{L}_{1n}$ in equation~\eqref{eq:loss_l1_ss}, it yields that 
\begin{align*}
    &\frac{1}{\mathcal{L}_{1n}}\cdot\sum_{j_1=1}^{k^*_1}\exp(\bjn)\Bigg[\sum_{j_2:|\mathcal{V}_{j_2|j_1}|>1}\sum_{i_2\in\mathcal{V}_{j_2|j_1}}\exp(\beta^n_{i_2|j_1})\Big(\|\deijn\|^{\frac{\hrj}{2}}+|\dtijn|^{\hrj}\\
    &\hspace{9cm}+|\dvijn|^{\frac{\hrj}{2}}\Big)\Bigg]\not\to0,
\end{align*}
which indicates that
\begin{align*}
    &\frac{1}{\mathcal{L}_{1n}}\cdot\sum_{j_1=1}^{k^*_1}\exp(\bjn)\Bigg[\sum_{j_2:|\mathcal{V}_{j_2|j_1}|>1}\sum_{i_2\in\mathcal{V}_{j_2|j_1}}\exp(\beta^n_{i_2|j_1})\Big(\|\doijn\|^{\hrj}+\|\dajn\|^{\hrj}\nonumber\\
    &\hspace{3cm}+\|\deijn\|^{\frac{\hrj}{2}}+|\dtijn|^{\hrj}+|\dvijn|^{\frac{\hrj}{2}}\Big)\Bigg]\not\to0,
\end{align*}
as $n\to\infty$. Therefore, there exist indices $j^*_1\in[k^*_1]$ and $j^*_2\in[k^*_2]:|\mathcal{V}_{j^*_2|j^*_1}|>1$ such that 
\begin{align}
    &\frac{1}{\mathcal{L}_{1n}}\cdot\sum_{i_2\in\mathcal{V}_{j^*_2|j^*_1}}\exp(\beta^n_{i_2|j^*_1})\Big(\|\boldsymbol{\omega}^n_{i_2|j^*_1}-\boldsymbol{\omega}^*_{j^*_2|j^*_1}\|^{r^{SS}_{j^*_2|j^*_1}}+\|\boldsymbol{a}^n_{j^*_1}-\boldsymbol{a}^*_{j^*_1}\|^{r^{SS}_{j^*_2|j^*_1}}+\|\boldsymbol{\eta}^n_{j^*_1i_2}-\boldsymbol{\eta}^*_{j^*_1j^*_2}\|^{\frac{r^{SS}_{j^*_2|j^*_1}}{2}}\nonumber\\
    \label{eq:non_zero_denominator_ss}
    &\hspace{4cm}+|\tau^n_{j^*_1i_2}-\tau^*_{j^*_1j^*_2}|^{r^{SS}_{j^*_2|j^*_1}}+|\nu^n_{j^*_1i_2}-\nu^*_{j^*_1j^*_2}|^{\frac{r^{SS}_{j^*_2|j^*_1}}{2}}\Big)\not\to0.
\end{align}
WLOG, we may assume that $j^*_1=j^*_2=1$. By examining the coefficients of the terms 
\begin{align*}
    \dfrac{\bx^{\brho_1}\exp((\oj)^{\top}\bx)\exp((\aj)^{\top}\bx)\frac{\partial^{\rho_2}\pi}{\partial\xi^{\rho_2}}(y|(\ej)^{\top}\bx+\tj,\vj)}{\sum_{j'_2=1}^{k^*_2}\exp((\boldsymbol{\omega}^*_{j'_2|j_1})^{\top}\bx+\beta^*_{j'_2|j_1})}
\end{align*}
in $A_{n}/\mathcal{L}_{1n}$ for $j_1=j_2=1$, we have $\exp(b^n_1)S_{n,1|1,\zerod,\brho_1,\rho_2}/\mathcal{L}_{1n}\to0$, or equivalently,
\begin{align}
    &\frac{1}{\mathcal{L}_{1n}}\cdot\sum_{i_2\in\mathcal{V}_{1|1}}\sum_{(\balpha_1,\balpha_2,\balpha_3,\alpha_4,\alpha_5)\in\mathcal{I}^{SS}_{\brho_1,\rho_2}}\frac{\exp(\beta^n_{i_2|1})}{2^{\alpha_5}\balpha!}\cdot(\Delta\boldsymbol{\omega}^n_{1i_21})^{\balpha_1}(\Delta\boldsymbol{a}^n_{1})^{\balpha_2}(\Delta\boldsymbol{\eta}^n_{1i_21})^{\balpha_3}\nonumber\\
    \label{eq:zero_one_limit_ss}
    &\hspace{7cm}\times(\Delta\tau^n_{1i_21})^{\alpha_4}(\Delta\nu^n_{1i_21})^{\alpha_5}\to0.
\end{align}
By dividing the left hand side of equation~\eqref{eq:zero_one_limit_ss} by that of equation~\eqref{eq:non_zero_denominator_ss}, we get 
\begin{align}
    \label{eq:ratio_before_limit_ss}
    \dfrac{\sum_{i_2\in\mathcal{V}_{1|1}}\sum_{(\balpha_1,\balpha_2,\balpha_3,\alpha_4,\alpha_5)\in\mathcal{I}^{SS}_{\brho_1,\rho_2}}\frac{\exp(\beta^n_{i_2|1})}{2^{\alpha_5}\balpha!}\cdot(\Delta\boldsymbol{\omega}^n_{1i_21})^{\balpha_1}(\Delta\boldsymbol{a}^n_{1})^{\balpha_2}(\Delta\boldsymbol{\eta}^n_{1i_21})^{\balpha_3}(\Delta\tau^n_{1i_21})^{\alpha_4}(\Delta\nu^n_{1i_21})^{\alpha_5}}{\sum_{i_2\in\mathcal{V}_{1|1}}\exp(\beta^n_{i_2|1})\Big(\|\Delta\boldsymbol{\omega}^n_{1i_21}\|^{\hrone}+\|\Delta\boldsymbol{a}^n_{1}\|^{\hrone}+\|\Delta\boldsymbol{\eta}^n_{1i_2i}\|^{\frac{\hrone}{2}}+|\Delta\tau^n_{1i_21}|^{\hrone}+|\Delta\nu^n_{1i_21}|^{\frac{\hrone}{2}}\Big)}\to0.
\end{align}
Let us define $\overline{M}_n:=\max\{\|\Delta\boldsymbol{\omega}^n_{1i_21}\|,\|\Delta\boldsymbol{a}^n_{1}\|, \|\Delta\boldsymbol{\eta}^n_{1i_21}\|^{1/2}, \|\Delta\tau^n_{1i_21}\|, \|\Delta\nu^n_{1i_21}\|^{1/2}:i_2\in\mathcal{V}_{1|1}\}$, and $\overline{\beta}_n:=\max_{i_2\in\mathcal{V}_{1|1}}\exp(\beta^n_{i_2|1})$. Since the sequence $\exp(\beta^n_{i_2|1})/\overline{\beta}_n$ is bounded, we can replace it by its subsequence which has a positive limit $p^2_{i_2}:=\lim_{n\to\infty}\exp(\beta^n_{i_2|1})/\overline{\beta}_n$. Note that at least one among the limits $p^2_{i_2}$ must be equal to one. Next, let us define
\begin{align*}
    (\Delta\boldsymbol{\omega}^n_{1i_21})/\overline{M}\to \boldsymbol{q}_{1i_2}& \quad (\Delta\boldsymbol{a}^n_{1})/\overline{M}_n\to \boldsymbol{q}_{2},&  (\Delta\boldsymbol{\eta}^n_{1i_21})/\overline{M}_n\to \boldsymbol{q}_{3i_2},\\
    (\Delta\tau^n_{1i_21})/\overline{M}_n\to q_{4i_2},& \quad (\Delta\nu^n_{1i_21})/2\overline{M}_n\to q_{5i_2} & \quad.
\end{align*}
\noindent
Note that at least one among $\boldsymbol{q}_{1i_2},\boldsymbol{q}_{2},\boldsymbol{q}_{3i_2},q_{4i_2},q_{5i_2}$ must be equal to either 1 or $-1$. 

\vspace{0.5 em}
\noindent
By dividing both the numerator and the denominator of the term in equation~\eqref{eq:ratio_before_limit_ss} by $\overline{\beta}_n\overline{M}^{|\brho_1|+\rho_2}_n$, we obtain the system of polynomial equations:
\begin{align*}
    \sum_{i_2\in\mathcal{V}_{1|1}}\sum_{(\balpha_1,\balpha_2,\balpha_3,\alpha_4,\alpha_5)\in\mathcal{I}^{SS}_{\brho_1,\rho_2}}\frac{1}{\balpha!}\cdot p^2_{i_2}\boldsymbol{q}_{1i_2}^{\balpha_1}\boldsymbol{q}_{2}^{\balpha_2}\boldsymbol{q}_{3i_2}^{\balpha_3}q_{4i_2}^{\alpha_4}q_{5i_2}^{\alpha_5}=0, \quad 1\leq|\brho_1|+\rho_2\leq\hrone.   
\end{align*}
According to the definition of the term $\hrone$, the above system does not have any non-trivial solutions, which is a contradiction. Consequently, at least one among the coefficients in the representation of $A_n/\mathcal{L}_{1n}$, $B_n/\mathcal{L}_{1n}$ and $C_n/\mathcal{L}_{1n}$ must not converge to zero as $n\to\infty$.

\vspace{0.5 em}
\noindent
\textbf{Step 3 - Application of the Fatou's lemma.} In this stage, we show that all the coefficients in the formulations of $A_n/\mathcal{L}_{1n}$, $B_n/\mathcal{L}_{1n}$ and $C_n/\mathcal{L}_{1n}$ go to zero as $n\to\infty$. Denote by $m_n$ the maximum of the absolute values of those coefficients, the result from Step 2 induces that $1/m_n\not\to\infty$. By employing the Fatou's lemma, we have
\begin{align*}
    0=\lim_{n\to\infty}\dfrac{\bbE_{\bbX}[V(p^{SS}_{G_n}(\cdot|\bbX),p^{SS}_{G_*}(\cdot|\bbX))]}{m_n\mathcal{L}_{1n}}\geq\int\liminf_{n\to\infty}\dfrac{|p^{SS}_{G_n}(y|\bx)-p^{SS}_{G_*}(y|\bx)|}{2m_n\mathcal{L}_{1n}}\dint(\bx,y).
\end{align*}
Thus, we deduce that 
\begin{align*}
    \dfrac{|p^{SS}_{G_n}(y|\bx)-p^{SS}_{G_*}(y|\bx)|}{2m_n\mathcal{L}_{1n}}\to0,
\end{align*}
which results in $Q_n/[m_n\mathcal{L}_{1n}]\to0$ as $n\to\infty$ for almost surely $(\bx,y)$. Next, we denote
\begin{align*}
    &\frac{\exp(\bjn)S_{n,j_2|j_1,\brho_1,\rho_2}}{m_n\mathcal{L}_{1n}}\to \phi_{j_2|j_1,\brho_1,\rho_2}, \quad &\frac{\exp(\bjn)T_{n,j_2|j_1,\bpsi}}{m_n\mathcal{L}_{1n}}\to\varphi_{j_2|j_1,\bpsi},\\
    &\frac{\exp(\bjn)(\dajn)^{\bgamma}}{m_n\mathcal{L}_{1n}}\to\lambda_{j_1,\bgamma}, \quad &\frac{\exp(\bjn)-\exp(\bj)}{m_n\mathcal{L}_{1n}}\to\chi_{j_1}
\end{align*}
with a note that at least one among them is non-zero. Then, the decomposition of $Q_n$ in equation~\eqref{eq:decompose_Qn_ss} indicates that
\begin{align*}
    \lim_{n\to\infty}\frac{Q_n}{m_n\mathcal{L}_{1n}}=\lim_{n\to\infty}\frac{A_{n}}{m_n\mathcal{L}_{1n}}-\lim_{n\to\infty}\frac{B_n}{m_n\mathcal{L}_{1n}}+\lim_{n\to\infty}\frac{C_n}{m_n\mathcal{L}_{1n}},
\end{align*}
in which 
\begin{align*}
    &\lim_{n\to\infty}\frac{A_{n}}{m_n\mathcal{L}_{1n}}=\sum_{j_1=1}^{k^*_1}\sum_{j_2=1}^{k^*_2}\Bigg[\sum_{|\brho_1|+\brho_2=0}^{2\hrj}S_{n,j_2|j_1,\brho_1,\rho_2}\cdot\bx^{\brho_1}\exp((\oj)^{\top}\bx)\\
    &\times\exp((\aj)^{\top}\bx)\frac{\partial^{\rho_2}\pi}{\partial\xi^{\rho_2}}(y|(\ej)^{\top}\bx+\tj,\vj)-\sum_{|\bpsi|=0}^{2}\varphi_{j_2|j_1,\bpsi}\cdot\bx^{\bpsi}\exp((\oj)^{\top}\bx)\nonumber\\
    &\hspace{3cm}\times\exp((\aj)^{\top}\bx)p^{SS,*}_{j_1}(y|\bx)\Bigg]\frac{1}{\sum_{j'_2=1}^{k^*_2}\exp((\boldsymbol{\omega}^*_{j'_2|j_1})^{\top}\bx+\beta^*_{j'_2|j_1})},\\
    &\lim_{n\to\infty}\frac{B_{n}}{m_n\mathcal{L}_{1n}}=\sum_{j_1=1}^{k^*_1}\sum_{|\gamma|=1}\lambda_{j_1,\bgamma}\cdot \bx^{\bgamma}\exp((\aj)^{\top}\bx)p^{SS}_{G_*}(y|\bx),\\
    &\lim_{n\to\infty}\frac{C_{n}(\bx)}{m_n\mathcal{L}_{1n}}=\sum_{j_1=1}^{k^*_1}\chi_{j_1}\exp((\aj)^{\top}\bx)\left[p^{SS,*}_{j_1}(y|\bx)-p^{SS}_{G_*}(y|\bx)\right].
\end{align*}
Since the set
\begin{align*}
    &\Bigg\{\dfrac{\bx^{\brho_1}\exp((\oj)^{\top}\bx)\exp((\aj)^{\top}\bx)\frac{\partial^{\rho_2}\pi}{\partial\xi^{\rho_2}}(y|(\ej)^{\top}\bx+\tj,\vj)}{\sum_{j'_2=1}^{k^*_2}\exp((\boldsymbol{\omega}^*_{j'_2|j_1})^{\top}\bx+\beta^*_{j'_2|j_1})}:j_1\in[k^*_1], j_2\in[k^*_2],\\
    &\hspace{8cm}0\leq|\brho_1|+\rho_2\leq2\hrj\Bigg\}\\
    \cup&~\Bigg\{\dfrac{\bx^{\bpsi}\exp((\oj)^{\top}\bx)\exp((\aj)^{\top}\bx)p^{SS,*}_{j_1}(y|\bx)}{\sum_{j'_2=1}^{k^*_2}\exp((\boldsymbol{\omega}^*_{j'_2|j_1})^{\top}\bx+\beta^*_{j'_2|j_1})}:j_1\in[k^*_1], j_2\in[k^*_2],0\leq|\bpsi|\leq2\Bigg\}\\
    \cup&~\Big\{\bx^{\bgamma}\exp((\aj)^{\top}\bx)p^{SS}_{G_*}(y|\bx), \ \exp((\aj)^{\top}\bx)p^{SS,*}_{j_1}(y|\bx), \ \exp((\aj)^{\top}\bx)p^{SS}_{G_*}(y|\bx)\\
    &\hspace{8cm}:j_1\in[k^*_1],0\leq|\bgamma|\leq2\Big\}
\end{align*}
is linearly independent, we obtain that $\phi_{j_2|j_1,\brho_1,\rho_2}=\varphi_{j_2|j_1,\bpsi}=\lambda_{j_1,\bgamma}=\chi_{j_1}=0$ for all $j_1\in[k^*_1]$, $j_2\in[k^*_2]$, $0\leq|\brho_1|+\rho_2\leq2\hrj$, $0\leq|\bpsi|\leq 2$ and $0\leq|\bgamma|\leq 1$, which is a contradiction. As a consequence, we obtain the inequality in equation~\eqref{eq:general_local_inequality_over_ss}. Hence, the proof is completed.

\subsection{Proof of Theorem~\ref{theorem:param_rates_SL}: When $type=SL$}
\label{appendix:softmax_laplace}
When $type=SL$, the corresponding Voronoi loss function is $\mathcal{L}_{(\frac{1}{2}r^{SL},r^{SL},\frac{1}{2}r^{SL})}(G_n,G_*)=\mathcal{L}_{2n}$ where we define
\begin{align}
    &\mathcal{L}_{2n}:=\sum_{j_1=1}^{k^*_1}\Big|\exp(b^n_{j_1})-\exp(b^*_{j_1})\Big|+\sum_{j_1=1}^{k^*_1}\exp(b^n_{j_1})\|\dajn\|+\sum_{j_1=1}^{k^*_1}\exp(b^n_{j_1})\nonumber\\
    &\times\Bigg[\sum_{j_2:|\mathcal{V}_{j_2|j_1}|=1}\sum_{i_2\in\mathcal{V}_{j_2|j_1}}\exp(\beta^n_{i_2|j_1})\Big(\|\doijn\|+\|\deijn\|+|\dtijn|+|\dvijn|\Big)\nonumber\\
    &+\sum_{j_2:|\mathcal{V}_{j_2|j_1}|>1}\sum_{i_2\in\mathcal{V}_{j_2|j_1}}\exp(\beta^n_{i_2|j_1})\Big(\|\doijn\|^{2}+\|\deijn\|^{\frac{\brj}{2}}+|\dtijn|^{\brj}\nonumber\\
     \label{eq:loss_l2_sl}
    &+|\dvijn|^{\frac{\brj}{2}}\Big)\Bigg]+\sum_{j_1=1}^{k^*_1}\exp(b^n_{j_1})\sum_{j_2=1}^{k^*_2}\Big|\sum_{i_2\in\mathcal{V}_{j_2|j_1}}\exp(\beta^n_{i_2|j_1})-\exp(\beta^*_{j_2|j_1})\Big|.
\end{align}
\noindent
\textbf{Step 1 - Taylor expansion:} In this step, we use the Taylor expansion to decompose the term
\begin{align*}
    Q_n:=\left[\sum_{j_1=1}^{k^*_1}\exp((\aj)^{\top}\bx+\bj)\right][p^{SL}_{G_n}(y|\bx)-p^{SL}_{G_*}(y|\bx)].
\end{align*}
Prior to that, let us denote
\begin{align*}
    p^{SL,n}_{j_1}(y|\bx)&:=\sum_{j_2=1}^{k^*_2}\sum_{i_2\in\mathcal{V}_{j_2|j_1}}\softmax(-\|\oin-\bx\|+\beta^n_{i_2|j_1}) \pi(y|(\ein)^{\top}\bx+\tin,\vin),\\
    p^{SL,*}_{j_1}(y|\bx)&:=\sum_{j_2=1}^{k^*_2}\softmax(-\|\oj-\bx\|+\bej) \pi(y|(\ej)^{\top}\bx+\tj,\vj).
\end{align*}
Then, the quantity $Q_n$ is divided into three terms as
\begin{align}
    Q_n&=\sum_{j_1=1}^{k^*_1}\exp(\bjn)\left[\exp((\ajn)^{\top}\bx)p^{SL,n}_{j_1}(y|\bx)-\exp((\aj)^{\top}\bx)p^{SL,*}_{j_1}(y|\bx)\right]\nonumber\\
    &-\sum_{j_1=1}^{k^*_1}\exp(\bjn)\left[\exp((\ajn)^{\top}\bx)-\exp((\aj)^{\top}\bx)\right]p^{SL}_{G_n}(y|\bx)\nonumber\\
    &+\sum_{j_1=1}^{k^*_1}\left(\exp(\bjn)-\exp(\bj)\right)\exp((\aj)^{\top}\bx)\left[p^{SL,n}_{j_1}(y|\bx)-p^{SL}_{G_n}(y|\bx)\right]\nonumber\\
     \label{eq:decompose_Qn_sl}
    :&=A_n-B_n+C_n.
\end{align}
\textbf{Step 1A - Decompose $A_n$:} We continue to decompose $A_n$:
\begin{align*}
    A_n&:=\sum_{j_1=1}^{k^*_1}\frac{\exp(\bjn)}{\sum_{j'_2=1}^{k^*_2}\exp(-\|\boldsymbol{\omega}^*_{j'_2|j_1}-\bx\|+\beta^*_{j'_2|j_1})}[A_{n,j_1,1}+A_{n,j_1,2}+A_{n,j_1,3}],
\end{align*}
in which
\begin{align*}
    A_{n,j_1,1}&:=\sum_{j_2=1}^{k^*_2}\sum_{i_2\in\mathcal{V}_{j_2|j_1}}\exp(\beta^n_{i_2|j_1})\Big[\exp(-\|\oin-\bx\|)\exp((\ajn)^{\top}\bx)\pi(y|(\ein){\top}\bx+\tin,\vin)\nonumber\\
    &\hspace{3cm}-\exp(-\|\oj-\bx\|)\exp((\aj)^{\top}\bx)\pi(y|(\ej)^{\top}\bx+\tj,\vj)\Big],\\
    A_{n,j_1,2}&:=\sum_{j_2=1}^{k^*_2}\sum_{i_2\in\mathcal{V}_{j_2|j_1}}\exp(\beta^n_{i_2|j_1})\Big[\exp(-\|\oin-\bx\|)-\exp(-\|\oj-\bx\|)\Big]\nonumber\\
    &\hspace{7cm}\times\exp((\ajn)^{\top}\bx)p^{SL,n}_{j_1}(y|\bx),\nonumber\\
    A_{n,j_1,3}&:=\sum_{j_2=1}^{k^*_2}\Big(\sum_{i_2\in\mathcal{V}_{j_2|j_1}}\exp(\bein)-\exp(\bej)\Big)\exp(-\|\oj-\bx\|)\\
    &\hspace{1cm}\times[\exp((\aj)^{\top}\bx)\pi(y|(\ej)^{\top}\bx+\tj,\vj)-\exp((\ajn)^{\top}\bx)p^{SL,n}_{j_1}(y|\bx)].
\end{align*}
Based on the cardinality of the Voronoi cells $\mathcal{V}_{j_2|j_1}$, we proceed to divide the term $A_{n,j_1,1}$ into two parts as
\begin{align*}
    A_{n,j_1,1}&=\sum_{j_2:|\mathcal{V}_{j_2|j_1}|=1}\sum_{i_2\in\mathcal{V}_{j_2|j_1}}\exp(\bein)\Big[\exp(-\|\oin-\bx\|)\exp((\ajn)^{\top}\bx)\pi(y|(\ein){\top}\bx+\tin,\vin)\\
    &\hspace{3cm}-\exp(-\|\oj-\bx\|)\exp((\aj)^{\top}\bx)\pi(y|(\ej)^{\top}\bx+\tj,\vj)\Big],\\
    &+\sum_{j_2:|\mathcal{V}_{j_2|j_1}|>1}\sum_{i_2\in\mathcal{V}_{j_2|j_1}}\exp(\bein)\Big[\exp(-\|\oin-\bx\|)\exp((\ajn)^{\top}\bx)\pi(y|(\ein){\top}\bx+\tin,\vin)\\
    &\hspace{3cm}-\exp(-\|\oj-\bx\|)\exp((\aj)^{\top}\bx)\pi(y|(\ej)^{\top}\bx+\tj,\vj)\Big]\\
    :&=A_{n,j_1,1,1}+A_{n,j_1,1,2}.
\end{align*}
Let us denote $F(\bx;\boldsymbol{\omega}):=\exp(-\|\boldsymbol{\omega}-\bx\|)$ and $\xi(\boldsymbol{\eta},\tau)=\boldsymbol{\eta}^{\top}\bx+\tau$. By means of the first-order Taylor expansion, $A_{n,j_1,1,1}$ can be represented as
\begin{align*}
    &A_{n,j_1,1,1}=\sum_{j_2:|\mathcal{V}_{j_2|j_1}|=1}\sum_{i_2\in\mathcal{V}_{j_2|j_1}}\sum_{|\balpha|=1}\frac{\exp(\bein)}{2^{\alpha_5}\balpha!}(\doijn)^{\balpha_1}(\dajn)^{\balpha_2}(\deijn)^{\balpha_3}(\dtijn)^{\alpha_4}\\
    &\times (\dvijn)^{\alpha_5}\bx^{\balpha_2+\balpha_3}\frac{\partial^{|\balpha_1|}F}{\partial\boldsymbol{\omega}^{\balpha_1}}(\bx;\oj)\exp((\aj)^{\top}\bx)\frac{\partial^{|\balpha_3|+\alpha_4+2\alpha_5}\pi}{\partial\xi^{|\balpha_3|+\alpha_4+2\alpha_5}}(y|(\ej)^{\top}\bx+\tj,\vj)
    +R_{n,1,1}(\bx)\\
    &=\sum_{j_2:|\mathcal{V}_{j_2|j_1}|=1}\sum_{|\balpha_1|=0}^{1}\sum_{|\brho_1|+\brho_2=0\vee 1-|\balpha_1|}^{2(1-|\balpha_1|)}S_{n,j_2|j_1,\balpha_1,\brho_1,\rho_2}\cdot \bx^{\brho_1}\cdot\frac{\partial^{|\balpha_1|}F}{\partial\boldsymbol{\omega}^{\balpha_1}}(\bx;\oj)\exp((\aj)^{\top}\bx)\\
    &\hspace{7cm}\times\frac{\partial^{\rho_2}\pi}{\partial\xi^{\rho_2}}(y|(\ej)^{\top}\bx+\tj,\vj)+R_{n,1,1}(\bx),
\end{align*}
where $R_{n,1,1}(\bx)$ is a Taylor remainder such that $R_{n,1,1}(\bx)/\mathcal{L}_{2n}\to0$ as $n\to\infty$, and
\begin{align*}
    S_{n,j_2|j_1,\balpha_1,\brho_1,\rho_2}&:=\sum_{i_2\in\mathcal{V}_{j_2|j_1}}\sum_{(\balpha_2,\balpha_3,\alpha_4,\alpha_5)\in\mathcal{I}^{SL}_{\brho_1,\rho_2}}\frac{\exp(\bein)}{2^{\alpha_5}\balpha!}(\doijn)^{\balpha_1}(\dajn)^{\balpha_2}(\deijn)^{\balpha_3}\\
    &\hspace{7cm}\times(\dtijn)^{\alpha_4}(\dvijn)^{\alpha_5},
\end{align*}
for any $(\balpha_1,\brho_1,\rho_2)\neq(\zerod,\zerod,0)$ and $j_1\in[k^*_1],j_2\in[k^*_2]$ in which 
\begin{align*}
    \mathcal{I}^{SL}_{\brho_1,\rho_2}:=\{(\balpha_2,\balpha_3,\alpha_4,\alpha_5)\in\mathbb{R}^d\times\mathbb{R}^d\times\mathbb{R}^d\times\mathbb{R}:\balpha_2+\balpha_3=\brho_1, |\balpha_3|+\alpha_4+2\alpha_5=\rho_2\}.
\end{align*}
\noindent
For each $(j_1,j_2)\in[k^*_1]\times[k^*_2]$, by applying the Taylor expansion of order $r^{SL}(|\mathcal{V}_{j_2|j_1}|):=\brj$, the term $A_{n,j_1,1,2}$ can be rewritten as
\begin{align*}
    &A_{n,j_1,1,2}=\sum_{j_2:|\mathcal{V}_{j_2|j_1}|>1}\sum_{|\balpha_1|=1}^{\brj}\sum_{|\brho_1|+\brho_2=0\vee 1-|\balpha_1|}^{2(\brj-|\balpha_1|)}S_{n,j_2|j_1,\balpha_1,\brho_1,\rho_2}\cdot \bx^{\brho_1}\cdot\frac{\partial^{|\balpha_1|}F}{\partial\boldsymbol{\omega}^{\balpha_1}}(\bx;\oj)\exp((\aj)^{\top}\bx)\\
    &\hspace{7cm}\times\frac{\partial^{\rho_2}\pi}{\partial\xi^{\rho_2}}(y|(\ej)^{\top}\bx+\tj,\vj)+R_{n,1,2}(\bx),
\end{align*}
where $R_{n,1,2}(\bx)$ is a Taylor remainder such that $R_{n,1,2}(\bx)/\mathcal{L}_{2n}\to0$ as $n\to\infty$.

\vspace{0.5 em}
\noindent
Next, we rewrite the term $A_{n,j_1,2}$ as follows:
\begin{align*}
   &\sum_{j_2:|\mathcal{V}_{j_2|j_1}|=1}\sum_{i_2\in\mathcal{V}_{j_2|j_1}}\exp(\bein)\Big[\exp(-\|\oin-\bx\|)\nonumber-\exp(-\|\oj-\bx\|)\Big]\exp((\ajn)^{\top}\bx)p^{SL,n}_{j_1}(y|\bx)\\
    &+\sum_{j_2:|\mathcal{V}_{j_2|j_1}|>1}\sum_{i_2\in\mathcal{V}_{j_2|j_1}}\exp(\bein)\Big[\exp(-\|\oin-\bx\|)\nonumber-\exp(-\|\oj-\bx\|)\Big]\exp((\ajn)^{\top}\bx)p^{SL,n}_{j_1}(y|\bx)\\
    :&=A_{n,j_1,2,1}+A_{n,j_1,2,2}.
\end{align*}
By applying the first-order Taylor expansion, we have
\begin{align*}
    &A_{n,j_1,2,1}=\sum_{j_2:|\mathcal{V}_{j_2|j_1}|=1}\sum_{i_2\in\mathcal{V}_{j_2|j_1}}\sum_{|\bpsi|=1}\frac{\exp(\bein)}{\bpsi!}(\doijn)^{\bpsi}\\
    &\hspace{5cm}\times\frac{\partial^{|\bpsi|}F}{\partial\boldsymbol{\omega}^{\bpsi}}(\bx;\oj)\exp((\ajn)^{\top}\bx)p^{SL,n}_{j_1}(y|\bx)+R_{n,2,1}(\bx),\\
    &=\sum_{j_2:|\mathcal{V}_{j_2|j_1}|=1}\sum_{|\bpsi|=1}T_{n,j_2|j_1,\bpsi}\cdot\frac{\partial^{|\bpsi|}F}{\partial\boldsymbol{\omega}^{\bpsi}}(\bx;\oj)\exp((\ajn)^{\top}\bx)p^{SL,n}_{j_1}(y|\bx)+R_{n,2,1}(\bx),
\end{align*}
where $R_{n,2,1}(\bx)$ is a Taylor remainder such that $R_{n,2,1}(\bx)/\mathcal{L}_{2n}\to0$ as $n\to\infty$, and
\begin{align*}
    T_{n,j_2|j_1,\bpsi}:=\sum_{i_2\in\mathcal{V}_{j_2|j_1}}\frac{\exp(\bein)}{\bpsi!}(\doijn)^{\bpsi},
\end{align*}
for any $j_2\in[k^*_2]$ and $\bpsi\neq\zerod$.

\vspace{0.5 em}
\noindent
Meanwhile, we employ the second-order Taylor expansion to $A_{n,j_1,2,2}$:
\begin{align*}
    A_{n,j_1,2,2}&=\sum_{j_2:|\mathcal{V}_{j_2|j_1}|>1}\sum_{|\bpsi|=1}^{2}T_{n,j_2|j_1,\bpsi}\cdot\frac{\partial^{|\bpsi|}F}{\partial\boldsymbol{\omega}^{\bpsi}}(\bx;\oj)\exp((\ajn)^{\top}\bx)p^{SL,n}_{j_1}(y|\bx)+R_{n,2,2}(\bx),
\end{align*}
where $R_{n,2,2}(\bx)$ is a Taylor remainder such that $R_{n,2,2}(\bx)/\mathcal{L}_{2n}\to0$ as $n\to\infty$.

\vspace{0.5 em}
\noindent
As a result, the term $A_{n}$ can be rewritten as
\begin{align}
    &A_{n}=\sum_{j_1=1}^{k^*_1}\sum_{j_2=1}^{k^*_2}\frac{\exp(\bjn)}{\sum_{j'_2=1}^{k^*_2}\exp(-\|\boldsymbol{\omega}^*_{j'_2|j_1}-\bx\|+\beta^*_{j'_2|j_1})}\Bigg[\sum_{|\balpha_1|=0}^{\brj}\sum_{|\brho_1|+\brho_2=0\vee 1-|\balpha_1|}^{2(\brj-|\balpha_1|)}S_{n,j_2|j_1,\balpha_1,\brho_1,\rho_2}\nonumber\\
    &\times \bx^{\brho_1}\cdot\frac{\partial^{|\balpha_1|}F}{\partial\boldsymbol{\omega}^{\balpha_1}}(\bx;\oj)\exp((\aj)^{\top}\bx)\frac{\partial^{\rho_2}\pi}{\partial\xi^{\rho_2}}(y|(\ej)^{\top}\bx+\tj,\vj)+R_{n,1,1}(\bx)+R_{n,1,2}(\bx)\nonumber\\
    \label{eq:decompose_An1_sl}
    &-\sum_{|\bpsi|=0}^{2}T_{n,j_2|j_1,\bpsi}\cdot\frac{\partial^{|\bpsi|}F}{\partial\boldsymbol{\omega}^{\bpsi}}(\bx;\oj)\exp((\ajn)^{\top}\bx)p^{SL,n}_{j_1}(y|\bx)-R_{n,2,1}(\bx)-R_{n,2,2}(\bx)\Bigg],
\end{align}
where $S_{n,j_2|j_1,\balpha_1,\brho_1,\rho_2}=T_{n,j_2|j_1,\bpsi}=\sum_{i_2\in\mathcal{V}_{j_2|j_1}}\exp(\bein)-\exp(\bej)$ for any $j_2\in[k^*_2]$ where $(\balpha_1,\brho_1,\rho_2)=(\zerod,\zerod,0)$ and $\bpsi=\zerod$.

\vspace{0.5 em}
\noindent
\textbf{Step 1B - Decompose $B_{n}$:} By invoking the first-order Taylor expansion, we decompose the term $B_{n}$ defined in equation~\eqref{eq:decompose_Qn_sl} as
\begin{align}
     \label{eq:decompose_Bn_sl}
    B_{n}&=\sum_{j_1=1}^{k^*_1}\exp(\bjn)\sum_{|\bgamma|=1}(\dajn)^{\bgamma}\cdot\bx^{\bgamma}\exp((\aj)^{\top}\bx)p^{SL}_{G_n}(y|\bx) +R_{n,3}(\bx),
\end{align}
where $R_{n,3}(\bx)$ is a Taylor remainder such that $R_{n,3}(\bx)/\mathcal{L}_{2n}\to0$ as $n\to\infty$.

\vspace{0.5 em}
\noindent
It can be seen from the decomposition in equations~\eqref{eq:decompose_Qn_sl}, \eqref{eq:decompose_An1_sl} and \eqref{eq:decompose_Bn_sl} that $A_n$, $B_n$ and $C_n$ can be treated as a linear combination of elements from the following set union:
\begin{align*}
   &\Bigg\{\bx^{\brho_1}\cdot\frac{\partial^{|\balpha_1|}F}{\partial\boldsymbol{\omega}^{\balpha_1}}(\bx;\oj)\exp((\aj)^{\top}\bx)\frac{\partial^{\rho_2}\pi}{\partial\xi^{\rho_2}}(y|(\ej)^{\top}\bx+\tj,\vj):j_1\in[k^*_1], \  j_2\in[k^*_2],\\
   &\hspace{4cm}0\leq|\balpha_1|\leq\brj, \ 0\leq|\brho_1|+\rho_2\leq2(\brj-|\balpha_1|)\Bigg\}\\
   \cup~&\Bigg\{\frac{\frac{\partial^{|\bpsi|}F}{\partial\boldsymbol{\omega}^{\bpsi}}(\bx;\oj)\exp((\ajn)^{\top}\bx)p^{SL,n}_{j_1}(y|\bx)}{\sum_{j'_2=1}^{k^*_2}\exp(-\|\boldsymbol{\omega}^*_{j'_2|j_1}-\bx\|+\beta^*_{j'_2|j_1})}:j_1\in[k^*_1], \ j_2\in[k^*_2], \ 0\leq|\bpsi|\leq2\Bigg\}\\
   \cup~&\left\{\bx^{\bgamma}\exp((\aj)^{\top}\bx)p^{SL,n}_{j_1}(y|\bx), \ \bx^{\bgamma}\exp((\aj)^{\top}\bx)p^{SL}_{G_n}(y|\bx): j_1\in[k^*_1], \ 0\leq|\bgamma|\leq1\right\}.
\end{align*}

\vspace{0.5 em}
\noindent
\textbf{Step 2 - Non-vanishing coefficients:} In this stage, we illustrate that not all the coefficients in the representation of $A_n/\mathcal{L}_{2n}$, $B_n/\mathcal{L}_{2n}$ and $C_n/\mathcal{L}_{2n}$ go to zero as $n\to\infty$. Suppose that all of them approach zero, then we examine the coefficients associated with the term 
\begin{itemize}
    \item $\exp((\aj)^{\top}\bx)p^{SL,n}_{j_1}(y|\bx)$ in $C_n/\mathcal{L}_{2n}$, we have
    \begin{align}
        \label{eq:limit_bias_1_sl}
        \frac{1}{\mathcal{L}_{2n}}\cdot\sum_{j_1=1}^{k^*_1}\Big|\exp(\bjn)-\exp(\bj)\Big|\to0.
    \end{align}
    \item $\dfrac{F(\bx;\oj)\exp((\aj)^{\top}\bx)\pi(y|(\ej)^{\top}\bx+\tj,\vj)}{\sum_{j'_2=1}^{k^*_2}\exp(-\|\boldsymbol{\omega}^*_{j'_2|j_1}-\bx\|+\beta^*_{j'_2|j_1})}$ in $A_{n}/\mathcal{L}_{2n}$, we get that
    \begin{align}
        \label{eq:limit_bias_2_sl}
        \frac{1}{\mathcal{L}_{2n}}\cdot\sum_{j_1=1}^{k^*_1}\exp(\bjn)\sum_{j_2=1}^{k^*_2}\Big|\sum_{i_2\in\mathcal{V}_{j_2|j_1}}\exp(\bein)-\exp(\bej)\Big|\to0.
    \end{align}
    \item $\dfrac{\frac{\partial^{|\balpha_1|}F}{\partial\boldsymbol{\omega}^{\balpha_1}}(\bx;\oj)\exp((\ajn)^{\top}\bx)\pi(y|(\ej)^{\top}\bx+\tj,\vj)}{\sum_{j'_2=1}^{k^*_2}\exp(-\|\boldsymbol{\omega}^*_{j'_2|j_1}-\bx\|+\beta^*_{j'_2|j_1})}$ in $A_{n}/\mathcal{L}_{2n}$ for $j_1\in[k^*_1],j_2\in[k^*_2]:|\mathcal{V}_{j_2|j_1}|=1$ and $\balpha_1=e_{d,u}$ where $e_{d,u}:=(0,\ldots,0,\underbrace{1}_{\textit{u-th}},0,\ldots,0)\in\mathbb{N}^{d}$, we receive
    \begin{align*}
         \frac{1}{\mathcal{L}_{2n}}\cdot\sum_{j_1=1}^{k^*_1}\exp(\bjn)\sum_{j_2\in[k^*_2]:|\mathcal{V}_{j_2|j_1}|=1}\sum_{i_2\in\mathcal{V}_{j_2|j_1}}\exp(\bein)\|\oin-\oj\|_1\to0.
    \end{align*}
    Note that since the norm-1 is equivalent to the norm-2, then we can replace the norm-1 with the norm-2, that is,
    \begin{align}
        \label{eq:limit_exact_1_sl}
        \frac{1}{\mathcal{L}_{2n}}\cdot\sum_{j_1=1}^{k^*_1}\exp(\bjn)\sum_{j_2\in[k^*_2]:|\mathcal{V}_{j_2|j_1}|=1}\sum_{i_2\in\mathcal{V}_{j_2|j_1}}\exp(\bein)\|\oin-\oj\|\to0.
    \end{align}
    \item $\dfrac{F(\bx;\oj)\exp((\aj)^{\top}\bx)\frac{\partial^{\rho_2}\pi}{\partial\xi^{\rho_2}}(y|(\ej)^{\top}\bx+\tj,\vj)}{\sum_{j'_2=1}^{k^*_2}\exp(-\|\boldsymbol{\omega}^*_{j'_2|j_1}-\bx\|+\beta^*_{j'_2|j_1})}$ in $A_{n}/\mathcal{L}_{2n}$ for $j_1\in[k^*_1],j_2\in[k^*_2]:|\mathcal{V}_{j_2|j_1}|=1$ and $\rho_2=1$, we have that
    \begin{align}
        \label{eq:limit_exact_3_sl}
         \frac{1}{\mathcal{L}_{2n}}\cdot\sum_{j_1=1}^{k^*_1}\exp(\bjn)\sum_{j_2\in[k^*_2]:|\mathcal{V}_{j_2|j_1}|=1}\exp(\bejn)|\tjn-\tj|\to0.
    \end{align}
    \item $\dfrac{\bx^{\brho_1}F(\bx;\oj)\exp((\aj)^{\top}\bx)\frac{\partial^{\rho_2}\pi}{\partial\xi^{\rho_2}}(y|(\ej)^{\top}\bx+\tj,\vj)}{\sum_{j'_2=1}^{k^*_2}\exp(-\|\boldsymbol{\omega}^*_{j'_2|j_1}-\bx\|+\beta^*_{j'_2|j_1})}$ in $A_{n}/\mathcal{L}_{2n}$ for $j_1\in[k^*_1],j_2\in[k^*_2]:|\mathcal{V}_{j_2|j_1}|=1$, $\brho_1=e_{d,u}$ and $\rho_2=1$, we have that
    \begin{align}
        \label{eq:limit_exact_6_sl}
         \frac{1}{\mathcal{L}_{2n}}\cdot\sum_{j_1=1}^{k^*_1}\exp(\bjn)\sum_{j_2\in[k^*_2]:|\mathcal{V}_{j_2|j_1}|=1}\sum_{i_2\in\mathcal{V}_{j_2|j_1}}\exp(\bejn)\|\ein-\ej\|\to0.
    \end{align}
    \item $\dfrac{F(\bx;\oj)\exp((\aj)^{\top}\bx)\frac{\partial^{\rho_2}\pi}{\partial\xi^{\rho_2}}(y|(\ej)^{\top}\bx+\tj,\vj)}{\sum_{j'_2=1}^{k^*_2}\exp(-\|\boldsymbol{\omega}^*_{j'_2|j_1}-\bx\|+\beta^*_{j'_2|j_1})}$ in $A_{n}/\mathcal{L}_{2n}$ for $j_1\in[k^*_1],j_2\in[k^*_2]:|\mathcal{V}_{j_2|j_1}|=1$ and $\rho_2=2$, we have that
    \begin{align}
        \label{eq:limit_exact_7_sl}
         \frac{1}{\mathcal{L}_{2n}}\cdot\sum_{j_1=1}^{k^*_1}\exp(\bjn)\sum_{j_2\in[k^*_2]:|\mathcal{V}_{j_2|j_1}|=1}\exp(\bejn)|\vjn-\vj|\to0.
    \end{align}
    \item $\bx^{\bgamma}\exp((\aj)^{\top}\bx)p^{SL}_{G_n}(y|\bx)$ in $B_n/\mathcal{L}_{2n}$ for $j_1\in[k^*_1]$ and $\bgamma=e_{d,u}$, we obtain
    \begin{align}
        \label{eq:limit_exact_4_sl}
        \frac{1}{\mathcal{L}_{2n}}\cdot\sum_{j_1=1}^{k^*_1}\exp(\bjn)\|\ajn-\aj\|\to0.
    \end{align}
     \item $\dfrac{\frac{\partial^{|\balpha_1|}F}{\partial\boldsymbol{\omega}^{\balpha_1}}(\bx;\oj)\exp((\aj)^{\top}\bx)\pi(y|(\ej)^{\top}\bx+\tj,\vj)}{\sum_{j'_2=1}^{k^*_2}\exp(-\|\boldsymbol{\omega}^*_{j'_2|j_1}-\bx\|+\beta^*_{j'_2|j_1})}$ in $A_{n}/\mathcal{L}_{2n}$ for $j_1\in[k^*_1],j_2\in[k^*_2]:|\mathcal{V}_{j_2|j_1}|>1$ and $\balpha_1=2e_{d,u}$, we receive that
    \begin{align}
         \frac{1}{\mathcal{L}_{2n}}\cdot\sum_{j_1=1}^{k^*_1}\exp(\bjn)\sum_{j_2\in[k^*_2]:|\mathcal{V}_{j_2|j_1}|>1}\sum_{i_2\in\mathcal{V}_{j_2|j_1}}\exp(\bein)\|\oin-\oj\|^2\to0.
    \end{align}
\end{itemize}
Putting the above limits together with the formulation of the loss $\mathcal{L}_{2n}$ in equation~\eqref{eq:loss_l2_sl}, we deduce that 
\begin{align*}
    &\frac{1}{\mathcal{L}_{2n}}\cdot\sum_{j_1=1}^{k^*_1}\exp(\bjn)\Bigg[\sum_{j_2:|\mathcal{V}_{j_2|j_1}|>1}\sum_{i_2\in\mathcal{V}_{j_2|j_1}}\exp(\beta^n_{i_2|j_1})\Big(\|\deijn\|^{\frac{\brj}{2}}+|\dtijn|^{\brj}\\
    &\hspace{9cm}+|\dvijn|^{\frac{\brj}{2}}\Big)\Bigg]\not\to0,
\end{align*}
which also suggests that
\begin{align*}
    &\frac{1}{\mathcal{L}_{2n}}\cdot\sum_{j_1=1}^{k^*_1}\exp(\bjn)\Bigg[\sum_{j_2:|\mathcal{V}_{j_2|j_1}|>1}\sum_{i_2\in\mathcal{V}_{j_2|j_1}}\exp(\beta^n_{i_2|j_1})\Big(\|\dajn\|^{\brj}+\|\deijn\|^{\frac{\brj}{2}}\nonumber\\
    &\hspace{6cm}+|\dtijn|^{\brj}+|\dvijn|^{\frac{\brj}{2}}\Big)\Bigg]\not\to0,
\end{align*}
as $n\to\infty$. Thus, we can find indices $j^*_1\in[k^*_1]$ and $j^*_2\in[k^*_2]:|\mathcal{V}_{j^*_2|j^*_1}|>1$ such that 
\begin{align}
    &\frac{1}{\mathcal{L}_{2n}}\cdot\sum_{i_2\in\mathcal{V}_{j^*_2|j^*_1}}\exp(\beta^n_{i_2|j^*_1})\Big(\|\boldsymbol{a}^n_{j^*_1}-\boldsymbol{a}^*_{j^*_1}\|^{r^{SL}_{j^*_2|j^*_1}}+\|\boldsymbol{\eta}^n_{j^*_1i_2}-\boldsymbol{\eta}^*_{j^*_1j^*_2}\|^{\frac{r^{SL}_{j^*_2|j^*_1}}{2}}\nonumber\\
    \label{eq:non_zero_denominator}
    &\hspace{4cm}+|\tau^n_{j^*_1i_2}-\tau^*_{j^*_1j^*_2}|^{r^{SL}_{j^*_2|j^*_1}}+|\nu^n_{j^*_1i_2}-\nu^*_{j^*_1j^*_2}|^{\frac{r^{SL}_{j^*_2|j^*_1}}{2}}\Big)\not\to0.
\end{align}
WLOG, we may assume that $j^*_1=j^*_2=1$. By considering the coefficients of the terms 
\begin{align*}
    \dfrac{\bx^{\brho_1}F(\bx;\oj)\exp((\aj)^{\top}\bx)\frac{\partial^{\rho_2}\pi}{\partial\xi^{\rho_2}}(y|(\ej)^{\top}\bx+\tj,\vj)}{\sum_{j'_2=1}^{k^*_2}\exp(-\|\boldsymbol{\omega}^*_{j'_2|j_1}-\bx\|+\beta^*_{j'_2|j_1})}
\end{align*}
in $A_{n}/\mathcal{L}_{2n}$ for $j_1=j_2=1$, we have $\exp(b^n_1)S_{n,1|1,\zerod,\brho_1,\rho_2}/\mathcal{L}_{2n}\to0$, or equivalently,
\begin{align}
    &\frac{1}{\mathcal{L}_{2n}}\cdot\sum_{i_2\in\mathcal{V}_{1|1}}\sum_{(\balpha_2,\balpha_3,\alpha_4,\alpha_5)\in\mathcal{I}^{SL}_{\brho_1,\rho_2}}\frac{\exp(\beta^n_{i_2|1})}{2^{\alpha_5}\balpha_2!\balpha_3!\alpha_4!\alpha_5!}\cdot(\Delta\boldsymbol{a}^n_{1})^{\balpha_2}(\Delta\boldsymbol{\eta}^n_{1i_21})^{\balpha_3}\nonumber\\
    \label{eq:zero_one_limit}
    &\hspace{7cm}\times(\Delta\tau^n_{1i_21})^{\alpha_4}(\Delta\nu^n_{1i_21})^{\alpha_5}\to0.
\end{align}
By dividing the left hand side of equation~\eqref{eq:zero_one_limit} by that of equation~\eqref{eq:non_zero_denominator}, we get 
\begin{align}
    \label{eq:ratio_before_limit}
    \dfrac{\sum_{i_2\in\mathcal{V}_{1|1}}\sum_{(\balpha_2,\balpha_3,\alpha_4,\alpha_5)\in\mathcal{I}^{SL}_{\brho_1,\rho_2}}\frac{\exp(\beta^n_{i_2|1})}{2^{\alpha_5}\balpha_2!\balpha_3!\alpha_4!\alpha_5!}\cdot(\Delta\boldsymbol{a}^n_{1})^{\balpha_2}(\Delta\boldsymbol{\eta}^n_{1i_21})^{\balpha_3}(\Delta\tau^n_{1i_21})^{\alpha_4}(\Delta\nu^n_{1i_21})^{\alpha_5}}{\sum_{i_2\in\mathcal{V}_{1|1}}\exp(\beta^n_{i_2|1})\Big(\|\Delta\boldsymbol{a}^n_{1}\|^{\brone}+\|\Delta\boldsymbol{\eta}^n_{1i_2i}\|^{\frac{\brone}{2}}+|\Delta\tau^n_{1i_21}|^{\brone}+|\Delta\nu^n_{1i_21}|^{\frac{\brone}{2}}\Big)}\to0.
\end{align}
Let us define $\overline{M}_n:=\max\{\|\Delta\boldsymbol{a}^n_{1}\|, \|\Delta\boldsymbol{\eta}^n_{1i_2i}\|^{1/2}, \|\Delta\tau^n_{1i_21}\|, \|\Delta\nu^n_{1i_21}\|^{1/2}:i_2\in\mathcal{V}_{1|1}\}$, and $\overline{\beta}_n:=\max_{i_2\in\mathcal{V}_{1|1}}\exp(\beta^n_{i_2|1})$. Since the sequence $\exp(\beta^n_{i_2|1})/\overline{\beta}_n$ is bounded, we can replace it by its subsequence which has a positive limit $p^2_{i_2}:=\lim_{n\to\infty}\exp(\beta^n_{i_2|1})/\overline{\beta}_n$. Note that at least one among the limits $p^2_{i_2}$ must be equal to one. Next, let us define
\begin{align*}
    (\Delta\boldsymbol{a}^n_{1})/\overline{M}_n\to \boldsymbol{q}_{2},& \quad (\Delta\boldsymbol{\eta}^n_{1i_21})/\overline{M}_n\to \boldsymbol{q}_{3i_2},\\
    (\Delta\tau^n_{1i_21})/\overline{M}_n\to q_{4i_2},& \quad (\Delta\nu^n_{1i_21})/2\overline{M}_n\to q_{5i_2}.
\end{align*}
Note that at least one among $q_{2},q_{3i_2},q_{4i_2},q_{5i_2}$ must be equal to either 1 or $-1$. 

\vspace{0.5 em}
\noindent
By dividing both the numerator and the denominator of the term in equation~\eqref{eq:ratio_before_limit} by $\overline{\beta}_n\overline{M}^{|\brho_1|+\rho_2}_n$, we obtain the system of polynomial equations:
\begin{align*}
    \sum_{i_2\in\mathcal{V}_{1|1}}\sum_{(\balpha_2,\balpha_3,\alpha_4,\alpha_5)\in\mathcal{I}^{SL}_{\brho_1,\rho_2}}\frac{1}{\balpha_2!\balpha_3!\alpha_4!\alpha_5!}\cdot p^2_{i_2}\boldsymbol{q}_{2}^{\balpha_2}\boldsymbol{q}_{3i_2}^{\balpha_3}q_{4i_2}^{\alpha_4}q_{5i_2}^{\alpha_5}=0, \quad 1\leq|\brho_1|+\rho_2\leq\brone.   
\end{align*}
According to the definition of the term $\brone$, the above system does not have any non-trivial solutions, which is a contradiction. Consequently, at least one among the coefficients in the representation of $A_n/\mathcal{L}_{2n}$, $B_n/\mathcal{L}_{2n}$ and $C_n/\mathcal{L}_{2n}$ must not converge to zero as $n\to\infty$.

\vspace{0.5 em}
\noindent
\textbf{Step 3 - Application of the Fatou's lemma.} In this stage, we show that all the coefficients in the formulations of $A_n/\mathcal{L}_{2n}$, $B_n/\mathcal{L}_{2n}$ and $C_n/\mathcal{L}_{2n}$ go to zero as $n\to\infty$. Denote by $m_n$ the maximum of the absolute values of those coefficients, the result from Step 2 induces that $1/m_n\not\to\infty$. By employing the Fatou's lemma, we have
\begin{align*}
    0=\lim_{n\to\infty}\dfrac{\bbE_{\bbX}[V(p^{SL}_{G_n}(\cdot|\bbX),p^{SL}_{G_*}(\cdot|\bbX))]}{m_n\mathcal{L}_{2n}}\geq\int\liminf_{n\to\infty}\dfrac{|p^{SL}_{G_n}(y|\bx)-p^{SL}_{G_*}(y|\bx)|}{2m_n\mathcal{L}_{2n}}\dint(\bx,y).
\end{align*}
Thus, we deduce that 
\begin{align*}
    \dfrac{|p^{SL}_{G_n}(y|\bx)-p^{SL}_{G_*}(y|\bx)|}{2m_n\mathcal{L}_{2n}}\to0,
\end{align*}
which results in $Q_n/[m_n\mathcal{L}_{2n}]\to0$ as $n\to\infty$ for almost surely $(\bx,y)$. 
Next, we denote
\begin{align*}
    &\frac{\exp(\bjn)S_{n,j_2|j_1,\balpha_1,\brho_1,\rho_2}}{m_n\mathcal{L}_{2n}}\to \phi_{j_2|j_1,\balpha_1,\brho_1,\rho_2}, \quad &\frac{\exp(\bjn)T_{n,j_2|j_1,\bpsi}}{m_n\mathcal{L}_{2n}}\to\varphi_{j_2|j_1,\bpsi},\\
    &\frac{\exp(\bjn)(\dajn)^{\bgamma}}{m_n\mathcal{L}_{2n}}\to\lambda_{j_1,\bgamma}, \quad &\frac{\exp(\bjn)-\exp(\bj)}{m_n\mathcal{L}_{2n}}\to\chi_{j_1}
\end{align*}
with a note that at least one among them is non-zero. Then, the decomposition of $Q_n$ in equation~\eqref{eq:decompose_Qn_sl} indicates that
\begin{align*}
    \lim_{n\to\infty}\frac{Q_n}{m_n\mathcal{L}_{2n}}=\lim_{n\to\infty}\frac{A_{n}}{m_n\mathcal{L}_{2n}}-\lim_{n\to\infty}\frac{B_n}{m_n\mathcal{L}_{2n}}+\lim_{n\to\infty}\frac{C_n}{m_n\mathcal{L}_{2n}},
\end{align*}
in which 
\begin{align*}
    &\lim_{n\to\infty}\frac{A_{n}}{m_n\mathcal{L}_{2n}}=\sum_{j_1=1}^{k^*_1}\sum_{j_2=1}^{k^*_2}\Bigg[\sum_{|\balpha_1|=1}^{\brj}\sum_{|\brho_1|+\brho_2=0\vee 1-|\balpha_1|}^{2(\brj-|\balpha_1|)}S_{n,j_2|j_1,\balpha_1,\brho_1,\rho_2}\cdot\bx^{\brho_1}\frac{\partial^{|\balpha_1|}F}{\partial\boldsymbol{\omega}^{\balpha_1}}(\bx;\oj)\\
    &\times\exp((\aj)^{\top}\bx)\frac{\partial^{\rho_2}\pi}{\partial\xi^{\rho_2}}(y|(\ej)^{\top}\bx+\tj,\vj)-\sum_{|\bpsi|=0}^{2}\varphi_{j_2|j_1,\bpsi}\cdot\frac{\partial^{|\bpsi|}F}{\partial\boldsymbol{\omega}^{\bpsi}}(\bx;\oj)\nonumber\\
    &\hspace{3cm}\times\exp((\aj)^{\top}\bx)p^{SL,*}_{j_1}(y|\bx)\Bigg]\frac{1}{\sum_{j'_2=1}^{k^*_2}\exp(-\|\boldsymbol{\omega}^*_{j'_2|j_1}-\bx\|+\beta^*_{j'_2|j_1})},\\
    &\lim_{n\to\infty}\frac{B_{n}}{m_n\mathcal{L}_{2n}}=\sum_{j_1=1}^{k^*_1}\sum_{|\gamma|=1}\lambda_{j_1,\bgamma}\cdot \bx^{\bgamma}\exp((\aj)^{\top}\bx)p^{SL}_{G_*}(y|\bx),\\
    &\lim_{n\to\infty}\frac{C_{n}(\bx)}{m_n\mathcal{L}_{2n}}=\sum_{j_1=1}^{k^*_1}\chi_{j_1}\exp((\aj)^{\top}\bx)\left[p^{SL,*}_{j_1}(y|\bx)-p^{SL}_{G_*}(y|\bx)\right].
\end{align*}
Since the set
\begin{align*}
    &\Bigg\{\dfrac{\bx^{\brho_1}\frac{\partial^{|\balpha_1|}F}{\partial\boldsymbol{\omega}^{\balpha_1}}(\bx;\oj)\exp((\aj)^{\top}\bx)\frac{\partial^{\rho_2}\pi}{\partial\xi^{\rho_2}}(y|(\ej)^{\top}\bx+\tj,\vj)}{\sum_{j'_2=1}^{k^*_2}\exp(-\|\boldsymbol{\omega}^*_{j'_2|j_1}-\bx\|+\beta^*_{j'_2|j_1})}:j_1\in[k^*_1], j_2\in[k^*_2],\\
    &\hspace{6cm}0\leq|\balpha_1|\leq \brj,0\leq|\brho_1|+\rho_2\leq2(\brj-|\balpha_1|)\Bigg\}\\
    \cup&~\Bigg\{\dfrac{\frac{\partial^{|\bpsi|}F}{\partial\boldsymbol{\omega}^{\bpsi}}(\bx;\oj)\exp((\aj)^{\top}\bx)p^{SL,*}_{j_1}(y|\bx)}{\sum_{j'_2=1}^{k^*_2}\exp(-\|\boldsymbol{\omega}^*_{j'_2|j_1}-\bx\|+\beta^*_{j'_2|j_1})}:j_1\in[k^*_1], j_2\in[k^*_2],0\leq|\bpsi|\leq2\Bigg\}\\
    \cup&~\Big\{\bx^{\bgamma}\exp((\aj)^{\top}\bx)p^{SL}_{G_*}(y|\bx), \ \exp((\aj)^{\top}\bx)p^{SL,*}_{j_1}(y|\bx), \ \exp((\aj)^{\top}\bx)p^{SL}_{G_*}(y|\bx)\\
    &\hspace{8cm}:j_1\in[k^*_1],0\leq|\bgamma|\leq2\Big\}
\end{align*}
is linearly independent, we obtain that $\phi_{j_2|j_1,\balpha_1,\brho_1,\rho_2}=\varphi_{j_2|j_1,\bpsi}=\lambda_{j_1,\bgamma}=\chi_{j_1}=0$ for all $j_1\in[k^*_1]$, $j_2\in[k^*_2]$, $0\leq|\balpha_1|\leq \brj$, $0\leq|\brho_1|+\rho_2\leq2(\brj-|\balpha_1|)$, $0\leq|\bpsi|\leq 2$ and $0\leq|\bgamma|\leq 1$, which is a contradiction. As a consequence, we obtain the inequality in equation~\eqref{eq:general_local_inequality_over_ss}. Hence, the proof is completed.

\subsection{Proof of Theorem~\ref{theorem:param_rates_LL}:When $type=LL$}
\label{appendix:laplace_laplace}
When $type=LL$, the corresponding Voronoi loss function is $\mathcal{L}_{(2,r^{LL},\frac{1}{2}r^{LL})}(G_n,G_*)=\mathcal{L}_{3n}$ where we define
\begin{align}
    &\mathcal{L}_{3n}:=\sum_{j_1=1}^{k^*_1}\Big|\exp(b^n_{j_1})-\exp(b^*_{j_1})\Big|+\sum_{j_1=1}^{k^*_1}\exp(b^n_{j_1})\|\dajn\|+\sum_{j_1=1}^{k^*_1}\exp(b^n_{j_1})\nonumber\\
    &\times\Bigg[\sum_{j_2:|\mathcal{V}_{j_2|j_1}|=1}\sum_{i_2\in\mathcal{V}_{j_2|j_1}}\exp(\beta^n_{i_2|j_1})\Big(\|\doijn\|+\|\deijn\|+|\dtijn|+|\dvijn|\Big)\nonumber\\
    &+\sum_{j_2:|\mathcal{V}_{j_2|j_1}|>1}\sum_{i_2\in\mathcal{V}_{j_2|j_1}}\exp(\beta^n_{i_2|j_1})\Big(\|\doijn\|^{2}+\|\deijn\|^{2}+|\dtijn|^{\trj}\nonumber\\
     \label{eq:loss_l3_ll}
    &+|\dvijn|^{\frac{\trj}{2}}\Big)\Bigg]+\sum_{j_1=1}^{k^*_1}\exp(b^n_{j_1})\sum_{j_2=1}^{k^*_2}\Big|\sum_{i_2\in\mathcal{V}_{j_2|j_1}}\exp(\beta^n_{i_2|j_1})-\exp(\beta^*_{j_2|j_1})\Big|.
\end{align}


\vspace{0.5 em}
\noindent
\textbf{Step 1 - Taylor expansion:} In this step, we use the Taylor expansion to decompose the term
\begin{align*}
    Q_n:=\left[\sum_{j_1=1}^{k^*_1}\exp(-\|\aj-\bx\|+\bj)\right][p^{LL}_{G_n}(y|\bx)-p^{LL}_{G_*}(y|\bx)].
\end{align*}
Prior to that, let us denote
\begin{align*}
    p^{LL,n}_{j_1}(y|\bx)&:=\sum_{j_2=1}^{k^*_2}\sum_{i_2\in\mathcal{V}_{j_2|j_1}}\softmax(-\|\oin-\bx\|+\beta^n_{i_2|j_1}) \pi(y|(\ein)^{\top}\bx+\tin,\vin),\\
    p^{LL,*}_{j_1}(y|\bx)&:=\sum_{j_2=1}^{k^*_2}\softmax(-\|\oj-\bx\|+\bej) \pi(y|(\ej)^{\top}\bx+\tj,\vj).
\end{align*}
Then, the quantity $Q_n$ is divided into three terms as
\begin{align}
    Q_n&=\sum_{j_1=1}^{k^*_1}\exp(\bjn)\left[\exp(-\|\ajn-\bx\|)p^{LL,n}_{j_1}(y|\bx)-\exp(-\|\aj-\bx\|)p^{LL,*}_{j_1}(y|\bx)\right]\nonumber\\
    &-\sum_{j_1=1}^{k^*_1}\exp(\bjn)\left[\exp(-\|\ajn-\bx\|)-\exp(-\|\aj-\bx\|)\right]p^{LL}_{G_n}(y|\bx)\nonumber\\
    &+\sum_{j_1=1}^{k^*_1}\left(\exp(\bjn)-\exp(\bj)\right)\exp(-\|\aj-\bx\|)\left[p^{LL,n}_{j_1}(y|\bx)-p^{LL}_{G_n}(y|\bx)\right]\nonumber\\
     \label{eq:decompose_Qn_ll}
    :&=A_n-B_n+C_n.
\end{align}
\textbf{Step 1A - Decompose $A_n$:} We continue to decompose $A_n$:
\begin{align*}
    A_n&:=\sum_{j_1=1}^{k^*_1}\frac{\exp(\bjn)}{\sum_{j'_2=1}^{k^*_2}\exp(-\|\boldsymbol{\omega}^*_{j'_2|j_1}-\bx\|+\beta^*_{j'_2|j_1})}[A_{n,j_1,1}+A_{n,j_1,2}+A_{n,j_1,3}],
\end{align*}
in which
\begin{align*}
    A_{n,j_1,1}&:=\sum_{j_2=1}^{k^*_2}\sum_{i_2\in\mathcal{V}_{j_2|j_1}}\exp(\beta^n_{i_2|j_1})\Big[\exp(-\|\oin-\bx\|)\exp(-\|\ajn-\bx\|)\pi(y|(\ein){\top}\bx+\tin,\vin)\nonumber\\
    &\hspace{2cm}-\exp(-\|\oj-\bx\|)\exp(-\|\aj-\bx\|)\pi(y|(\ej)^{\top}\bx+\tj,\vj)\Big],\\
    A_{n,j_1,2}&:=\sum_{j_2=1}^{k^*_2}\sum_{i_2\in\mathcal{V}_{j_2|j_1}}\exp(\beta^n_{i_2|j_1})\Big[\exp(-\|\oin-\bx\|)-\exp(-\|\oj-\bx\|)\Big]\nonumber\\
    &\hspace{7cm}\times\exp(-\|\ajn-\bx\|)p^{LL,n}_{j_1}(y|\bx),\nonumber\\
    A_{n,j_1,3}&:=\sum_{j_2=1}^{k^*_2}\Big(\sum_{i_2\in\mathcal{V}_{j_2|j_1}}\exp(\bein)-\exp(\bej)\Big)\exp(-\|\oj-\bx\|)\\
    &\hspace{1cm}\times[\exp(-\|\aj-\bx\|)\pi(y|(\ej)^{\top}\bx+\tj,\vj)-\exp(-\|\ajn-\bx\|)p^{LL,n}_{j_1}(y|\bx)].
\end{align*}
Firstly, we separate the term $A_{n,j_1,1}$ into two parts based on the cardinality of the Voronoi cells $\mathcal{V}_{j_2|j_1}$ as
\begin{align*}
    A_{n,j_1,1}&=\sum_{j_2:|\mathcal{V}_{j_2|j_1}|=1}\sum_{i_2\in\mathcal{V}_{j_2|j_1}}\exp(\bein)\Big[\exp(-\|\oin-\bx\|)\exp(-\|\ajn-\bx\|)\pi(y|(\ein)^{\top}\bx+\tin,\vin)\\
    &\hspace{2cm}-\exp(-\|\oj-\bx\|)\exp(-\|\aj-\bx\|)\pi(y|(\ej)^{\top}\bx+\tj,\vj)\Big],\\
    &+\sum_{j_2:|\mathcal{V}_{j_2|j_1}|>1}\sum_{i_2\in\mathcal{V}_{j_2|j_1}}\exp(\bein)\Big[\exp(-\|\oin-\bx\|)\exp(-\|\ajn-\bx\|)\pi(y|(\ein)^{\top}\bx+\tin,\vin)\\
    &\hspace{2cm}-\exp(-\|\oj-\bx\|)\exp(-\|\aj-\bx\|)\pi(y|(\ej)^{\top}\bx+\tj,\vj)\Big]\\
    :&=A_{n,j_1,1,1}+A_{n,j_1,1,2}.
\end{align*}
By denoting $F(\bx;\boldsymbol{\omega}):=\exp(-\|\boldsymbol{\omega}-\bx\|)$ and employing the first-order Taylor expansion, we can represent $A_{n,j_1,1,1}$ as
\begin{align*}
    &A_{n,j_1,1,1}=\sum_{j_2:|\mathcal{V}_{j_2|j_1}|=1}\sum_{i_2\in\mathcal{V}_{j_2|j_1}}\sum_{|\balpha|=1}\frac{\exp(\bein)}{2^{\alpha_5!}\balpha!}(\doijn)^{\balpha_1}(\dajn)^{\balpha_2}(\deijn)^{\balpha_3}(\dtijn)^{\alpha_4}\\
    &\times(\dvijn)^{\alpha_5} \bx^{\balpha_3}\frac{\partial^{|\balpha_1|}F}{\partial\boldsymbol{\omega}^{\balpha_1}}(\bx;\oj)\frac{\partial^{|\balpha_2|}F}{\partial\boldsymbol{a}^{\balpha_2}}(\bx;\aj)\frac{\partial^{|\balpha_3|+\alpha_4+2\alpha_5}\pi}{\partial\xi^{|\balpha_3|+\alpha_4+2\alpha_5}}(y|(\ej)^{\top}\bx+\tj,\vj)
    +R_{n,1,1}(\bx)\\
    &=\sum_{j_2:|\mathcal{V}_{j_2|j_1}|=1}\sum_{|\balpha_1|+|\balpha_2|+|\balpha_3|=0}^{1}\sum_{\rho=0\vee1-|\balpha_1|-|\balpha_2|-|\balpha_3|}^{2(1-|\balpha_1|-|\balpha_2|-|\balpha_3|)}S_{n,j_2|j_1,\balpha_1,\balpha_2,\balpha_3,\rho}\cdot\bx^{\balpha_3} \frac{\partial^{|\balpha_1|}F}{\partial\boldsymbol{\omega}^{\balpha_1}}(\bx;\oj)\\
    &\hspace{4cm}\times\frac{\partial^{|\balpha_2|}F}{\partial\boldsymbol{a}^{\balpha_2}}(\bx;\aj)\frac{\partial^{|\balpha_3|+\rho}\pi}{\partial\xi^{|\balpha_3|+\rho}}(y|(\ej)^{\top}\bx+\tj,\vj)+R_{n,1,1}(\bx),
\end{align*}
where $R_{n,1,1}(\bx,y)$ is a Taylor remainder such that $R_{n,1,1}(\bx,y)/\mathcal{L}_{3n}\to0$ as $n\to\infty$, and
\begin{align*}
    S_{n,j_2|j_1,\balpha_1,\balpha_2,\balpha_3,\rho}&:=\sum_{i_2\in\mathcal{V}_{j_2|j_1}}\sum_{\alpha_4+2\alpha_5=\rho}\frac{\exp(\bein)}{2^{\alpha_5}\balpha!}(\doijn)^{\balpha_1}(\dajn)^{\balpha_2}(\deijn)^{\balpha_3}\\
    &\hspace{6cm}\times(\dtijn)^{\alpha_4}(\dvijn)^{\alpha_5},
\end{align*}
for any $(\balpha_1,\balpha_2,\balpha_3,\rho)\neq(\zerod,\zerod,\zerod,0)$, $j_1\in[k^*_1]$ and $j_2\in[k^*_2]$.

\vspace{0.5 em}
\noindent
For each $(j_1,j_2)\in[k^*_1]\times[k^*_2]$, by invoking the Taylor expansion of order $r^{LL}(|\mathcal{V}_{j_2|j_1}|):=\trj$, the term $A_{n,j_1,1,2}$ can be represented as
\begin{align*}
    &A_{n,j_1,1,2}=\sum_{j_2:|\mathcal{V}_{j_2|j_1}|>1}\sum_{|\balpha_1|+|\balpha_2|+|\balpha_3|=0}^{\trj}\sum_{\rho=0\vee1-|\balpha_1|-|\balpha_2|-|\balpha_3|}^{2(\trj-|\balpha_1|-|\balpha_2|-|\balpha_3|)}S_{n,j_2|j_1,\balpha_1,\balpha_2,\balpha_3,\rho}\cdot\bx^{\balpha_3} \\
    &\hspace{1cm}\times\frac{\partial^{|\balpha_1|}F}{\partial\boldsymbol{\omega}^{\balpha_1}}(\bx;\oj)\frac{\partial^{|\balpha_2|}F}{\partial\boldsymbol{a}^{\balpha_2}}(\bx;\aj)\frac{\partial^{|\balpha_3|+\rho}\pi}{\partial\xi^{|\balpha_3|+\rho}}(y|(\ej)^{\top}\bx+\tj,\vj)+R_{n,1,2}(\bx,y),
\end{align*}
where $R_{n,1,2}(\bx,y)$ is a Taylor remainder such that $R_{n,1,2}(\bx,y)/\mathcal{L}_{3n}\to0$ as $n\to\infty$.

\vspace{0.5 em}
\noindent
Secondly, we rewrite the term $A_{n,j_1,2}$ as follows:
\begin{align*}
   &\sum_{j_2:|\mathcal{V}_{j_2|j_1}|=1}\sum_{i_2\in\mathcal{V}_{j_2|j_1}}\exp(\bein)\Big[\exp(-\|\oin-\bx\|)\nonumber-\exp(-\|\oj-\bx\|)\Big]\exp(-\|\ajn-\bx\|)p^{LL,n}_{j_1}(y|\bx)\\
    &+\sum_{j_2:|\mathcal{V}_{j_2|j_1}|>1}\sum_{i_2\in\mathcal{V}_{j_2|j_1}}\exp(\bein)\Big[\exp(-\|\oin-\bx\|)\nonumber-\exp(-\|\oj-\bx\|)\Big]\exp(-\|\ajn-\bx\|)p^{LL,n}_{j_1}(y|\bx)\\
    :&=A_{n,j_1,2,1}+A_{n,j_1,2,2}.
\end{align*}
According to the first-order Taylor expansion, we have
\begin{align*}
    &A_{n,j_1,2,1}=\sum_{j_2:|\mathcal{V}_{j_2|j_1}|=1}\sum_{i_2\in\mathcal{V}_{j_2|j_1}}\sum_{|\bpsi|=1}\frac{\exp(\bein)}{\bpsi!}(\doijn)^{\bpsi}\\
    &\hspace{5cm}\times\frac{\partial^{|\bpsi|}F}{\partial\boldsymbol{\omega}^{\bpsi}}(\bx;\oj)\exp(-\|\ajn-\bx\|)p^{LL,n}_{j_1}(y|\bx)+R_{n,2,1}(\bx,y),\\
    &=\sum_{j_2:|\mathcal{V}_{j_2|j_1}|=1}\sum_{|\bpsi|=1}T_{n,j_2|j_1,\bpsi}\cdot\frac{\partial^{|\bpsi|}F}{\partial\boldsymbol{\omega}^{\bpsi}}(\bx;\oj)\exp(-\|\ajn-\bx\|)p^{LL,n}_{j_1}(y|\bx)+R_{n,2,1}(\bx,y),
\end{align*}
where $R_{n,2,1}(\bx,y)$ is a Taylor remainder such that $R_{n,2,1}(\bx,y)/\mathcal{L}_{3n}\to0$ as $n\to\infty$, and
\begin{align*}
    T_{n,j_2|j_1,\bpsi}:=\sum_{i_2\in\mathcal{V}_{j_2|j_1}}\frac{\exp(\bein)}{\bpsi!}(\doijn)^{\bpsi},
\end{align*}
for any $j_2\in[k^*_2]$ and $\bpsi\neq\zerod$.

\vspace{0.5 em}
\noindent
Meanwhile, we apply the second-order Taylor expansion to $A_{n,j_1,2,2}$:
\begin{align*}
    A_{n,j_1,2,2}&=\sum_{j_2:|\mathcal{V}_{j_2|j_1}|>1}\sum_{|\bpsi|=1}^{2}T_{n,j_2|j_1,\bpsi}\cdot\frac{\partial^{|\bpsi|}F}{\partial\boldsymbol{\omega}^{\bpsi}}(\bx;\oj)\exp(-\|\ajn-\bx\|)p^{LL,n}_{j_1}(y|\bx)+R_{n,2,2}(\bx,y),
\end{align*}
where $R_{n,2,2}(\bx,y)$ is a Taylor remainder such that $R_{n,2,2}(\bx,y)/\mathcal{L}_{3n}\to0$ as $n\to\infty$.

\vspace{0.5 em}
\noindent
Combine the above results together, we can illustrate the term $A_{n}$ as
\begin{align}
    &A_{n}=\sum_{j_1=1}^{k^*_1}\sum_{j_2=1}^{k^*_2}\frac{\exp(\bjn)}{\sum_{j'_2=1}^{k^*_2}\exp(-\|\boldsymbol{\omega}^*_{j'_2|j_1}-\bx\|+\beta^*_{j'_2|j_1})}\Bigg[\sum_{|\balpha_1|+|\balpha_2|+|\balpha_3|=0}^{\trj}\sum_{\rho=0\vee1-|\balpha_1|-|\balpha_2|-|\balpha_3|}^{2(\trj-|\balpha_1|-|\balpha_2|-|\balpha_3|)}S_{n,j_2|j_1,\balpha_1,\balpha_2,\balpha_3,\rho}\nonumber\\
    &\times\bx^{\balpha_3}\frac{\partial^{|\balpha_1|}F}{\partial\boldsymbol{\omega}^{\balpha_1}}(\bx;\oj)\frac{\partial^{|\balpha_2|}F}{\partial\boldsymbol{a}^{\balpha_2}}(\bx;\aj)\frac{\partial^{|\balpha_3|+\rho}\pi}{\partial\xi^{|\balpha_3|+\rho}}(y|(\ej)^{\top}\bx+\tj,\vj)+R_{n,1,1}(\bx,y)+R_{n,1,2}(\bx,y)\nonumber\\
    \label{eq:decompose_An1_ll}
    &-\sum_{|\bpsi|=0}^{2}T_{n,j_2|j_1,\bpsi}\cdot\frac{\partial^{|\bpsi|}F}{\partial\boldsymbol{\omega}^{\bpsi}}(\bx;\oj)\exp(-\|\ajn-\bx\|)p^{LL,n}_{j_1}(y|\bx)-R_{n,2,1}(\bx,y)-R_{n,2,2}(\bx,y)\Bigg],
\end{align}
where $S_{n,j_2|j_1,\balpha_1,\balpha_2,\balpha_3,\rho}=T_{n,j_2|j_1,\bpsi}=\sum_{i_2\in\mathcal{V}_{j_2|j_1}}\exp(\bein)-\exp(\bej)$ for any $j_1\in[k^*_1]$, $j_2\in[k^*_2]$, $(\balpha_1,\balpha_2,\balpha_3,\rho)=(\zerod,\zerod,\zerod,0)$ and $\bpsi=\zerod$.

\vspace{0.5 em}
\noindent
\textbf{Step 1B - Decompose $B_{n}$:} By invoking the first-order Taylor expansion, we decompose the term $B_{n}$ defined in equation~\eqref{eq:decompose_Qn_ll} as
\begin{align}
     \label{eq:decompose_Bn_ll}
    B_{n}&=\sum_{j_1=1}^{k^*_1}\exp(\bjn)\sum_{|\bgamma|=1}(\dajn)^{\bgamma}\cdot \frac{\partial^{|\bgamma|}F}{\partial\boldsymbol{a}^{\bgamma}}(\bx;\aj)p^{LL}_{G_n}(y|\bx) +R_{n,3}(\bx,y)
\end{align}
where $R_{n,3}(\bx,y)$ is a Taylor remainder such that $R_{n,3}(\bx,y)/\mathcal{L}_{3n}\to0$ as $n\to\infty$.

\vspace{0.5 em}
\noindent
Putting the decomposition in equations~\eqref{eq:decompose_Qn_ll}, \eqref{eq:decompose_An1_ll} and \eqref{eq:decompose_Bn_ll} together, we realize that $A_n$, $B_n$ and $C_n$ can be treated as a linear combination of elements from the following set union:
\begin{align*}
   &\Bigg\{\frac{\bx^{\balpha_3}\frac{\partial^{|\balpha_1|}F}{\partial\boldsymbol{\omega}^{\balpha_1}}(\bx;\oj)\frac{\partial^{|\balpha_2|}F}{\partial\boldsymbol{a}^{\balpha_2}}(\bx;\aj)\frac{\partial^{|\balpha_3|+\rho}\pi}{\partial\xi^{|\balpha_3|+\rho}}(y|(\ej)^{\top}\bx+\tj,\vj)}{\sum_{j'_2=1}^{k^*_2}\exp(-\|\boldsymbol{\omega}^*_{j'_2|j_1}-\bx\|+\beta^*_{j'_2|j_1})}:j_1\in[k^*_1], \  j_2\in[k^*_2],\\
   &\hspace{2cm}0\leq|\balpha_1|+|\balpha_2|+|\balpha_3|\leq2\trj, 0\leq\rho\leq2(\trj-|\balpha_1|-|\balpha_2|-|\balpha_3|)\Bigg\}\\
   \cup~&\Bigg\{\frac{\frac{\partial^{|\bpsi|}F}{\partial\boldsymbol{\omega}^{\bpsi}}(\bx;\oj)\exp(-\|\ajn-\bx\|)p^{LL,n}_{j_1}(y|\bx)}{\sum_{j'_2=1}^{k^*_2}\exp(-\|\boldsymbol{\omega}^*_{j'_2|j_1}-\bx\|+\beta^*_{j'_2|j_1})}:j_1\in[k^*_1], \ j_2\in[k^*_2], \ 0\leq|\bpsi|\leq2\Bigg\}\\
   \cup~&\left\{\frac{\partial^{|\bgamma|}F}{\partial\boldsymbol{a}^{\bgamma}}(\bx;\aj)p^{LL,n}_{j_1}(y|\bx), \ \frac{\partial^{|\bgamma|}F}{\partial\boldsymbol{a}^{\bgamma}}(\bx;\aj)p^{LL}_{G_n}(y|\bx): j_1\in[k^*_1], \ 0\leq|\bgamma|\leq1\right\}.
\end{align*}

\vspace{0.5 em}
\noindent
\textbf{Step 2 - Non-vanishing coefficients:} In this step, we demonstrate that not all the coefficients in the representation of $A_n/\mathcal{L}_{3n}$, $B_n/\mathcal{L}_{3n}$ and $C_n/\mathcal{L}_{3n}$ converge to zero as $n\to\infty$. Assume by contrary that all of them go to zero. Then, we look into the coefficients associated with the term 
\begin{itemize}
    \item $\exp(-\|\aj-\bx\|)p^{LL,n}_{j_1}(y|\bx)$ in $C_n/\mathcal{L}_{3n}$, we have
    \begin{align}
        \label{eq:limit_bias_1_ll}
        \frac{1}{\mathcal{L}_{3n}}\cdot\sum_{j_1=1}^{k^*_1}\Big|\exp(\bjn)-\exp(\bj)\Big|\to0.
    \end{align}
    \item $\dfrac{F(\bx;\oj)F(\bx;\aj)\pi(y|(\ej)^{\top}\bx+\tj,\vj)}{\sum_{j'_2=1}^{k^*_2}\exp(-\|\boldsymbol{\omega}^*_{j'_2|j_1}-\bx\|+\beta^*_{j'_2|j_1})}$ in $A_{n}/\mathcal{L}_{3n}$, we get that
    \begin{align}
        \label{eq:limit_bias_2_ll}
        \frac{1}{\mathcal{L}_{3n}}\cdot\sum_{j_1=1}^{k^*_1}\exp(\bjn)\sum_{j_2=1}^{k^*_2}\Big|\sum_{i_2\in\mathcal{V}_{j_2|j_1}}\exp(\bein)-\exp(\bej)\Big|\to0.
    \end{align}
    \item $\dfrac{\frac{\partial^{|\balpha_1|}F}{\partial\boldsymbol{\omega}^{\balpha_1}}(\bx;\oj)F(\bx;\aj)\pi(y|(\ej)^{\top}\bx+\tj,\vj)}{\sum_{j'_2=1}^{k^*_2}\exp(-\|\boldsymbol{\omega}^*_{j'_2|j_1}-\bx\|+\beta^*_{j'_2|j_1})}$ in $A_{n}/\mathcal{L}_{3n}$ for $j_1\in[k^*_1],j_2\in[k^*_2]:|\mathcal{V}_{j_2|j_1}|=1$ and $\balpha_1=e_{d,u}$ where $e_{d,u}:=(0,\ldots,0,\underbrace{1}_{\textit{u-th}},0,\ldots,0)\in\mathbb{N}^{d}$, we receive that
    \begin{align*}
         \frac{1}{\mathcal{L}_{3n}}\cdot\sum_{j_1=1}^{k^*_1}\exp(\bjn)\sum_{j_2\in[k^*_2]:|\mathcal{V}_{j_2|j_1}|=1}\sum_{i_2\in\mathcal{V}_{j_2|j_1}}\exp(\bein)\|\oin-\oj\|_1\to0.
    \end{align*}
    Note that since the norm-1 is equivalent to the norm-2, then we can replace the norm-1 with the norm-2, that is,
    \begin{align}
        \label{eq:limit_exact_1_ll}
        \frac{1}{\mathcal{L}_{3n}}\cdot\sum_{j_1=1}^{k^*_1}\exp(\bjn)\sum_{j_2\in[k^*_2]:|\mathcal{V}_{j_2|j_1}|=1}\sum_{i_2\in\mathcal{V}_{j_2|j_1}}\exp(\bein)\|\oin-\oj\|\to0.
    \end{align}
    \item $\bx^{\balpha_3}\dfrac{F(\bx;\oj)F(\bx;\aj)\frac{\partial^{|\balpha_3|}\pi}{\partial\xi^{|\balpha_3|}}(y|(\ej)^{\top}\bx+\tj,\vj)}{\sum_{j'_2=1}^{k^*_2}\exp(-\|\boldsymbol{\omega}^*_{j'_2|j_1}-\bx\|+\beta^*_{j'_2|j_1})}$ in $A_{n}/\mathcal{L}_{3n}$ for $j_1\in[k^*_1],j_2\in[k^*_2]:|\mathcal{V}_{j_2|j_1}|=1$ and $\balpha_3=e_{d,u}$, we have that
    \begin{align}
        \label{eq:limit_exact_3_ll}
         \frac{1}{\mathcal{L}_{3n}}\cdot\sum_{j_1=1}^{k^*_1}\exp(\bjn)\sum_{j_2\in[k^*_2]:|\mathcal{V}_{j_2|j_1}|=1}\sum_{i_2\in\mathcal{V}_{j_2|j_1}}\exp(\bejn)\|\ein-\ej\|\to0.
    \end{align}
    \item $\frac{\partial^{|\bgamma|}F}{\partial\boldsymbol{a}^{\bgamma}}(\bx;\aj)p^{LL}_{G_n}(y|\bx)$ in $B_n/\mathcal{L}_{3n}$ for $j_1\in[k^*_1]$ and $\bgamma=e_{d,u}$, we obtain
    \begin{align}
        \label{eq:limit_exact_4_ll}
        \frac{1}{\mathcal{L}_{3n}}\cdot\sum_{j_1=1}^{k^*_1}\exp(\bjn)\|\ajn-\aj\|\to0.
    \end{align}
     \item $\dfrac{\frac{\partial^{|\balpha_1|}F}{\partial\boldsymbol{\omega}^{\balpha_1}}(\bx;\oj)F(\bx;\aj)\pi(y|(\ej)^{\top}\bx+\tj,\vj)}{\sum_{j'_2=1}^{k^*_2}\exp(-\|\boldsymbol{\omega}^*_{j'_2|j_1}-\bx\|+\beta^*_{j'_2|j_1})}$ in $A_{n}/\mathcal{L}_{3n}$ for $j_1\in[k^*_1],j_2\in[k^*_2]:|\mathcal{V}_{j_2|j_1}|>1$ and $\balpha_1=2e_{d,u}$, we receive that
    \begin{align}
         \frac{1}{\mathcal{L}_{3n}}\cdot\sum_{j_1=1}^{k^*_1}\exp(\bjn)\sum_{j_2\in[k^*_2]:|\mathcal{V}_{j_2|j_1}|>1}\sum_{i_2\in\mathcal{V}_{j_2|j_1}}\exp(\bein)\|\oin-\oj\|^2\to0.
    \end{align}
    \item $\dfrac{\bx^{\balpha_3}F(\bx;\oj)F(\bx;\aj)\frac{\partial^{|\balpha_3|}\pi}{\partial\xi^{|\balpha_3|}}(y|(\ej)^{\top}\bx+\tj,\vj)}{\sum_{j'_2=1}^{k^*_2}\exp(-\|\boldsymbol{\omega}^*_{j'_2|j_1}-\bx\|+\beta^*_{j'_2|j_1})}$ in $A_{n}/\mathcal{L}_{3n}$ for $j_1\in[k^*_1],j_2\in[k^*_2]:|\mathcal{V}_{j_2|j_1}|>1$ and $\balpha_3=2e_{d,u}$, we have that
    \begin{align}
        \label{eq:limit_over_1_ll}
         \frac{1}{\mathcal{L}_{3n}}\cdot\sum_{j_1=1}^{k^*_1}\exp(\bjn)\sum_{j_2\in[k^*_2]:|\mathcal{V}_{j_2|j_1}|>1}\sum_{i_2\in\mathcal{V}_{j_2|j_1}}\exp(\bein)\|\ein-\ej\|^2\to0.
    \end{align}
\end{itemize}
Combine the above limits and the formulation of the loss $\mathcal{L}_{3n}$ in equation~\eqref{eq:loss_l3_ll}, we deduce that
\begin{align*}
    \frac{1}{\mathcal{L}_{3n}}\cdot\sum_{j_1=1}^{k^*_1}\exp(\bjn)\sum_{j_2:|\mathcal{V}_{j_2|j_1}|>1}\sum_{i_2\in\mathcal{V}_{j_2|j_1}}\exp(\beta^n_{i_2|j_1})\Big(|\dtijn|^{\trj}+|\dvijn|^{\frac{\trj}{2}}\Big)\not\to0.
\end{align*}
This indicates that there exist indices $j^*_1\in[k^*_1]$ and $j^*_2\in[k^*_2]:|\mathcal{V}_{j^*_2|j^*_1}|>1$ such that 
\begin{align}
    \label{eq:zero_limit_ll}
    \frac{1}{\mathcal{L}_{3n}}\cdot\sum_{i_2\in\mathcal{V}_{j^*_2|j^*_1}}\exp(\beta^n_{i_2|j^*_1})\Big(|\Delta\tau^n_{j^*_1i_2j^*_2}|^{r^{LL}_{j^*_2|j^*_1}}+|\Delta\nu^n_{j^*_1i_2j^*_2}|^{\frac{r^{LL}_{j^*_2|j^*_1}}{2}}\Big)\not\to0.
\end{align}
WLOG, we may assume that $j^*_1=j^*_2=1$. Then, considering the coefficients of the term $F(\bx;\oj)F(\bx;\aj)\frac{\partial^\rho\pi}{\partial\xi^{\rho}}(y|(\ej)^{\top}\bx+\tj,\vj)$ in $A_n/\mathcal{L}_{3n}$ where $j_1=j_2=1$, we get
\begin{align*}
    \exp(b^n_{1})S_{n,1|1,\zerod,\zerod,\zerod,\rho}/\mathcal{L}_{3n}\to0,
\end{align*}
or equivalently,
\begin{align}
    \label{eq:non_zero_denom_ll}
    &\frac{1}{\mathcal{L}_{3n}}\cdot\sum_{i_2\in\mathcal{V}_{1|1}}\sum_{\alpha_4+2\alpha_5=\rho}\frac{\exp(\beta^n_{i_2|1})}{2^{\alpha_5}\alpha_4!\alpha_5!}\cdot(\Delta\tau^n_{1i_21})^{\alpha_4}(\Delta\nu^n_{1i_21})^{\alpha_5}\to0.
\end{align}
Next, we divide the left hand side of equation~\eqref{eq:zero_limit_ll} by that of equation~\eqref{eq:non_zero_denom_ll}, and get that
\begin{align}
    \label{eq:ratio_before_limit_ll}
    \dfrac{\sum_{i_2\in\mathcal{V}_{1|1}}\sum_{\alpha_4+2\alpha_5=\rho}\frac{\exp(\beta^n_{i_2|1})}{2^{\alpha_5}\alpha_4!\alpha_5!}\cdot(\Delta\tau^n_{1i_21})^{\alpha_4}(\Delta\nu^n_{1i_21})^{\alpha_5}}{\sum_{i_2\in\mathcal{V}_{1|1}}\exp(\beta^n_{i_2|1})\Big(|\Delta\tau^n_{1i_21}|^{\trone}+|\Delta\nu^n_{1i_21}|^{\frac{\trone}{2}}\Big)}\to0.
\end{align}
Let us define $\overline{M}_n:=\max\{\|\Delta\tau^n_{1i_21}\|, \|\Delta\nu^n_{1i_21}\|^{1/2}:i_2\in\mathcal{V}_{1|1}\}$, and $\overline{\beta}_n:=\max_{i_2\in\mathcal{V}_{1|1}}\exp(\beta^n_{i_2|1})$. Since the sequence $\exp(\beta^n_{i_2|1})/\overline{\beta}_n$ is bounded, we can replace it by its subsequence which has a positive limit $p^2_{i_2}:=\lim_{n\to\infty}\exp(\beta^n_{i_2|1})/\overline{\beta}_n$. Note that at least one among the limits $p^2_{i_2}$ must be equal to one. Next, let us define
\begin{align*}
    (\Delta\tau^n_{1i_21})/\overline{M}_n\to q_{4i_2},& \quad (\Delta\nu^n_{1i_21})/2\overline{M}_n\to q_{5i_2}.
\end{align*}
Note that at least one among $q_{4i_2},q_{5i_2}$ must be equal to either 1 or $-1$. By dividing both the numerator and the denominator of the term in equation~\eqref{eq:ratio_before_limit} by $\overline{\beta}_n\overline{M}^{\rho}_n$, we obtain the system of polynomial equations:
\begin{align*}
    \sum_{i_2\in\mathcal{V}_{1|1}}\sum_{\alpha_4+2\alpha_5=\rho}\frac{1}{\alpha_4!\alpha_5!}\cdot p^2_{i_2}q_{4i_2}^{\alpha_4}q_{5i_2}^{\alpha_5}=0, \quad 1\leq\rho\leq\trone.   
\end{align*}
According to the definition of the term $\trone$, the above system does not have any non-trivial solutions, which is a contradiction. Consequently, at least one among the coefficients in the representation of $A_n/\mathcal{L}_{3n}$, $B_n/\mathcal{L}_{3n}$ and $C_n/\mathcal{L}_{3n}$ must not approach zero as $n\to\infty$.

\vspace{0.5 em}
\noindent
\textbf{Step 3 - Application of the Fatou's lemma.} In this stage, we show that all the coefficients in the formulations of $A_n/\mathcal{L}_{3n}$, $B_n/\mathcal{L}_{3n}$ and $C_n/\mathcal{L}_{3n}$ go to zero as $n\to\infty$. Denote by $m_n$ the maximum of the absolute values of those coefficients, the result from Step 2 induces that $1/m_n\not\to\infty$. 

\vspace{0.5 em}
\noindent
By employing the Fatou's lemma, we have
\begin{align*}
    0=\lim_{n\to\infty}\dfrac{\bbE_{\bbX}[V(p^{LL}_{G_n}(\cdot|\bbX),p^{LL}_{G_*}(\cdot|\bbX))]}{m_n\mathcal{L}_{3n}}\geq\int\liminf_{n\to\infty}\dfrac{|p^{LL}_{G_n}(y|\bx)-p^{LL}_{G_*}(y|\bx)|}{2m_n\mathcal{L}_{3n}}\dint(\bx,y).
\end{align*}
Thus, we deduce that 
\begin{align*}
    \dfrac{|p^{LL}_{G_n}(y|\bx)-p^{LL}_{G_*}(y|\bx)|}{2m_n\mathcal{L}_{3n}}\to0,
\end{align*}
which results in $Q_n/[m_n\mathcal{L}_{3n}]\to0$ as $n\to\infty$ for almost surely $(\bx,y)$. Next, we denote
\begin{align*}
    &\frac{\exp(\bjn)S_{n,j_2|j_1,\balpha_1,\balpha_2,\balpha_3,\rho}}{m_n\mathcal{L}_{3n}}\to \phi_{j_2|j_1,\balpha_1,\balpha_2,\balpha_3,\rho}, \quad &\frac{\exp(\bjn)T_{n,j_2|j_1,\bpsi}}{m_n\mathcal{L}_{3n}}\to\varphi_{j_2|j_1,\bpsi},\\
    &\frac{\exp(\bjn)(\dajn)^{\bgamma}}{m_n\mathcal{L}_{3n}}\to\lambda_{j_1,\bgamma}, \quad &\frac{\exp(\bjn)-\exp(\bj)}{m_n\mathcal{L}_{3n}}\to\chi_{j_1}
\end{align*}
with a note that at least one among them is non-zero. Then, the decomposition of $Q_n$ in equation~\eqref{eq:decompose_Qn_ll} indicates that
\begin{align*}
    \lim_{n\to\infty}\frac{Q_n}{m_n\mathcal{L}_{3n}}=\lim_{n\to\infty}\frac{A_{n}}{m_n\mathcal{L}_{3n}}-\lim_{n\to\infty}\frac{B_n}{m_n\mathcal{L}_{3n}}+\lim_{n\to\infty}\frac{C_n}{m_n\mathcal{L}_{3n}},
\end{align*}
in which 
\begin{align*}
    &\lim_{n\to\infty}\frac{A_{n}}{m_n\mathcal{L}_{3n}}=\sum_{j_1=1}^{k^*_1}\sum_{j_2=1}^{k^*_2}\Bigg[\sum_{|\balpha|=0}^{2}\phi_{j_2|j_1,\balpha_1,\balpha_2,\balpha_3,\rho}\cdot\bx^{\balpha_3}\frac{\partial^{|\balpha_1|}F}{\partial\boldsymbol{\omega}^{\balpha_1}}(\bx;\oj)\frac{\partial^{|\balpha_2|}F}{\partial\boldsymbol{a}^{\balpha_2}}(\bx;\aj)\\
    &\hspace{7cm}\times\frac{\partial^{|\balpha_3|+\rho}\pi}{\partial\xi^{|\balpha_3|+\rho}}(y|(\ej)^{\top}\bx+\tj,\vj)\nonumber\\
    &-\sum_{|\bpsi|=0}^{2}\varphi_{j_2|j_1,\bpsi}\cdot\frac{\partial^{|\bpsi|}F}{\partial\boldsymbol{\omega}^{\bpsi}}(\bx;\oj)\exp(-\|\aj-\bx\|)p^{LL,*}_{j_1}(y|\bx)\Bigg]\frac{1}{\sum_{j'_2=1}^{k^*_2}\exp(-\|\boldsymbol{\omega}^*_{j'_2|j_1}-\bx\|+\beta^*_{j'_2|j_1})},\\
    &\lim_{n\to\infty}\frac{B_{n}}{m_n\mathcal{L}_{3n}}=\sum_{j_1=1}^{k^*_1}\sum_{|\gamma|=1}\lambda_{j_1,\bgamma}\cdot \frac{\partial^{|\bgamma|}F}{\partial\boldsymbol{a}^{\bgamma}}(\bx;\aj)p^{LL}_{G_*}(y|\bx),\\
    &\lim_{n\to\infty}\frac{C_{n}}{m_n\mathcal{L}_{3n}}=\sum_{j_1=1}^{k^*_1}\chi_{j_1}\exp(-\|\aj-\bx\|)\left[p^{LL,*}_{j_1}(y|\bx)-p^{LL}_{G_*}(y|\bx)\right].
\end{align*}
Since the set
\begin{align*}
    &\Bigg\{\dfrac{\bx^{\balpha_3}\frac{\partial^{|\balpha_1|}F}{\partial\boldsymbol{\omega}^{\balpha_1}}(\bx;\oj)\frac{\partial^{|\balpha_2|}F}{\partial\boldsymbol{a}^{\balpha_2}}(\bx;\aj)\frac{\partial^{|\balpha_3|+\rho}\pi}{\partial\xi^{|\balpha_3|+\rho}}(y|(\ej)^{\top}\bx+\tj,\vj)}{\sum_{j'_2=1}^{k^*_2}\exp(-\|\boldsymbol{\omega}^*_{j'_2|j_1}-\bx\|+\beta^*_{j'_2|j_1})}:j_1\in[k^*_1], \\
    &\quad j_2\in[k^*_2],0\leq|\balpha_1|+|\balpha_2|+|\balpha_3|\leq \trj,0\leq\rho\leq2(\trj-|\balpha_1|-|\balpha_2|-|\balpha_3|\Bigg\}\\
    \cup&~\Bigg\{\dfrac{\frac{\partial^{|\bpsi|}F}{\partial\boldsymbol{\omega}^{\bpsi}}(\bx;\oj)\exp((\aj)^{\top}\bx)p^{LL,*}_{j_1}(y|\bx)}{\sum_{j'_2=1}^{k^*_2}\exp(-\|\boldsymbol{\omega}^*_{j'_2|j_1}-\bx\|+\beta^*_{j'_2|j_1})}:j_1\in[k^*_1], j_2\in[k^*_2],0\leq|\bpsi|\leq2\Bigg\}\\
    \cup&~\Big\{\bx^{\bgamma}\exp((\aj)^{\top}\bx)p^{LL}_{G_*}(y|\bx), \ \exp((\aj)^{\top}\bx)p^{LL,*}_{j_1}(y|\bx), \ \exp((\aj)^{\top}\bx)p^{LL}_{G_*}(y|\bx)\\
    &\hspace{8cm}:j_1\in[k^*_1],0\leq|\bgamma|\leq2\Big\}
\end{align*}
is linearly independent, we obtain that $\phi_{j_2|j_1,\balpha_1,\balpha_2,\balpha_3,\rho}=\varphi_{j_2|j_1,\bpsi}=\lambda_{j_1,\bgamma}=\chi_{j_1}=0$ for all $j_1\in[k^*_1]$, $j_2\in[k^*_2]$, $0\leq|\balpha_1|+|\balpha_2|+|\balpha_3|\leq \trj$, $0\leq\rho\leq2(\trj-|\balpha_1|-|\balpha_2|-|\balpha_3|)$, $0\leq|\bpsi|\leq 2$ and $0\leq|\bgamma|\leq 1$, which is a contradiction. As a consequence, we obtain the inequality in equation~\eqref{eq:general_local_inequality_over_ss}. Hence, the proof is completed.

\textbf{}\\

\appendix
\centering
\textbf{\Large{Supplementary to
``On Expert Estimation in Hierarchical Mixture of Experts: Beyond Softmax Gating Functions''}}

\justifying
\setlength{\parindent}{0pt}
\textbf{}\\

\noindent
We first discuss the dataset information, preprocessing procedures, and implementation details in Appendices \ref{appendix:dataset_information}, \ref{appendix:data_preprocessing}, and \ref{appendix:implementation_details}. Next, we provide the proof for the convergence of density estimation in Appendix~\ref{appendix:density_rate}. Then, we continue to streamline the proof of Lemma~\ref{lemma:rss_values} in Appendix~\ref{appendix:rss_values} before investigating the identifiability of the Gaussian HMoE in Appendix~\ref{appendix:identifiability}.

\section{Dataset Information}
\label{appendix:dataset_information}
\subsection{MIMIC-IV}
MIMIC-IV \cite{johnson2020mimic} is a comprehensive database containing records from nearly 300,000 patients admitted to a medical center between 2008 and 2019, focusing on a subset of 73,181 ICU stays. We linked core ICU records, including lab results and vital signs, with corresponding chest X-rays \cite{johnson2019mimiccxr}, radiological notes \cite{johnson2023mimic}, and electrocardiogram (ECG) data \cite{gowmimicecg} recorded during the same ICU stay. 

\vspace{0.5 em}
\noindent
\textbf{Tasks of Interest.}
We design an in-hospital mortality prediction task (referred to as \textbf{48-IHM}) to assess our method’s capability in forecasting short-term patient deterioration. Additionally, accurately predicting patient discharge times is vital for improving patient outcomes and managing hospital resources efficiently \cite{bertsimas2022predicting}, leading us to implement the length-of-stay (\textbf{LOS}) task. Both the 48-IHM and LOS tasks are framed as binary classification problems, utilizing a 48-hour observation window (for patients staying at least 48 hours in the ICU) to predict in-hospital mortality (48-IHM) and patient discharge (without death) within the subsequent 48 hours (LOS). Moreover, recognizing the presence of specific acute care conditions in patient records is key for several clinical goals, such as forming cohorts for studies and identifying comorbidities \cite{agarwal2016learning}. Traditional approaches, which often rely on manual chart reviews or billing codes, are increasingly being complemented by machine learning models \cite{harutyunyan2019multitask}. Automating this process demands high-accuracy classifications, which drives the development of our 25-type phenotype classification (\textbf{25-PHE}) task. This multilabel classification problem involves predicting one of 25 acute care conditions using data from the entire ICU stay. We summarize the details of these tasks below:
\begin{itemize}
\item \textbf{48-IHM}: This is a binary classification task where we aim to predict in-hospital mortality based on data collected during the first 48 hours of ICU admission, applicable only to patients who remained in the ICU for at least 48 hours.
\item \textbf{LOS}: The length-of-stay task is structured similarly to 48-IHM. For patients who stayed in the ICU for a minimum of 48 hours, the objective is to predict whether they will be discharged (without death) within the next 48 hours.
\item \textbf{25-PHE}: This multilabel classification task involves predicting one of 25 acute care conditions \cite{elixhauser2009clinical, lovaasen2012icd}, such as congestive heart failure, pneumonia, or shock, at the conclusion of each patient’s ICU stay. Since the original task was developed for diagnoses based on ICD-9 codes, and MIMIC-IV includes both ICD-9 and ICD-10 codes, we convert diagnoses coded in ICD-10 using the conversion database from \cite{butler2007icd}.
\end{itemize}
\noindent
\textbf{Evaluation.}
We concentrated on patients with complete data across all modalities, which yielded a dataset of 8,770 ICU stays for the 48-IHM and LOS tasks, and 14,541 stays for the 25-PHE task. To assess the performance of the single-label tasks, 48-IHM and LOS, we utilize the F1-score and AUROC as our evaluation metrics. For the 25-PHE task, following prior research \cite{zhang2023improving, lin2019predicting, arbabi2019identifying}, we rely on macro-averaged F1-score and AUROC as the primary measures of evaluation. For the multimodal fusion task, we allocated 70\% data for training, while the remaining 30\% was evenly divided between validation and testing. For clinical latent domain discovery, similar to \cite{wu2024iterative}, we segment the dataset into four temporal groups: 2008-2010, 2011-2013, 2014-2016, and 2017-2019. Each group is then divided into training, validation, and testing sets, following a 70\%, 10\%, and 20\% split, respectively. Patients admitted after 2014 are treated as the target test data, while all earlier patients are used as the source training data.

\subsection{eICU}
The eICU dataset \cite{pollard2018eicu} includes over 200,000 visits from 139,000 patients admitted to ICUs in 208 hospitals across the United States. The data was gathered between 2014 and 2015. The 208 hospitals are categorized into four regions based on their geographic location: Midwest, Northeast, West, and South. We define our cohorts by excluding visits from patients younger than 18 or older than 89, as well as visits exceeding 10 days in length or containing fewer than 3 or more than 256 timestamps. Additionally, we omit visits shorter than 12 hours, since predictions are made 12 hours post-admission.

\vspace{0.5 em}
\noindent
\textbf{Tasks of Interest.}
For the readmission task using the eICU dataset, our goal is to predict whether a patient will be readmitted within 15 days after discharge. Similar to the MIMIC-IV dataset, the mortality prediction task focuses on determining whether a patient will pass away following discharge.

\vspace{0.5 em}
\noindent
\textbf{Evaluation.}
The eICU dataset is divided into four regional groups: Midwest, Northeast, West, and South. Each region is further split into 70\% for training, 10\% for validation, and 20\% for testing. To assess the performance gap between regions, we compare the backbone model’s performance when trained on data from the same region versus data from other regions, as proposed by \cite{wu2024iterative}. The region with the largest performance gap (Midwest) is selected as the target test data, while the remaining regions (Northeast, West, and South) are used as the source training data. To compare with baselines from \cite{wu2024iterative}, we use the same evaluation metrics: Area Under the Precision-Recall Curve (AUPRC) and the Area Under the Receiver Operating Characteristic Curve (AUROC) scores.

\subsection{Image Classification Datasets}
\textbf{CIFAR-10.}
CIFAR-10 \cite{krizhevsky2009learning} is a well-known dataset in computer vision, commonly used for object recognition tasks. It contains 60,000 color images, each with a resolution of 32x32 pixels, representing one of 10 object categories (“plane,” “car,” “bird,” “cat,” “deer,” “dog,” “frog,” “horse,” “ship,” “truck”), with 6,000 images per class.

\vspace{0.5 em}
\noindent
\textbf{ImageNet.}
We use the ImageNet database from ILSVRC2012 \cite{russakovsky2015imagenet}, where the task is to classify images into 1,000 distinct categories, using a vast dataset of over 1.2 million training images and 150,000 validation and test images sourced from the ImageNet database.

\vspace{0.5 em}
\noindent
\textbf{Tiny-ImageNet.}
The Tiny-ImageNet is a smaller, more manageable subset of the ImageNet dataset. It contains 100,000 images and 200 classes selected from full ImageNet dataset. All images are resized to 64×64 pixels to reduce computational demands.

\vspace{0.5 em}
\noindent
\textbf{CIFAR-10-Corruption.}
The CIFAR-10-corruption \cite{hendrycks2018benchmarking} dataset is a standard benchmark for evaluating distribution shifts. It contains 50,000 clean samples in total, along with 10,000 corrupted samples for each corruption type and each severity level. There are 20 types of corruptions, each with 5 levels of severity.

\section{Data Preprocessing for Clinical Tasks}
\label{appendix:data_preprocessing}
During preprocessing, we selected 30 relevant lab and chart events from each patient’s ICU records to capture vital sign measurements. For chest X-rays, we employed a pre-trained DenseNet-121 model \cite{cohen2022torchxrayvision}, which had been fine-tuned on the CheXpert dataset \cite{irvin2019chexpert}, to extract 1024-dimensional image embeddings. Additionally, we used the BioClinicalBERT model \cite{alsentzer2019publicly} to generate 768-dimensional embeddings for the radiological notes.

\vspace{0.5 em}
\noindent
\textbf{Time Series.}
We selected 30 time-series events for analysis, as outlined in \cite{soenksen2022integrated}. This included nine vital signs: heart rate, mean/systolic/diastolic blood pressure, respiratory rate, oxygen saturation, and Glasgow Coma Scale (GCS) verbal, eye, and motor response. Additionally, 21 laboratory values were incorporated: potassium, sodium, chloride, creatinine, urea nitrogen, bicarbonate, anion gap, hemoglobin, hematocrit, magnesium, platelet count, phosphate, white blood cell count, total calcium, MCH, red blood cell count, MCHC, MCV, RDW, platelet count, neutrophil count, and vancomycin. Each time series value was standardized to have a mean of 0 and a standard deviation of 1, based on values from the training set. We use the Transformer as an encoder for time series data.

\vspace{0.5 em}
\noindent
\textbf{Chest X-Rays.}
To integrate medical imaging into our analysis, we use the MIMIC-CXR-JPG module \cite{johnson2019mimicjpg} available through Physionet \cite{goldberger2000physiobank}, which contains 377,110 JPG images derived from the DICOM-based MIMIC-CXR database \cite{johnson2019mimiccxr}. As described in \cite{soenksen2022integrated}, each image is resized to 224 $\times$ 224 pixels, and we extract embeddings from the final layer of the DenseNet121 model. To identify X-rays taken during the patient’s ICU stay, we match subject IDs from MIMIC-CXR-JPG with the core MIMIC-IV database and then filter the X-rays to those captured between the ICU admission and discharge times.

\vspace{0.5 em}
\noindent
\textbf{Clinical Notes}
To incorporate text data, we use the MIMIC-IV-Note module \cite{johnson2023mimic}, which includes 2,321,355 deidentified radiology reports for 237,427 patients. These reports can be linked to patients in the main MIMIC-IV dataset using a similar matching method as employed for chest X-rays. It is important to note that we were unable to access intermediate clinical notes (i.e., notes recorded by clinicians during the patient’s stay), as they have not yet been made publicly available. We extract note embeddings using the Bio-Clinical BERT model \cite{alsentzer2019publicly}.

\section{Implementation Details}
\label{appendix:implementation_details}
\subsection{Model Architecture}
Once embeddings from each input modality or domain are generated, we address the issue of irregularity in the data. To do this, we use a discretized multi-time attention (mTAND) module \cite{shukla2021multi}, which applies a time attention mechanism \cite{kazemi2019time2vec} to convert irregularly sampled observations into discrete time intervals. This approach has been employed in previous works such as \cite{zhang2023improving, han2024fusemoe}. The mTAND module transforms the irregular sequences into fixed-length representations, which are then passed into the MoE fusion layer with a residual connection. This fusion layer comprises multi-head self-attention followed by the HMoE module. In total, there are 12 MoE fusion layers, and the output from this layer is optimized using task-specific loss and load imbalance loss. We apply a dropout rate of 0.1 and use the Adam optimizer with a learning rate of 1e-4 and a weight decay of 1e-5. All models are trained for 100 epochs. For the multimodal experiment, we use a batch size of 2, while for the latent domain discovery experiment, the batch size is set to 256.

\section{Proofs for Convergence of Density Estimation}
\label{appendix:density_rate}
\begin{proof}[Proof of Proposition~\ref{prop:density_estimation}]
To streamline the arguments for this proof, it is necessary to define some notations that will be used in the sequel. First of all, let $\mathcal{P}^{type}_{k^*_1,k_2}(\Theta)$ stand for the set of conditional density functions w.r.t mixing measures in $\mathcal{G}_{k^*_1,k_2}(\Theta)$ where $type\in\{SS,SL,LL\}$, that is,
\begin{align*}
    \mathcal{P}^{type}_{k^*_1,k_2}(\Theta):=\{p^{type}_{G}(y|\bx):G\in\mathcal{G}_{k^*_1,k_2}(\Theta)\}.
\end{align*}
Additionally, we also define
\begin{align*}
    \widetilde{\mathcal{P}}^{type}_{k^*_1,k_2}(\Theta)&:=\{p^{type}_{(G+G_*)/2}(y|\bx):G\in\mathcal{G}_{k^*_1,k_2}(\Theta)\},\\
    \widetilde{\mathcal{P}}^{type,1/2}_{k^*_1,k_2}(\Theta)&:=\{(p^{type}_{(G+G_*)/2})^{1/2}(y|\bx):G\in\mathcal{G}_{k^*_1,k_2}(\Theta)\}.
\end{align*}
Next, for each $\delta>0$, we define the $L^{2}$-ball centered around the density function $p^{type}_{G_*}$ and intersected with the set $\widetilde{\mathcal{P}}^{type,1/2}_{k^*_1,k_2}(\Theta)$ as
\begin{align*}   
\widetilde{\mathcal{P}}^{type,1/2}_{k^*_1,k_2}(\Theta,\delta):=\left\{p^{1/2} \in \widetilde{\mathcal{P}}^{type,1/2}_{k^*_1,k_2}(\Theta): h(p,p^{type}_{G_*}) \leq\delta\right\}.
\end{align*}
Following the suggestion from Geer et. al. \cite{vandeGeer-00}, we utilize the following integral to capture the size of the above $L^2$-ball:
\begin{align}
    \label{eq:bracket_size}
    \mathcal{J}_B(\delta, \widetilde{\mathcal{P}}^{type,1/2}_{k^*_1,k_2}(\Theta,\delta)):=\int_{\delta^2/2^{13}}^{\delta}H_B^{1/2}(t, \widetilde{\mathcal{P}}^{type,1/2}_{k^*_1,k_2}(\Theta,t),\|\cdot\|_{L^2})~\dint t\vee \delta,
\end{align}
where the term $H_B(t, \widetilde{\mathcal{P}}^{type,1/2}_{k^*_1,k_2}(\Theta,t),\|\cdot\|_{L^2})$ denotes the bracketing entropy \cite{vandeGeer-00} of $ \widetilde{\mathcal{P}}^{type,1/2}_{k^*_1,k_2}(\Theta,t)$ under the $L^{2}$-norm, and $t\vee\delta:=\max\{t,\delta\}$. \\

\vspace{0.5 em}
\noindent
Let us recall the statement of Theorem 7.4 in \cite{vandeGeer-00} with adapted notations to our paper as follows:
\begin{lemma}[Theorem 7.4, \cite{vandeGeer-00}]
    \label{lemma:density_rate}
    Let $\Psi(\delta)\geq \mathcal{J}_B(\delta, \widetilde{\mathcal{P}}^{type,1/2}_{k^*_1,k_2}(\Theta,\delta))$ be such that $\Psi(\delta)/\delta^2$ is a non-increasing function of $\delta$. Then, for some universal constant $c$ and for some sequence $(\delta_n)$ such that $\sqrt{n}\delta^2_n\geq c\Psi(\delta_n)$, the following inequality holds for all  $\delta\geq \delta_n$:
    \begin{align*}
        \mathbb{P}\Big(\bbE_{\bbX}[h(p^{type}_{\widehat{G}^{type}_n}(\cdot|\bbX),p^{type}_{G_*}(\cdot|\bbX))] > \delta\Big)\leq c \exp\left(-\frac{n\delta^2}{c^2}\right).
    \end{align*}
\end{lemma}
\noindent
\textbf{Proof overview.} Given that the expert functions are Lipschitz continuous, we begin with showing that the following bound holds for any $0 < \varepsilon \leq 1/2$:
\begin{align}    
H_B(\varepsilon,\mathcal{P}^{type}_{k^*_1,k_2}(\Theta),h) \lesssim \log(1/\varepsilon), \label{eq:bracket_entropy_bound}
\end{align}
which yields that 
\begin{align}
    \label{eq:bracketing_integral}
    \mathcal{J}_B(\delta, \widetilde{\mathcal{P}}^{type,1/2}_{k^*_1,k_2}(\Theta,\delta))&= \int_{\delta^2/2^{13}}^{\delta}H_B^{1/2}(t, \widetilde{\mathcal{P}}^{type,1/2}_{k^*_1,k_2}(\Theta,t),\|\cdot\|_{L^2})~\dint t\vee \delta\nonumber\\
    &\leq\int_{\delta^2/2^{13}}^{\delta}H_B^{1/2}(t, \mathcal{P}^{type}_{k^*_1,k_2}(\Theta,t),h)~\dint t\vee \delta\nonumber\\
    &\lesssim \int_{\delta^2/2^{13}}^{\delta}\log(1/t)dt\vee\delta.
\end{align}
Let $\Psi(\delta)=\delta\cdot[\log(1/\delta)]^{1/2}$, then it can be checked that $\Psi(\delta)/\delta^2$ is a non-increasing function of $\delta$. Moreover, the result in equation~\eqref{eq:bracketing_integral} implies that $\Psi(\delta)\geq \mathcal{J}_B(\delta,\widetilde{\mathcal{P}}^{type,1/2}_{k^*_1,k_2}(\Theta,\delta))$. By choosing $\delta_n=\sqrt{\log(n)/n}$, we have that $\sqrt{n}\delta^2_n\geq c\Psi(\delta_n)$ for some universal constant $c$. Then, the conclusion of this theorem is achieved according to Lemma~\ref{lemma:density_rate}. Consequently, it is sufficient to derive the bracketing entropy bound in equation~\eqref{eq:bracket_entropy_bound}.

\vspace{0.5 em}
\noindent
\textbf{Proof for the bound~\eqref{eq:bracket_entropy_bound}.} To begin with, we provide an upper bound for the Gaussian density function $\pi(y|\eta^{\top}\bx+\tau,\nu)$. In particular, since the input space $\mathcal{X}$ and the parameter space $\Theta$ are both bounded, we can find some constant $\kappa,\ell,u>0$ such that $-\kappa\leq\eta^{\top}\bx+\tau\leq\kappa$ and $\ell\leq\nu\leq u$. Then, it can be validated that 
\begin{align*}
    \pi(y|\eta^{\top}\bx+\tau,\nu)=\frac{1}{\sqrt{2\pi\nu}}\exp\Big(-\frac{(y-(\eta^{\top}\bx+\tau))^2}{2\nu}\Big)\leq\frac{1}{\sqrt{2\pi\ell}},
\end{align*}
for any $|y|<2\kappa$. On the other hand, for $|y|\geq2\kappa$, since $\frac{(y-(\eta^{\top}\bx+\tau))^2}{2\nu}\geq\frac{y^2}{8u}$, we have that
\begin{align*}
     \pi(y|\eta^{\top}\bx+\tau,\nu)\leq\frac{1}{\sqrt{2\pi\ell}}\exp\Big(-\frac{y^2}{8u}\Big).
\end{align*}
Therefore, we deduce that $\pi(y|\eta^{\top}\bx+\tau,\nu)\leq M(y|\bx)$, where 
\begin{align*}
    M(y|\bx)=\begin{cases}
        \frac{1}{\sqrt{2\pi\ell}}\exp\Big(-\frac{y^2}{8u}\Big), \quad \text{for } |y|\geq 2\kappa,\\
        \frac{1}{\sqrt{2\pi\ell}}, \hspace{2.3cm} \text{for } |y|<2\kappa.
    \end{cases}
\end{align*}
Next, let $0<\tau\leq\varepsilon$ and $\{\pi_1,\ldots,\pi_N\}$ be the $\tau$-cover under the $L^{\infty}$-norm of the set $\mathcal{P}^{type}_{k^*_1,k_2}(\Theta)$ where $N:={N}(\tau,\mathcal{P}^{type}_{k^*_1,k_2}(\Theta),\|\cdot\|_{L^{\infty}})$ stands for the $\tau$-covering number of the norm space $(\mathcal{P}^{type}_{k^*_1,k_2}(\Theta),\|\cdot\|_{L^{\infty}})$. Equipped with the brackets of the form $[L_i,U_i]$ where
    \begin{align*}
        L_i(y|\bx)&:=\max\{\pi_i(y|\bx)-\tau,0\},\\
        U_i(y|\bx)&:=\max\{\pi_i(y|\bx)+\tau, M(y|\bx)\},
    \end{align*}
for all $i\in[N]$, we can validate that $\mathcal{P}^{type}_{k^*_1,k_2}(\Theta)\subset\cup_{i=1}^{N}[L_i,U_i]$, and $U_i(y|\bx)-L_i(y|\bx)\leq \min\{2\tau,M\}$. Those results yield that
\begin{align*}
    \|U_i-L_i\|_{L^1}=\int(U_i(y|\bx)-L_i(y|\bx))\dint(\bx,y)\leq\int 2\tau\dint(\bx,y)=2\tau,
\end{align*}
From the definition of the bracketing entropy, we have
\begin{align}
    \label{eq:standard_bracketing_covering}
    H_B(2\tau,\mathcal{P}^{type}_{k^*_1,k_2}(\Theta),\|\cdot\|_{L^1})\leq\log N=\log {N}(\tau,\mathcal{P}^{type}_{k^*_1,k_2}(\Theta),\|\cdot\|_{L^{\infty}}).
\end{align}
Therefore, it suffices to provide an upper bound for the covering number $N$. Indeed, let us denote $\Delta:=\{(b,\boldsymbol{a})\in\mathbb{R}\times\mathbb{R}^{d}:(b,\boldsymbol{a},\beta,\boldsymbol{\omega},\tau,\boldsymbol{\eta},\nu)\in\Theta\}$ and $\Omega:=\{(\beta,\boldsymbol{\omega},\tau,\boldsymbol{\eta},\nu)\in\mathbb{R}\times\mathbb{R}^{d}\times\mathbb{R}\times\mathbb{R}^{d}\times\mathbb{R}_+:(b,\boldsymbol{a},\beta,\boldsymbol{\omega},\tau,\boldsymbol{\eta},\nu)\in\Theta\}$. As $\Theta$ is a compact set, so are $\Delta$ and $\Omega$. Thus, we can find $\tau$-covers $\Delta_{\tau}$ and ${\Omega}_{\tau}$ for $\Delta$ and $\Omega$, respectively. Furthermore, it can be validated that 
\begin{align*}
    |\Delta_{\tau}|\leq \mathcal{O}_{P}(\tau^{-(d+1)k^*_1}), \quad |\Omega_{\tau}|\leq \mathcal{O}_{P}(\tau^{-(2d+3)k^*_1k_2}).
\end{align*}
For each mixing measure $G=\sum_{i_1=1}^{k^*_1}\exp(b_{i_1})\sum_{i_2=1}^{k_2}\exp(\beta_{i_2|i_1})\delta_{(\boldsymbol{a}_{i_1},\boldsymbol{\omega}_{i_2|i_1},\boldsymbol{\eta}_{i_1i_2},\tau_{i_1i_2},\nu_{i_1i_2})}\in\mathcal{G}_{k^*_1,k_2}(\Theta)$, we consider two other mixing measures $G'$ and $\overline{G}$ defined as
\begin{align*}
    G'&:=\sum_{i_1=1}^{k^*_1}\exp(b_{i_1})\sum_{i_2=1}^{k_2}\exp(\overline{\beta}_{i_2|i_1})\delta_{(\boldsymbol{a}_{i_1},\overline{\boldsymbol{\omega}}_{i_2|i_1},\overline{\boldsymbol{\eta}}_{i_1i_2},\overline{\tau}_{i_1i_2},\overline{\nu}_{i_1i_2})}, \\
    \overline{G}&:=\sum_{i_1=1}^{k^*_1}\exp(\overline{b}_{i_1})\sum_{i_2=1}^{k_2}\exp(\overline{\beta}_{i_2|i_1})\delta_{(\overline{\boldsymbol{a}}_{i_1},\overline{\boldsymbol{\omega}}_{i_2|i_1},\overline{\boldsymbol{\eta}}_{i_1i_2},\overline{\tau}_{i_1i_2},\overline{\nu}_{i_1i_2})}.
\end{align*}
Above, $(\overline{\beta}_{i_2|i_1},\overline{\boldsymbol{\omega}}_{i_2|i_1},\overline{\boldsymbol{\eta}}_{i_1i_2},\overline{\tau}_{i_1i_2},\overline{\nu}_{i_1i_2})\in{\Omega}_{\tau}$ such that $(\overline{\beta}_{i_2|i_1},\overline{\boldsymbol{\omega}}_{i_2|i_1},\overline{\boldsymbol{\eta}}_{i_1i_2},\overline{\tau}_{i_1i_2},\overline{\nu}_{i_1i_2})$ is the closest to $(\beta_{i_2|i_1},\boldsymbol{\omega}_{i_2|i_1},\boldsymbol{\eta}_{i_1i_2},\tau_{i_1i_2},\nu_{i_1i_2})$ in that set, while $(\overline{b}_{i_1},\overline{\boldsymbol{a}}_{i_1})\in\Delta_{\tau}$ is the closest to $(b_{i_1},\boldsymbol{\omega}_{i})$ in that set. 

\vspace{0.5 em}
\noindent
Now, we begin bounding the term $\|p^{type}_{G}-p^{type}_{G'}\|_{L^{\infty}}$. For brevity, we will consider only the case when $type=SS$, while the other two cases when $type=SL$ and $type=LL$ can be argued in a similar fashion.

\textbf{When $type=SS$:} Let us define
\begin{align*}
    p^{SS}_{i_1}(\bx)&:=\sum_{i_2=1}^{k_2}\softmax((\boldsymbol{\omega}_{i_2|i_1})^{\top}\bx+\beta_{i_2|i_1})\pi(y|(\boldsymbol{\eta}_{i_1i_2})^{\top}\bx+\tau_{i_1i_2},\nu_{i_1i_2}),\\
    \overline{p}^{SS}_{i_1}(\bx)&:=\sum_{i_2=1}^{k_2}\softmax((\overline{\boldsymbol{\omega}}_{i_2|i_1})^{\top}\bx+\overline{\beta}_{i_2|i_1})\pi(y|(\overline{\boldsymbol{\eta}}_{i_1i_2})^{\top}\bx+\overline{\tau}_{i_1i_2},\overline{\nu}_{i_1i_2}).
\end{align*}
Then, we have
\begin{align}
    \label{eq:triangle_1}
    \|p^{SS}_{G}-p^{SS}_{G'}\|_{L^{\infty}}&=\sum_{i_1=1}^{k^*_1}\softmax\left((\boldsymbol{a}_{i_1})^{\top}\bx+b_{i_1}\right)\cdot\|p^{SS}_{i_1}-\overline{p}^{SS}_{i_1}\|_{L^{\infty}}\leq\sum_{i_1=1}^{k^*_1}\|p^{SS}_{i_1}-\overline{p}^{SS}_{i_1}\|_{L^{\infty}}.
\end{align}
Next, we need to bound the terms $p^{SS}_{i_1}(\bx)-\overline{p}^{SS}_{i_1}(\bx)$ using the triangle inequality
\begin{align}
    \label{eq:sub_triangle}
    \|p^{SS}_{i_1}-\overline{p}^{SS}_{i_1}\|_{L^{\infty}}\leq\|p^{SS}_{i_1}-\widetilde{p}^{SS}_{i_1}\|_{L^{\infty}}+\|\widetilde{p}^{SS}_{i_1}-\overline{p}^{SS}_{i_1}\|_{L^{\infty}},
\end{align}
where we define
\begin{align*}
    \widetilde{p}^{SS}_{i_1}(\bx):=\sum_{i_2=1}^{k_2}\softmax((\boldsymbol{\omega}_{i_2|i_1})^{\top}\bx+\beta_{i_2|i_1})\pi(y|(\overline{\boldsymbol{\eta}}_{i_1i_2})^{\top}\bx+\overline{\tau}_{i_1i_2},\overline{\nu}_{i_1i_2}).
\end{align*} 
Firstly, we have
\begin{align}
    \label{eq:sub_triangle_1}
    &\|p^{SS}_{i_1}-\widetilde{p}^{SS}_{i_1}\|_{L^{\infty}}\leq\sum_{i_2=1}^{k_2}\softmax((\boldsymbol{\omega}_{i_2|i_1})^{\top}\bx+\beta_{i_2|i_1})\nonumber\\
    &\hspace{3cm}\times\|\pi(y|(\boldsymbol{\eta}_{i_1i_2})^{\top}\bx+\tau_{i_1i_2},\nu_{i_1i_2})-\pi(y|(\overline{\boldsymbol{\eta}}_{i_1i_2})^{\top}\bx+\overline{\tau}_{i_1i_2},\overline{\nu}_{i_1i_2})\|_{L^{\infty}}\nonumber\\
    &\leq\sum_{i_2=1}^{k_2}\|\pi(y|(\boldsymbol{\eta}_{i_1i_2})^{\top}\bx+\tau_{i_1i_2},\nu_{i_1i_2})-\pi(y|(\overline{\boldsymbol{\eta}}_{i_1i_2})^{\top}\bx+\overline{\tau}_{i_1i_2},\overline{\nu}_{i_1i_2})\|_{L^{\infty}}\nonumber\\
    &\lesssim\sum_{i_2=1}^{k_2}\Big(\|\boldsymbol{\eta}_{i_1i_2}-\overline{\boldsymbol{\eta}}_{i_1i_2}\|+|\tau_{i_1i_2}-\overline{\tau}_{i_1i_2}|+|\nu_{i_1i_2}-\overline{\nu}_{i_1i_2}|\Big)\lesssim\tau.
\end{align}
Secondly, since $\mathcal{X}$ is a bounded set, we may assume that $\|\bx\|\leq B$ for any $\bx\in\mathcal{X}$. Then, it follows that
\begin{align}
    \label{eq:sub_triangle_2}
    \|\widetilde{p}^{SS}_{i_1}-\overline{p}^{SS}_{i_1}\|_{L^{\infty}}&\leq\sum_{i_2=1}^{k_2}\Big|\softmax((\boldsymbol{\omega}_{i_2|i_1})^{\top}\bx+\beta_{i_2|i_1})-\softmax((\overline{\boldsymbol{\omega}}_{i_2|i_1})^{\top}\bx+\overline{\beta}_{i_2|i_1})\Big|\nonumber\\
    &\hspace{4cm}\times\|\pi(y|(\overline{\boldsymbol{\eta}}_{i_1i_2})^{\top}\bx+\overline{\tau}_{i_1i_2},\overline{\nu}_{i_1i_2})\|_{L^{\infty}}\nonumber\\
    &\lesssim\sum_{i_2=1}^{k_2}\Big[\|\boldsymbol{\omega}_{i_2|i_1}-\overline{\boldsymbol{\omega}}_{i_2|i_1}\|\cdot|\bx\|+|\beta_{i_2|i_1}-\overline{\beta}_{i_2|i_1}|\Big]\nonumber\\
    &\leq\sum_{i_2=1}^{k_2}\Big(\tau B+\tau\Big)\lesssim\tau.
\end{align}
From the results in equations~\eqref{eq:triangle_1}, \eqref{eq:sub_triangle}, \eqref{eq:sub_triangle_1} and \eqref{eq:sub_triangle_2}, we deduce that
\begin{align}
    \label{eq:standard_first_term_triangle}
    \|p^{SS}_{G}-p^{SS}_{G'}\|_{L^{\infty}}\lesssim\tau.
\end{align}
\noindent
Furthermore, we have
\begin{align}
    \|p^{SS}_{G'}-p^{SS}_{\overline{G}}\|_{L^{\infty}}&=\sum_{i_1=1}^{k^*_1}|\softmax((\boldsymbol{a}_{i_1})^{\top}\bx+b_{i_1})-\softmax((\overline{\boldsymbol{a}}_{i_1})^{\top}\bx+\overline{b}_{i_1})|\cdot\|\pi(y|(\overline{\boldsymbol{\eta}}_{i_1i_2})^{\top}\bx+\overline{\tau}_{i_1i_2},\overline{\nu}_{i_1i_2})\|_{L^{\infty}}\nonumber\\
    &\lesssim\sum_{i_1=1}^{k^*_1}\Big(\|\boldsymbol{a}_{i_1}-\overline{\boldsymbol{a}}_{i_1}\|\cdot\|\bx\|+|b_{i_1}-\overline{b}_{i_1}|\Big)\nonumber\\
    \label{eq:standard_second_term_triangle}
    &\leq\sum_{i_1=1}^{k^*_1}(\tau B+\tau)\lesssim\tau.
\end{align}
According to the triangle inequality and the results in equations~\eqref{eq:standard_first_term_triangle}, \eqref{eq:standard_second_term_triangle}, we have
\begin{align*}
    \|p^{SS}_{G}-p^{SS}_{\overline{G}}\|_{L^{\infty}}\leq \|p^{SS}_{G}-p^{SS}_{G'}\|_{L^{\infty}}+\|p^{SS}_{G'}-p^{SS}_{\overline{G}}\|_{L^{\infty}}\lesssim\tau.
\end{align*}
By definition of the covering number, we deduce that
\begin{align}
    \label{eq:standard_covering_bound}
    {N}(\tau,\mathcal{P}^{type}_{k^*_1,k_2}(\Theta),\normf{\cdot})&\leq |\Delta_{\tau}|\times|\Omega_{\tau}|\nonumber\\
    &\leq\mathcal{O}_{P}(\tau^{-(d+1)k^*_1})\times\mathcal{O}_{P}(\tau^{-(2d+3)k^*_1k_2})\nonumber\\
    &\leq\mathcal{O}_P(\tau^{-(d+1)k^*_1-(2d+3)k^*_1k_2}).
\end{align}
Combine the result in equation~\eqref{eq:standard_bracketing_covering} with that in \eqref{eq:standard_covering_bound}, we arrive at
\begin{align*}
    H_B(2\tau,\mathcal{P}^{type}_{k^*_1,k_2}(\Theta),\|\cdot\|_{L^1})\lesssim \log(1/\tau).
\end{align*}
Let $\tau=\varepsilon/2$, then it follows that
\begin{align*}
    H_B(\varepsilon,\mathcal{P}^{type}_{k^*_1,k_2}(\Theta),\|.\|_{L^1}) \lesssim \log(1/\varepsilon).
\end{align*}
Finally, due to the inequality between the Hellinger distance and the $L^1$-norm $h\leq\|\cdot\|_{L^1}$, we achieve the conclusion that
\begin{align*}
    H_B(\varepsilon,\mathcal{P}^{type}_{k^*_1,k_2}(\Theta),h) \lesssim \log(1/\varepsilon).
\end{align*}
Hence, the proof is completed.
\end{proof}

\section{Proof of Lemma~\ref{lemma:rss_values}}
\label{appendix:rss_values}

Firstly, let us recall the system of polynomial equations given in equation~\eqref{eq:system_SS}:
\begin{align}
    \label{eq:original_system}
    \sum_{i_2 = 1}^{m} \sum_{\alpha \in \mathcal{I}^{SS}_{\brho_{1}, \rho_{2}}} \dfrac{p_{i_2}^2~ \boldsymbol{q}_{1i_2}^{\balpha_1} ~\boldsymbol{q}_{2i_2}^{\balpha_2}~\boldsymbol{q}_{3i_2}^{\balpha_3} ~q_{4i_2}^{\alpha_4} ~q_{5i_2}^{\alpha_5}}{\balpha_1!~\balpha_2!~\balpha_3!~\alpha_4!\alpha_5!} = 0, \quad 1\leq |\brho_1|+\rho_2\leq r,
\end{align}
where $\mathcal{I}^{SS}_{\brho_{1}, \rho_{2}} = \{\alpha = (\balpha_1, \balpha_2, \balpha_3, \alpha_4, \alpha_5) \in \mathbb{N}^{d} \times \mathbb{N}^{d} \times \mathbb{N}^{d} \times \mathbb{N} \times \mathbb{N}: \ \balpha_1 + \balpha_2+\balpha_3 = \brho_1, \ \alpha_4 + 2 \alpha_5 = \rho_2- |\balpha_3| \}$. 

\vspace{0.5 em}
\noindent
\textbf{When $m=2$:} By observing a portion of the above system when $\brho_1= \mathbf{0}_{d}$, which is given by
\begin{align}
    \label{eq:reduced_system}
    \sum_{i_2=1}^{m}\sum_{\alpha_4+2\alpha_5=\rho_2}\dfrac{p_{i_2}^{2}~q_{4i_2}^{\alpha_4}~q_{5i_2}^{\alpha_5}}{\alpha_4!~\alpha_5!}=0, \quad \rho_2=1,2,\ldots,r.
\end{align}
Proposition 2.1 in \cite{Ho-Nguyen-Ann-16} shows that the smallest $r\in\mathbb{N}$ such that the system~\eqref{eq:reduced_system} does not admit any non-trivial solutions when $m=2$ is $r=4$. Note that a solution of the system~\ref{eq:reduced_system} is called non-trivial in \cite{Ho-Nguyen-Ann-16} if all the values of $p_{i_2}$ are different from zero, whereas at least one among $q_{4i_2}$ is non-zero. This definition of non-trivial solutions totally aligns with ours for the system~\eqref{eq:original_system}. Therefore, we have $\Bar{r}(m)\leq 4$, and it suffices to prove that $\Bar{r}(m)>3$.

\vspace{0.5 em}
\noindent
Indeed, when $r=3$, we demonstrate that the system~\eqref{eq:original_system} admits a non-trivial solution: $p_{i_2}=1$, $\boldsymbol{q}_{1i_2} = \boldsymbol{q}_{2i_2}=\boldsymbol{q}_{3i_2} = \zerod$ for all $i_2\in[m]$, $q_{41}=1$, $q_{42}=-1$, $q_{51}=q_{52}=-\frac{1}{2}$. Since $\boldsymbol{q}_{1i_2} = \boldsymbol{q}_{2i_2}=\boldsymbol{q}_{3i_2} = \zerod$, this solution clearly satisfies the equations associated with $\brho_1\neq\zerod$. Thus, we only need to verify those with $\brho_1=\zerod$, which are given by
\begin{align*}
    \sum_{j=1}^mp_{i_2}^2q_{4i_2}&=0,\nonumber\\
     \sum_{i_2=1}^{m}p_{i_2}^2\Big(\frac{1}{2}q_{4i_2}^2+q_{5i_2}\Big)&=0,\nonumber\\
    \sum_{i_2=1}^{m}p_{i_2}^2\Big(\frac{1}{3!}q_{4i_2}^3+q_{4i_2}q_{5i_2}\Big)&=0.
\end{align*}
By simple calculations, we can check that $p_{i_2}=1$, $q_{41}=1$, $q_{42}=-1$, $q_{51}=q_{52}=-\frac{1}{2}$ satisfies the above equations. Hence, we obtain that $\Bar{r}(m)>3$, leading to $\Bar{r}(m)=4$.

\vspace{0.5 em}
\noindent
\textbf{When $m=3$:} Note that $\Bar{r}(m)$ is a monotonically increasing function of $m$. Therefore, it follows from the previous result that $\Bar{r}(m)>\Bar{r}(2)=4$, or equivalently, $\Bar{r}(m)\geq 5$. Additionally, according to Proposition 2.1 in \cite{Ho-Nguyen-Ann-16}, we deduce that $\Bar{r}(m)\leq 6$ based on the reduced system in equation~\eqref{eq:reduced_system}. Thus, we only need to show that $\Bar{r}(m)>5$.

\vspace{0.5 em}
\noindent
Indeed, we show that the following is a non-trivial solution of the system~\eqref{eq:original_system} when $r=5$: 
\begin{align*}
    p_{i_2}=1,\quad \boldsymbol{q}_{1i_2} = \boldsymbol{q}_{2i_2}=\boldsymbol{q}_{3i_2} = \zerod, \quad \forall i_2\in[m],\\
    q_{41}=\frac{\sqrt{3}}{3}, \quad q_{42}=-\frac{\sqrt{3}}{3}, \quad q_{43}=0,\\
    q_{51}=q_{52}=-\frac{1}{6}, \quad q_{53}=0.
\end{align*}
Since $\boldsymbol{q}_{1i_2} = \boldsymbol{q}_{2i_2}=\boldsymbol{q}_{3i_2} = \zerod$, this solution clearly satisfies the equations associated with $\brho_1\neq\zerod$. Thus, we only need to verify those with $\brho_1=\zerod$, which are given by
\begin{align*} 
    \sum_{j=1}^mp_{i_2}^2q_{4i_2}&=0,\nonumber\\
     \sum_{i_2=1}^{m}p_{i_2}^2\Big(\frac{1}{2}q_{4i_2}^2+q_{5i_2}\Big)&=0,\nonumber\\
    \sum_{i_2=1}^{m}p_{i_2}^2\Big(\frac{1}{3!}q_{4i_2}^3+q_{4i_2}q_{5i_2}\Big)&=0,\\
    \sum_{i_2=1}^{m}p_{i_2}^2\Big(\frac{1}{4!}q_{4i_2}^4+\frac{1}{2!}q_{4i_2}^2q_{5i_2}+\frac{1}{2!}q_{5i_2}^2\Big)&=0,\\
    \sum_{i_2=1}^{m}p_{i_2}^2\Big(\frac{1}{5!}q_{4i_2}^5+\frac{1}{3!}q_{4i_2}^3q_{5i_2}+\frac{1}{2!}q_{4i_2}q_{5i_2}^2\Big)&=0.
\end{align*}
By simple calculations, it can be validated that $p_{i_2}=1$, $q_{41}=\frac{\sqrt{3}}{3}$, $q_{42}=-\frac{\sqrt{3}}{3}$, $q_{43}=0$, $q_{51}=q_{52}=-\frac{1}{6}$, $q_{53}=0$ satisfies the above equations. Hence, we conclude $\Bar{r}(m)>5$, meaning that $\Bar{r}(m)=6$.

    

\section{Identifiability of the Gaussian HMoE}
\label{appendix:identifiability}
\begin{proof}[Proof of Proposition~\ref{prop:identifiability}]
    In this proof, we will consider only the case when $type=SS$ as other cases can be done similarly.

    \vspace{0.5 em}
\noindent
    To start with, let us write the equation  $p^{SS}_{G}(y|\bx)=p^{SS}_{G_*}(y|\bx)$ explicitly as follows:
    \begin{align}
        \label{eq:general_identifiable_equation}
        &\sum_{i_1=1}^{k^*_1}\softmax\Big((\boldsymbol{a}_{i_1})^{\top}\bx+b_{i_1}\Big)\sum_{i_2=1}^{k_2}\softmax\Big((\boldsymbol{\omega}_{i_2|i_1})^{\top}\bx+\beta_{i_2|i_1}\Big)\pi(y|(\boldsymbol{\eta}_{i_1i_2})^{\top}\bx+\tau_{i_1i_2},\nu_{i_1i_2})\nonumber\\
        &=\sum_{i_1=1}^{k^*_1}\softmax\Big((\boldsymbol{a}^*_{i_1})^{\top}\bx+b^*_{i_1}\Big)\sum_{i_2=1}^{k^*_2}\softmax\Big((\boldsymbol{\omega}^*_{i_2|i_1})^{\top}\bx+\beta^*_{i_2|i_1}\Big)\pi(y|(\boldsymbol{\eta}^*_{i_1i_2})^{\top}\bx+\tau^*_{i_1i_2},\nu^*_{i_1i_2}).
    \end{align}
    Then, it follows from the identifiability of the location-scale Gaussian mixtures \cite{Teicher-1960,Teicher-1961} that the number of components and the weight set of the mixing measure $G$ equal to those of its counterpart $G_*$, i.e. $k_2=k_2^*$ and
    \begin{align*}
        &\Bigg\{\softmax\Big((\boldsymbol{a}_{i_1})^{\top}\bx+b_{i_1}\Big)\cdot\softmax\Big((\boldsymbol{\omega}_{i_2|i_1})^{\top}\bx+\beta_{i_2|i_1}\Big):i_1\in[k^*_1],i_2\in[k^*_2]\Bigg\}\\
        &=\Bigg\{\softmax\Big((\boldsymbol{a}^*_{i_1})^{\top}\bx+b^*_{i_1}\Big)\cdot\softmax\Big((\boldsymbol{\omega}^*_{i_2|i_1})^{\top}\bx+\beta^*_{i_2|i_1}\Big):i_1\in[k^*_1],i_2\in[k^*_2]\Bigg\},
    \end{align*}
    for almost every $\bx$. WLOG, we may assume that 
    \begin{align}
        \label{eq:general_soft}
        &\softmax\Big((\boldsymbol{a}_{i_1})^{\top}\bx+b_{i_1}\Big)\cdot\softmax\Big((\boldsymbol{\omega}_{i_2|i_1})^{\top}\bx+\beta_{i_2|i_1}\Big)=\softmax\Big((\boldsymbol{a}^*_{i_1})^{\top}\bx+b^*_{i_1}\Big)\cdot\softmax\Big((\boldsymbol{\omega}^*_{i_2|i_1})^{\top}\bx+\beta^*_{i_2|i_1}\Big),
    \end{align}
    for almost every $\bx$, for any $i_1\in[k^*_1],i_2\in[k^*_2]$.  
    Due to the assumptions that $\boldsymbol{\omega}_{k^*_2|i_1}=\boldsymbol{\omega}^*_{k^*_2|i_1}=\mathbf{0}_{d}$ and $\beta_{k^*_2|i_1}=\beta^*_{k^*_2|i_1}=0$, we have that
    \begin{align}
        \label{eq:general_soft-soft}
        \softmax\Big((\boldsymbol{a}_{i_1})^{\top}\bx+b_{i_1}\Big)=\softmax\Big((\boldsymbol{a}^*_{i_1})^{\top}\bx+b^*_{i_1}\Big),
    \end{align}
    for almost every $\bx$, for any $i_1\in$. Since the $\softmax$ function is invariant to translations, then it follows from the equation~\eqref{eq:general_soft-soft} that 
    \begin{align*}
       \boldsymbol{a}_{i_1}&=\boldsymbol{a}^*_{i_1}+\boldsymbol{a}\\
        b_{i_1}&=b^*_{i_1}+b,
    \end{align*}
    for some $\boldsymbol{a}\in\mathbb{R}^d$ and $b\in\mathbb{R}$. Moreover, due to the assumption that $\boldsymbol{a}_{k^*_1}=\boldsymbol{a}^*_{k^*_1}$ and $b_{k^*_1}=b^*_{k^*_1}=0$, we get $\boldsymbol{a}=\zerod$ and $b=0$. This leads to $\boldsymbol{a}_{i_1}=\boldsymbol{a}^*_{i_1}$ and $b_{i_1}=b^*_{i_1}$ for any $i_1\in[k^*_1]$. Those results together with equation~\eqref{eq:general_soft} yield that
    \begin{align*}
        \softmax\Big((\boldsymbol{\omega}_{i_2|i_1})^{\top}\bx+\beta_{i_2|i_1}\Big)=\softmax\Big((\boldsymbol{\omega}^*_{i_2|i_1})^{\top}\bx+\beta^*_{i_2|i_1}\Big),
    \end{align*}
    for almost every $\bx$, for any $i_1\in[k^*_1],i_2\in[k^*_2]$. By employing the previous arguments, we also obtain that 
    \begin{align*}
    \boldsymbol{\omega}_{i_2|i_1}&=\boldsymbol{\omega}^*_{i_2|i_1},\\
        \beta_{i_2|i_1}&=\beta^*_{i_2|i_1}.
    \end{align*}
    Then, the equation~\eqref{eq:general_identifiable_equation} can be rewritten as
    \begin{align}
        \label{eq:general_new_identifiable_equation}
        &\sum_{i_1=1}^{k^*_1}\exp(b_{i_1})\sum_{i_2=1}^{k^*_2}\exp(\beta_{i_2|i_1})\exp\Big((\boldsymbol{a}_{i_1}+\boldsymbol{\omega}_{i_2|i_1})^{\top}\bx\Big)\pi(y|(\boldsymbol{\eta}_{i_1i_2})^{\top}\bx+\tau_{i_1i_2},\nu_{i_1i_2})\nonumber\\
        &=\sum_{i_1=1}^{k^*_1}\exp(b^*_{i_1})\sum_{i_2=1}^{k^*_2}\exp(c^*_{i_2|i_1})\exp\Big((\boldsymbol{a}^*_{i_1}+\boldsymbol{\omega}^*_{i_2|i_1})^{\top}\bx\Big)\pi(y|(\boldsymbol{\eta}^*_{i_1i_2})^{\top}\bx+\tau^*_{i_1i_2},\nu^*_{i_1i_2}).
    \end{align}
    for almost every $\bx\in\mathcal{X}$. 

    \vspace{0.5 em}
\noindent
    Next, we denote $P_1,P_2,\ldots,P_{m_1}$ as a partition of the index set $[k^*_1]$, where $m_1\leq k^*_1$, such that $\exp(b_{i_1})=\exp(b^*_{i'_1})$ for any $i_1,i'_1\in P_j$ and $j_1\in[m_1]$. On the other hand, when $i_1$ and $i'_1$ do not belong to the same set $P_{j_1}$, we let $\exp(b_{i_1})\neq\exp(b^*_{i'_1})$. 

    \vspace{0.5 em}
\noindent
    Similarly, for each $i_1\in[k^*_1]$, we also define $Q_{1|i_1},Q_{2|i_1},\ldots,Q_{m_2|i_1}$ as a partition of the index set $[k^*_2]$, where $m_2\leq k^*_2$, such that $\exp(\beta_{i_2|i_1})=\exp(\beta^*_{i'_2|i_1})$ for any $i_2,i'_2\in Q_{j_2|i_1}$ and $j_2\in[m_2]$. Conversely, when $i_2$ and $i'_2$ do not belong to the same set $Q_{j_2|i_1}$, we let $\exp(\beta_{i_2|i_1})\neq\exp(\beta^*_{i'_2|i_1})$. 

    \vspace{0.5 em}
\noindent
    Thus, we can represent equation~\eqref{eq:general_new_identifiable_equation} as
    \begin{align*}
        &\sum_{j_1=1}^{m_1}\sum_{i_1\in{P}_{j_1}}\exp(b_{i_1})\sum_{j_2=1}^{m_2}\sum_{i_1\in{Q}_{j_2|i_1}}\exp(\beta_{i_2|i_1})\exp\Big((\boldsymbol{a}_{i_1}+\boldsymbol{\omega}_{i_2|i_1})^{\top}\bx\Big)\pi(y|(\boldsymbol{\eta}_{i_1i_2})^{\top}\bx+\tau_{i_1i_2},\nu_{i_1i_2})\nonumber\\
        &=\sum_{j_1=1}^{m_1}\sum_{i_1\in{P}_{j_1}}\exp(b^*_{i_1})\sum_{j_2=1}^{m_2}\sum_{i_1\in{Q}_{j_2|i_1}}\exp(\beta^*_{i_2|i_1})\exp\Big((\boldsymbol{a}^*_{i_1}+\boldsymbol{\omega}^*_{i_2|i_1})^{\top}\bx\Big)\pi(y|(\boldsymbol{\eta}^*_{i_1i_2})^{\top}\bx+\tau^*_{i_1i_2},\nu^*_{i_1i_2}),
    \end{align*}
    for almost every $\bx\in\mathcal{X}$. Recall that we have $b_{i_1}=b^*_{i_1}$, $\boldsymbol{a}_{i_1}=\boldsymbol{a}^*_{i_1}$, $\boldsymbol{\omega}_{i_2|i_1}=\boldsymbol{\omega}^*_{i_2|i_1}$ and $\beta_{i_2|i_1}=\beta^*_{i_2|i_1}$, for any $i_1\in[k^*_1]$ and $i_2\in[k^*_2]$, then the above result leads to
    \begin{align*}
        \Big\{\Big((\boldsymbol{\eta}_{i_1i_2})^{\top}\bx+\tau_{i_1i_2},\nu_{i_1i_2}\Big)&:i_1\in P_{j_1}, i_2\in Q_{j_2|i_1}\Big\}\\
        &\equiv\Big\{\Big((\boldsymbol{\eta}^*_{i_1i_2})^{\top}\bx+\tau^*_{i_1i_2},\nu^*_{i_1i_2}\Big):i_1\in P_{j_1}, i_2\in Q_{j_2|i_1}\Big\},
    \end{align*}
    for any $j_1\in[m_1]$ and $j_2\in[m_2]$. Consequently, we obtain that
    \begin{align*}
        G&=\sum_{j_1=1}^{m_1}\sum_{i_1\in{P}_{j_1}}\exp(b_{i_1})\sum_{j_2=1}^{m_2}\sum_{i_1\in{Q}_{j_2|i_1}}\exp(\beta_{i_2|i_1})\delta_{(\boldsymbol{a}_{i_1},\boldsymbol{\omega}_{i_2|i_1},\boldsymbol{\eta}_{i_1i_2},\tau_{i_1i_2},\nu_{i_1i_2})}\\
        &=\sum_{j_1=1}^{m_1}\sum_{i_1\in{P}_{j_1}}\exp(b^*_{i_1})\sum_{j_2=1}^{m_2}\sum_{i_1\in{Q}_{j_2|i_1}}\exp(\beta^*_{i_2|i_1})\delta_{\boldsymbol{a}^*_{i_1},\boldsymbol{\omega}^*_{i_2|i_1},\boldsymbol{\eta}^*_{i_1i_2},\tau^*_{i_1i_2},\nu^*_{i_1i_2})}\\
        &\equiv G_*.
    \end{align*}
    Hence, the proof is totally completed.
\end{proof}

\bibliography{reference}
\bibliographystyle{abbrv}

\end{document}